\documentclass[10pt,journal,compsoc]{IEEEtran}
\usepackage{times}
\usepackage{epsfig}
\usepackage{graphicx}
\usepackage{amsmath,mathrsfs}
\usepackage{amssymb}
\usepackage{enumerate}
\usepackage{threeparttable}
\usepackage{multirow}
\usepackage{booktabs}
\usepackage{algorithm}
\usepackage{algorithmic}
\usepackage{amsmath,bm,amsfonts}
\usepackage{mcode}
\usepackage{subcaption}
\usepackage{amsthm}
\usepackage{url}
\usepackage{caption,color}
\usepackage{threeparttable}
\usepackage[top=1in, bottom=1in, left=0.65in, right=0.75in]{geometry}

\newcommand{\Sbm}{\bm{\mathcal{S}}}

\newcommand{\sgnop}{\operatorname{sgn}}
\newcommand{\sgn}[1]{\ensuremath{\sgnop\left(#1\right)}}

\newcommand{\inproduct}[2]{\left\langle#1,#2\right\rangle}

\newtheorem{thm}{Theorem}[section]
\newtheorem{lemma}[thm]{Lemma}
\newtheorem{cor}[thm]{Corollary}

\newtheorem{defn}{Definition}[section]
\DeclareMathOperator*{\argmin}{argmin}

\newcommand{\Tr}{\mcode{Tr}}
\newcommand{\conj}{\mcode{conj}}
\newcommand{\Diag}{\mcode{Diag}}
\newcommand{\unfold}{\mcode{unfold}}
\newcommand{\fold}{\mcode{fold}}

\newcommand{\st}{{\text{s.t.}}}

\newcommand{\fft}{\mcode{fft}}
\newcommand{\ifft}{\mcode{ifft}}
\newcommand{\rankm}{\text{rank}}
\newcommand{\rankt}{\text{rank}_{\text{t}}}
\newcommand{\rank}{\text{rank}}

\newcommand{\ranka}{\text{rank}_{\text{a}}}

\newcommand{\none}{n_{(1)}}
\newcommand{\ntwo}{n_{(2)}}
\newcommand{\Rn}{\mathbb{R}^{n_1\times n_2\times n_3}}
\newcommand{\Cnnn}{\mathbb{C}^{n_1\times n_2\times n_3}}

\newcommand{\Cn}{\mathbb{C}^{n}}

\newcommand{\nss}{n\times n\times n_3}

\newcommand{\Pomega}{\bm{\mathcal{P}}_{\bm{{\Omega}}}}
\newcommand{\Pomegao}{\bm{\mathcal{P}}_{{\bm{{\Omega}}^\bot}}}
\newcommand{\Omegat}{\bm{{\Omega}}}

\newcommand{\PT}{\bm{\mathcal{P}}_{\bm{T}}}
\newcommand{\PTo}{\bm{\mathcal{P}}_{{\bm{T}}^\bot}}

\newcommand{\eabc}{{\mathfrak{e}}_{abc}}
\newcommand{\eijk}{{\mathfrak{e}}_{ijk}}
\newcommand{\ei}{\mathring{\mathfrak{e}}_i}
\newcommand{\ej}{\mathring{\mathfrak{e}}_j}

\newcommand{\ek}{\dot{\mathfrak{e}}_k}

\newcommand{\bcirc}{\mcode{bcirc}}
\newcommand{\bdiag}{\mcode{bdiag}}
\newcommand{\Ber}{\text{Ber}}
\newcommand{\A}{\bm{\mathcal{A}}}
\newcommand{\B}{\bm{\mathcal{B}}}
\newcommand{\C}{\bm{\mathcal{C}}}
\newcommand{\D}{\bm{\mathcal{D}}}
\newcommand{\E}{\bm{\mathcal{E}}}
\newcommand{\G}{\bm{\mathcal{G}}}
\newcommand{\HH}{\bm{\mathcal{H}}}
\newcommand{\I}{\bm{\mathcal{I}}}
\newcommand{\LL}{\bm{\mathcal{L}}}
\newcommand{\M}{\bm{\mathcal{M}}}
\newcommand{\PP}{\bm{\mathcal{P}}}
\newcommand{\Q}{\bm{\mathcal{Q}}}
\newcommand{\R}{\bm{\mathcal{R}}}

\newcommand{\T}{\bm{\mathcal{T}}}
\newcommand{\U}{\bm{\mathcal{U}}}
\newcommand{\V}{\bm{\mathcal{V}}}
\newcommand{\W}{\bm{{\mathcal{W}}}}
\newcommand{\X}{\bm{\mathcal{X}}}
\newcommand{\Y}{\bm{\mathcal{Y}}}
\newcommand{\Z}{\bm{\mathcal{Z}}}
\newcommand{\Lhat}{\hat{\bm{\mathcal{L}}}}
\newcommand{\Shat}{\hat{\bm{\mathcal{S}}}}

\newcommand{\WL}{{\bm{\mathcal{W}}}^{\bm{\mathcal{L}}} }
\newcommand{\WS}{{\bm{\mathcal{W}}}^{\bm{\mathcal{S}}} }

\newcommand{\Ehat}{\hat{\bm{\mathcal{E}}}}
\newcommand{\Abar}{\bm{\mathcal{\bar{A}}}}
\newcommand{\Bbar}{\bm{\mathcal{\bar{B}}}}
\newcommand{\Cbar}{\bm{\mathcal{\bar{C}}}}

\newcommand{\Ubar}{\bm{\mathcal{\bar{U}}}}
\newcommand{\Vbar}{\bm{\mathcal{\bar{V}}}}
\newcommand{\Sbar}{\bm{\mathcal{\bar{S}}}}
\newcommand{\Wbar}{\bm{\mathcal{\bar{W}}}}

\newcommand{\Ybar}{\bm{\mathcal{\bar{Y}}}}
\newcommand{\Mbar}{\bm{\mathcal{\bar{M}}}}
\newcommand{\Ibar}{\bm{\mathcal{\bar{I}}}}

\newcommand{\Am}{\bm{{A}}}
\newcommand{\Bm}{\bm{{B}}}

\newcommand{\Em}{\bm{E}}
\newcommand{\F}{\bm{{F}}}
\newcommand{\Wm}{\bm{W}}
\newcommand{\Lm}{\bm{L}}
\newcommand{\Sm}{\bm{S}}
\newcommand{\Xm}{\bm{X}}

\newcommand{\Vm}{\bm{V}}
\newcommand{\Um}{\bm{U}}
\newcommand{\Tm}{\bm{T}}
\newcommand{\Zm}{\bm{Z}}
\renewcommand{\Im}{\bm{I}}

\newcommand{\Ambar}{\bm{\bar{A}}}
\newcommand{\Bmbar}{\bm{\bar{{B}}}}
\newcommand{\Cmbar}{\bm{\bar{C}}}

\newcommand{\Umbar}{\bm{\bar{U}}}
\newcommand{\Vmbar}{\bm{\bar{V}}}
\newcommand{\Xmbar}{\bm{\bar{X}}}
\newcommand{\Lmbar}{\bm{\bar{L}}}
\newcommand{\Ymbar}{\bm{\bar{Y}}}

\newcommand{\Smbar}{\bm{\bar{S}}}

\newcommand{\Zmbar}{\bm{\bar{Z}}}
\newcommand{\Hmbar}{\bm{\bar{H}}}
\newcommand{\Mmbar}{\bm{\bar{M}}}

\newcommand{\Wmbar}{\bm{\bar{W}}}

\newcommand{\Bmtilde}{\bm{\tilde{{B}}}}

\newcommand{\x}{\bm{x}}
\newcommand{\y}{\bm{y}}

\newcommand{\f}{\ensuremath{\bm{f}}}
\newcommand{\vbar}{\bar{\bm{v}}}

\renewcommand{\aa}{{\bm{a}}}
\newcommand{\e}{{\text{e}}}
\newcommand{\vv}{{\bm{v}}}
\newcommand{\0}{\ensuremath{\mathbf{0}}}

\newcommand{\abs}[1]{\left\lvert#1\right\rvert}
\newcommand{\norm}[1]{\lVert#1\rVert}
\newcommand{\normlarge}[1]{\left\lVert#1\right\rVert}

{%
\begin{list}{#1}{
\vspace{-\topsep}
\vspace{-\partopsep}
\setlength{\itemindent}{0cm}
\setlength{\rightmargin}{0cm}
\setlength{\listparindent}{0cm}
\settowidth{\labelwidth}{#1}
\setlength{\leftmargin}{\labelwidth}
\addtolength{\leftmargin}{\labelsep}
\setlength{\itemsep}{0cm}
}%
}%
{%
\end{list}
\vspace{-\topsep}
\vspace{-\partopsep}
}

{\begin{enumerate}%
\renewcommand{\theenumi}{\roman{enumi}}%
}%
{\end{enumerate}}

\newcommand{\Mh}{{\bm{{M}}}^{{{H}}}}
\newcommand{\sumi}{\sum_{i=1}^{n_3}}

\newcommand{\canyi}[1]{\textcolor{black}{#1}}

\begin{document}

\title{Tensor Robust Principal Component Analysis with A New Tensor Nuclear Norm}

\author{Canyi~Lu,~Jiashi~Feng,~Yudong~Chen,~Wei Liu,~Member,~IEEE,~Zhouchen Lin,~\IEEEmembership{Fellow,~IEEE},
	
	 ~and~Shuicheng Yan,~\IEEEmembership{Fellow,~IEEE}
	
	\IEEEcompsocitemizethanks{\IEEEcompsocthanksitem C. Lu is with the Department of Electrical and Computer Engineering, Carnegie Mellon University (e-mail: canyilu@gmail.com).\protect
	\IEEEcompsocthanksitem J. Feng and S. Yan are with the Department of Electrical and Computer Engineering, National University of Singapore, Singapore (e-mail: elefjia@nus.edu.sg; eleyans@nus.edu.sg).
	\IEEEcompsocthanksitem Y. Chen is with the School of Operations Research and Information Engineering, Cornell University (e-mail: yudong.chen@cornell.edu).
	\IEEEcompsocthanksitem  W. Liu is with the Tencent AI Lab, Shenzhen, China (e-mail:
	wl2223@columbia.edu).
	\IEEEcompsocthanksitem Z. Lin is with the Key Laboratory of Machine Perception (MOE), School of Electronics Engineering and Computer Science, Peking University, Beijing 100871, China (e-mail: zlin@pku.edu.cn).
	}
}

\markboth{IEEE TRANSACTIONS ON PATTERN ANALYSIS AND MACHINE INTELLIGENCE}%
{Shell \MakeLowercase{\textit{et al.}}: Bare Advanced Demo of IEEEtran.cls for Journals}

\IEEEtitleabstractindextext{%
\begin{abstract}
	In this paper, we consider the Tensor Robust Principal Component Analysis (TRPCA) problem, which aims to exactly recover the low-rank and sparse components from their sum. Our model is based on the recently proposed tensor-tensor product (or t-product) \cite{kilmer2011factorization}. Induced by the t-product, we first rigorously deduce the tensor spectral norm, tensor nuclear norm, and tensor average rank, and show that the tensor nuclear norm is the convex envelope of the tensor average rank within the unit ball of the tensor spectral norm. These definitions, their relationships and properties are consistent with matrix cases. Equipped with the new tensor nuclear norm, we then solve the TRPCA problem by solving a convex program and provide the theoretical guarantee for the exact recovery. Our TRPCA model and recovery guarantee include matrix RPCA as a special case. Numerical experiments verify our results, and the applications to image recovery and background modeling problems demonstrate the effectiveness of our method. 
\end{abstract}
	
}
\maketitle		
\IEEEdisplaynontitleabstractindextext
\IEEEpeerreviewmaketitle
	
	\ifCLASSOPTIONcompsoc
	\IEEEraisesectionheading{\section{Introduction}\label{sec:introduction}}
	\else
	
\section{Introduction}\label{sec:introduction}
\fi
\IEEEPARstart{P}{rincipal} Component Analysis (PCA) is a fundamental approach for data analysis. It exploits low-dimensional structure in high-dimensional data, which commonly exists in different types of data, \textit{e.g.}, image, text, video and bioinformatics. It is computationally efficient and powerful for data instances which are mildly corrupted by small noises. However, a major issue of PCA is that it is brittle to be grossly corrupted or outlying observations, which are ubiquitous in real-world data. To date, a number of robust versions of PCA have been proposed, but many of them suffer from a high computational cost. 

\begin{figure}[!t]
	\centering
	\begin{subfigure}[b]{0.4\textwidth}
		\centering
		\includegraphics[width=\textwidth]{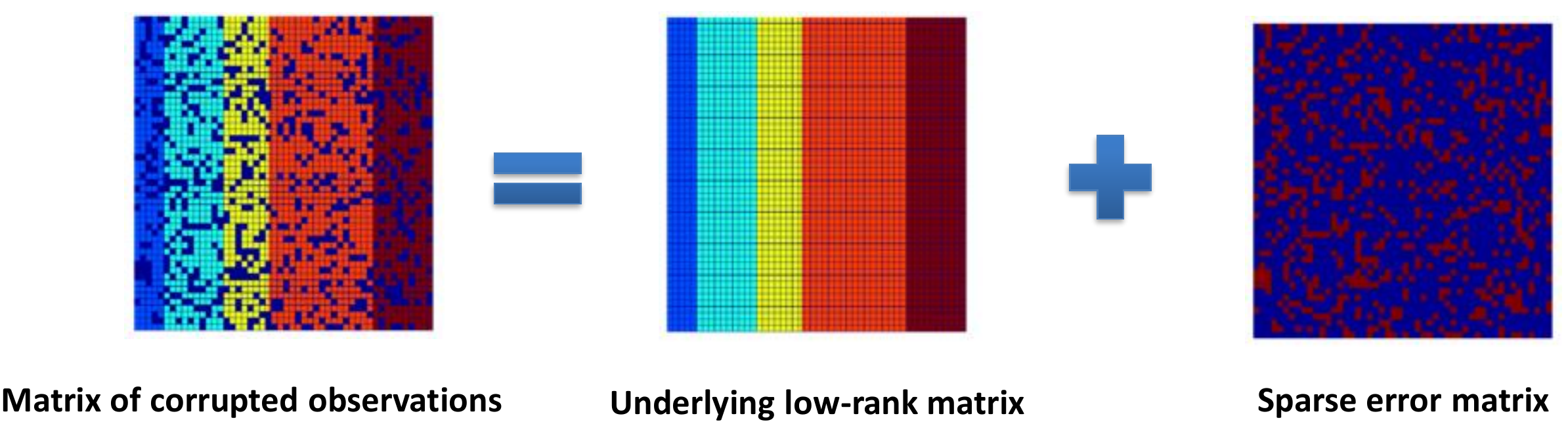}\vspace{2mm}
	\end{subfigure} 
	
	\begin{subfigure}[b]{0.4\textwidth}
		\centering
		\includegraphics[width=\textwidth]{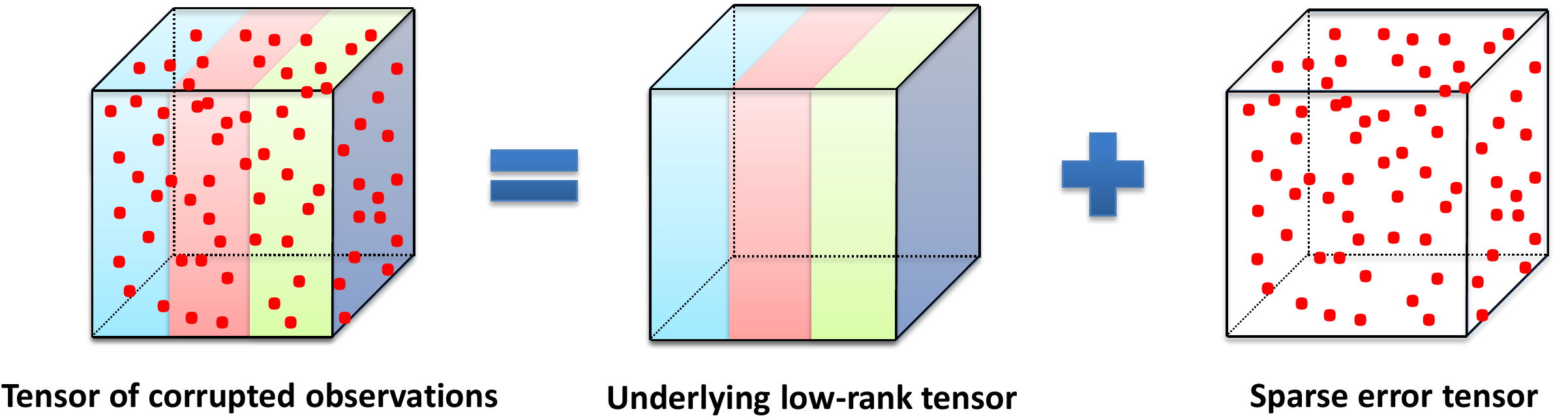}
	\end{subfigure}
	\caption{\small{Illustrations of RPCA \cite{RPCA} (up row) and our Tensor RPCA (bottom row). RPCA: low-rank and sparse matrix decomposition from noisy matrix observations. Tensor RPCA: low-rank and sparse tensor decomposition from noisy tensor observations.}}
	\label{fig_trpca1}
	\vspace{-0.45cm}
\end{figure}	
	
The Robust PCA \cite{RPCA} is the first polynomial-time algorithm with strong recovery guarantees. Suppose that we are given an observed matrix $\Xm\in\mathbb{R}^{n_1\times n_2}$, which can be decomposed as $\Xm=\Lm_0+\Em_0$, where $\Lm_0$ is low-rank and $\Em_0$ is sparse. It is shown in \cite{RPCA} that if the singular vectors of $\Lm_0$ satisfy some incoherent conditions, \textit{e.g.}, $\Lm_0$ is low-rank and \canyi{$\Em_0$} is sufficiently sparse, then $\Lm_0$ and \canyi{$\Em_0$} can be exactly recovered with high probability by solving the following convex problem
\begin{equation}\label{rpca}
	\min_{\Lm,\Em} \ \norm{\Lm}_*+\lambda\norm{\Em}_1, \ \st \ \Xm=\Lm+\Em,
\end{equation}
where $\norm{\Lm}_*$ denotes the nuclear norm (sum of the singular values of $\Lm$), and $\norm{\Em}_1$ denotes the $\ell_1$-norm (sum of the absolute values of all the entries in $\Em$). Theoretically, RPCA is guaranteed to work even if the rank of $\Lm_0$ grows almost linearly in the dimension of the matrix, and the errors in \canyi{$\Em_0$} are up to a constant fraction of all entries. The parameter $\lambda$ is suggested to be set as $1/\sqrt{\max(n_1,n_2)}$ which works well in practice. Algorithmically, program (\ref{rpca}) can be solved by efficient algorithms, at a cost not too much higher than PCA. RPCA and its extensions have been successfully applied to background modeling \cite{RPCA}, subspace clustering \cite{robustlrr}, video compressive sensing \cite{waters2011sparcs}, \textit{etc}. 

One major shortcoming of RPCA is that it can only handle 2-way (matrix) data. However, real data is usually multi-dimensional in nature-the information is stored in multi-way arrays known as tensors \cite{kolda2009tensor}. For example, a color image is a 3-way object with column, row and color modes; a greyscale video is indexed by two spatial variables and one temporal variable. To use RPCA, one has to first restructure the multi-way data into a matrix. Such a preprocessing usually leads to an information loss and would cause a performance degradation. To alleviate this issue, it is natural to consider extending RPCA to manipulate the tensor data by taking advantage of its multi-dimensional structure.
	
In this work, we are interested in the Tensor Robust Principal Component (TRPCA) model which aims to exactly recover a low-rank tensor corrupted by sparse errors. See Figure \ref{fig_trpca1} for an intuitive illustration. More specifically, suppose that we are given a data tensor $\X$, and know that it can be decomposed as
\begin{equation}\label{xls}
	\X=\LL_0+\E_0,
\end{equation}
where $\LL_0$ is low-rank and $\E_0$ is sparse, and both components are of arbitrary magnitudes. Note that we do not know the locations of the nonzero elements of $\E_0$, not even how many there are. Now we consider a similar problem to RPCA. Can we recover the low-rank and sparse components exactly and efficiently from $\X$? This is the problem of tensor RPCA studied in this work.
	
The tensor extension of RPCA is not easy since the numerical algebra of tensors is fraught with hardness results \cite{hillar2013most}, \canyi{\cite{anandkumar2016homotopy,zhang2018tensor}}. A main issue is that the tensor rank is not well defined with a tight convex relaxation. Several tensor rank definitions and their convex relaxations have been proposed but each has its limitation. For example, the CP rank \cite{kolda2009tensor}, defined as the smallest number of rank one tensor decomposition, is generally NP-hard to compute. Also its convex relaxation is intractable. This makes the low CP rank tensor recovery challenging. The tractable Tucker rank \cite{kolda2009tensor} and its convex relaxation are more widely used. For a $k$-way tensor $\X$, the Tucker rank is a vector defined as $\text{rank}_{\text{tc}}(\X):=\left( \text{rank}(\Xm^{\{1\}}), \text{rank}(\Xm^{\{2\}}), \cdots, \text{rank}(\Xm^{\{k\}}) \right)$, where $\Xm^{\{i\}}$ is the mode-$i$ matricization of $\X$ \cite{kolda2009tensor}. Motivated by the fact that the nuclear norm is the convex envelope of the matrix rank within the unit ball of the spectral norm, the Sum of Nuclear Norms (SNN)~\cite{liu2013tensor}, defined as $\sum_i\norm{\Xm^{\{i\}}}_*$, is used as a convex surrogate of $\sum_i\text{rank}({\Xm^{\{i\}}})$. Then the work \cite{mu2013square} considers the Low-Rank Tensor Completion (LRTC) model based on SNN:
\begin{equation}\label{tensorc}
	\min_{\X} \ \sum_{i=1}^{k}\lambda_i\norm{\Xm^{\{i\}}}_*, \ \st \ \Pomega(\X) = \Pomega(\M),
\end{equation}
where $\lambda_i>0$, and $\Pomega(\X)$ denotes the projection of $\X$ on the observed set $\Omegat$. The effectiveness of this approach for image processing has been well studied in \cite{liu2013tensor,tomioka2010estimation}. However, SNN is not the convex envelope of $\sum_i\text{rank}({\Xm^{\{i\}}})$ \cite{romera2013new}. Actually, the above model can be substantially suboptimal \cite{mu2013square}: reliably recovering a $k$-way tensor of length $n$ and Tucker rank $(r,r,\cdots,r)$ from Gaussian measurements requires $O(rn^{k-1})$ observations. In contrast, a certain (intractable) nonconvex formulation needs only $O(rK +nrK)$ observations. A better (but still suboptimal) convexification based on a more balanced matricization is proposed in \cite{mu2013square}. The work \cite{huang2014provable} presents the recovery guarantee for the SNN based tensor RPCA model
\begin{equation}\label{eqsnn}
	\min_{\LL,\E} \ \sum_{i=1}^{k} \lambda_i\norm{{\bm{L}}^{\{i\}}}_*+\norm{\E}_1, \ \st \ \X=\LL+\E.
\end{equation}
\canyi{A robust tensor CP decomposition problem is studied in \cite{anandkumar2016tensor}. Though the recovery is guaranteed, the algorithm is nonconvex.}
 
The limitations of existing works motivate us to consider an interesting problem: is it possible to define a new tensor nuclear norm such that it is a tight convex surrogate of certain tensor rank, and thus its resulting tensor RPCA enjoys a similar tight recovery guarantee to that of the matrix RPCA? This work will provide a positive answer to this question. Our solution is inspired by the recently proposed tensor-tensor product (t-product) \cite{kilmer2011factorization} which is a generalization of the matrix-matrix product. It enjoys several similar properties to the matrix-matrix product. For example, based on t-product, any tensors have the tensor Singular Value Decomposition (t-SVD) and this motivates a new tensor rank, \textit{i.e.}, tensor tubal rank \cite{kilmer2013third}. To recover a tensor of low tubal rank, we propose a new tensor nuclear norm which is rigorously induced by the t-product. First, the tensor spectral norm can be induced by the operator norm when treating the t-product as an operator. Then the tensor nuclear norm is defined as the dual norm of the tensor spectral norm. We further propose the tensor average rank (which is closely related to the tensor tubal rank), and prove that its convex envelope is the tensor nuclear norm within the unit ball of the tensor spectral norm. It is interesting that this framework, including the new tensor concepts and their relationships, is consistent with the one for the matrix cases. Equipped with these new tools, we then study the TRPCA problem which aims to recover the low tubal rank component $\LL_0$ and sparse component $\E_0$ from noisy observations $\X=\LL_0+\E_0\in\mathbb{R}^{n_1\times n_2\times n_3}$ (this work focuses on the 3-way tensor) by convex optimization
\begin{align}\label{trpca}
	\min_{\LL,\ \E} \ \norm{\LL}_*+\lambda\norm{\E}_1, \ \st \ \X=\LL+\E,
\end{align}
where $\norm{\LL}_*$ is our new tensor nuclear norm (see the definition in Section \ref{sec_TNN}). We prove that under certain incoherence conditions, the solution to (\ref{trpca}) perfectly recovers the low-rank and the sparse components, provided of course that the tubal rank of $\LL_0$ is not too large, and that $\E_0$ is reasonably sparse. A remarkable fact, like in RPCA, is that (\ref{trpca}) has no tunning parameter either. \canyi{Our analysis shows that $\lambda=1/\sqrt{\max(n_1,n_2)n_3}$ guarantees the exact recovery when $\LL_0$ and $\E_0$ satisfy certain assumptions.} As a special case, if $\X$ reduces to a matrix ($n_3=1$ in this case), all the new tensor concepts reduce to the matrix cases. Our TRPCA model (\ref{trpca}) reduces to RPCA in (\ref{rpca}), and also our recovery guarantee in Theorem \ref{thm1} reduces to Theorem 1.1 in \cite{RPCA}. Another advantage of (\ref{trpca}) is that it can be solved by polynomial-time algorithms.
	
The contributions of this work are summarized as follows:
\begin{enumerate}[1.]
	\item Motivated by the t-product \cite{kilmer2011factorization} which is a natural generalization of the matrix-matrix product, we rigorously deduce a new tensor nuclear norm and some other related tensor concepts, and they own the same relationship as the matrix cases. This is the foundation for the extensions of the models, optimization method and theoretical analyzing techniques from matrix cases to tensor cases. 
	\item Equipped with the tensor nuclear norm, we theoretically show that under certain incoherence conditions, the solution to the convex TRPCA model (\ref{trpca}) perfectly recovers the underlying low-rank component $\LL_0$ and sparse component $\E_0$. RPCA \cite{RPCA} and its recovery guarantee fall into our special cases. 
	\item \canyi{We give a new rigorous proof of t-SVD factorization and a more efficient way than \cite{lu2016tensorrpca} for solving TRPCA. We further perform several simulations to corroborate our theoretical results. Numerical experiments on images and videos also show the superiority of TRPCA over RPCA and SNN.}
\end{enumerate}

The rest of this paper is structured as follows. Section \ref{sec_notations} gives some notations and preliminaries. Section \ref{sec_TNN} presents the way for defining the tensor nuclear norm induced by the t-product. Section \ref{sec_tcTNN} provides the recovery guarantee of TRPCA and the optimization details. Section \ref{sec_exp} presents numerical experiments conducted on synthetic and real data. We conclude this work in Section \ref{sec_con}.

\section{Notations and Preliminaries}\label{sec_notations}


\subsection{Notations}

In this paper, we denote tensors by boldface Euler script letters, \textit{e.g.}, $\A$. Matrices are denoted by boldface capital letters, \textit{e.g.}, $\Am$; vectors are denoted by boldface lowercase letters, \textit{e.g.}, $\aa$, and scalars are denoted by lowercase letters, \textit{e.g.}, $a$. We denote $\bm I_n$ as the $n\times n$ identity matrix. The fields of real numbers and complex numbers are denoted as $\mathbb{R}$ and $\mathbb{C}$, respectively. \canyi{For a 3-way tensor $\A\in\Cnnn$, we denote its $(i,j,k)$-th entry as $\A_{ijk}$ or $a_{ijk}$ and use the Matlab notation $\A(i,:,:)$, $\A(:,i,:)$ and $\A(:,:,i)$ to denote respectively the $i$-th horizontal, lateral and frontal slice (see definitions in \cite{kolda2009tensor}).} More often, the frontal slice $\A(:,:,i)$ is denoted compactly as $\Am^{(i)}$. The tube is denoted as $\A(i,j,:)$. The inner product between $\Am$ and $\Bm$ in $\mathbb{C}^{n_1\times n_2}$ is defined as $\inproduct{\Am}{\Bm}=\Tr(\Am^*\Bm)$, where $\Am^*$ denotes the conjugate transpose of $\Am$ and $\Tr(\cdot)$ denotes the matrix trace. The inner product between $\A$ and $\B$ in $\Cnnn$ is defined as $\langle\A,\B\rangle=\sum_{i=1}^{n_3}\inproduct{\Am^{(i)}}{\Bm^{(i)}}$. For any $\A\in\Cnnn$, the complex conjugate of $\A$ is denoted as $\conj(\A)$ which takes the complex conjugate of each entry of $\A$. We denote $\left\lfloor t \right\rfloor$ as the nearest integer less than or equal to $t$ and $\lceil t \rceil$ as the one greater than or equal to $t$.

Some norms of vector, matrix and tensor are used. We denote the $\ell_1$-norm as $\norm{\A}_1=\sum_{ijk}|a_{ijk}|$, the infinity norm as $\norm{\A}_\infty=\max_{ijk}|a_{ijk}|$ and the Frobenius norm as $\norm{\A}_F=\sqrt{\sum_{ijk}|a_{ijk}|^2}$, respectively. The above norms reduce to the vector or matrix norms if $\A$ is a vector or a matrix. For $\vv\in\mathbb{C}^n$, the $\ell_2$-norm is $\norm{\vv}_2 = \sqrt{\sum_{i}|v_{i}|^2}$. The spectral norm of a matrix $\Am$ is denoted as $\norm{\Am} = \max_{i}\sigma_i(\Am)$, where $\sigma_i(\Am)$'s are the singular values of $\Am$. The matrix nuclear norm is $\norm{\Am}_* = \sum_{i}\sigma_i(\Am)$. 

\subsection{Discrete Fourier Transformation}

The Discrete Fourier Transformation (DFT) plays a core role in tensor-tensor product introduced later. We give some related background knowledge and notations here. The DFT on $\vv\in\mathbb{R}^n$, denoted as $\vbar$, is given by
\begin{align}\label{eq_fourtrans}
	\vbar = \F_n\vv \in \mathbb{C}^n,
\end{align}
where $\F_n$ is the DFT matrix defined as
\begin{align*}
	\F_n = 
	\begin{bmatrix}
		1 & 1 & 1 & \cdots & 1 \\
		1 & \omega & \omega^2 & \cdots & \omega^{n-1} \\
		\vdots & \vdots & \vdots & \ddots & \vdots \\
		1 & \omega^{n-1} & \omega^{2(n-1)} & \cdots &\omega^{(n-1)(n-1)}
	\end{bmatrix}\in\mathbb{C}^{n\times n},
\end{align*}
where $\omega = \e^{-\frac{2\pi i}{n}}$ is a primitive $n$-th root of unity in which $i = \sqrt{-1}$. Note that $\F_n/\sqrt{n}$ is a \canyi{unitary} matrix, \textit{i.e.},
\begin{align}\label{eq_Fnorthogonal}
	\F_n^*\F_n = \F_n\F^*_n = n\Im_n.
\end{align}
Thus $\F_n^{-1} = \F^*_n/n$. The above property will be frequently used in this paper. Computing $\vbar$ by using (\ref{eq_fourtrans}) costs $O(n^2)$. A more widely used method is the Fast Fourier Transform (FFT) which costs $O(n\log n)$. By using the Matlab command $\mcode{fft}$, we have $\vbar = \mcode{fft}(\vv)$. Denote the circulant matrix of $\vv$ as
\begin{align*}
	\mcode{circ}(\vv) = 
	\begin{bmatrix}
		v_1 & v_n &\cdots &v_2 \\
		v_2 & v_1 & \cdots & v_3 \\
		\vdots & \vdots & \ddots & \vdots \\
		v_n & v_{n-1} & \cdots & v_1
	\end{bmatrix}\in\mathbb{R}^{n\times n}.
\end{align*} 
It is known that it can be diagonalized by the DFT matrix, \textit{i.e.},
\begin{align}\label{eq_diagonalv}
	\F_n \cdot \mcode{circ}(\vv) \cdot \F_n^{-1} = \Diag(\vbar),
\end{align}
where $\Diag(\vbar)$ denotes a diagonal matrix with its $i$-th diagonal entry as $\bar{v}_i$. The above equation implies that the columns of $\F_n$ are the eigenvectors of $(\mcode{circ}(\vv))^\top$ and $\bar{v}_i$'s are the corresponding eigenvalues. 
\begin{lemma}\label{lem_keyprofft}\cite{rojo2004some} 
	Given any real vector $\vv\in\mathbb{R}^n$, the associated $\vbar$ satisfies 
\begin{align}\label{fftvproper}
	\bar{v}_1 \in\mathbb{R} \text{ and } \conj({\bar{v}}_i) = \bar{v}_{n-i+2}, \ i=2,\cdots, \left\lfloor\frac{n+1}{2} \right\rfloor.
\end{align} 
Conversely, for any given complex $\vbar\in \Cn$ satisfying (\ref{fftvproper}), there exists a real block circulant matrix $\mcode{circ}(\vv)$ such that (\ref{eq_diagonalv}) holds.	
\end{lemma}
As will be seen later, the above properties are useful for efficient computation and important for proofs. Now we consider the DFT on tensors. For $\A\in\Rn$, we denote $\Abar\in\Cnnn$ as the result of DFT on $\A$ along the 3-rd dimension, \textit{i.e.}, performing the DFT on all the tubes of $\A$. By using the Matlab command $\mcode{fft}$, we have
\begin{equation*}
	\Abar=\mcode{fft}(\A,[\ ],3).
\end{equation*}
In a similar fashion, we can compute $\A$ from ${\Abar}$ using the inverse FFT, \textit{i.e.},
\begin{equation*}
	\A = \mcode{ifft}({\Abar},[\ ],3).
\end{equation*}
In particular, we denote $\Ambar\in\mathbb{C}^{n_1n_3\times n_2n_3}$ as a block diagonal matrix with its $i$-th block on the diagonal as the $i$-th frontal slice $\Ambar^{(i)}$ of ${\Abar}$, \textit{i.e.},
\begin{equation*}\label{eq_Abardef}
	\Ambar = \bdiag(\Abar) =
	\begin{bmatrix}
		\Ambar^{(1)} & & & \\
		& \Ambar^{(2)} & & \\
		& & \ddots & \\
		& & & \Ambar^{(n_3)}
	\end{bmatrix},
\end{equation*}
where $\bdiag$ is an operator which maps the tensor $\Abar$ to the block diagonal matrix $\Ambar$. Also, we define the block circulant matrix ${\bcirc}(\A)\in\mathbb{R}^{n_1n_3\times n_2n_3}$ of $\A$ as
\begin{align*} 
	{\bcirc}(\A) =
	\begin{bmatrix}
		\Am^{(1)} &\Am^{(n_3)} &\cdots &\Am^{(2)} \\
		\Am^{(2)} &\Am^{(1)} & \cdots &\Am^{(3)} \\
		\vdots & \vdots & \ddots & \vdots \\
		\Am^{(n_3)} & \Am^{(n_3-1)} & \cdots & \Am^{(1)}
	\end{bmatrix}.
\end{align*}
Just like the circulant matrix which can be diagonalized by DFT, the block circulant matrix can be block diagonalized, \textit{i.e.},
\begin{align}\label{dftpro}
	(\F_{n_3} \otimes \bm{I}_{n_1}) \cdot \mcode{bcirc}(\A) \cdot (\F_{n_3}^{-1} \otimes \bm{I}_{n_2}) = \Ambar,
\end{align}
where $\otimes$ denotes the Kronecker product and $(\F_{n_3}\otimes \bm{I}_{n_1})/\sqrt{n_3}$ is \canyi{unitary}. By using Lemma \ref{lem_keyprofft}, we have
\begin{equation}\label{keyprofffttensor}
	\begin{cases}
		\Ambar^{(1)} \in \mathbb{R}^{n_1\times n_2}, \\
		\conj({\Ambar}^{(i)}) = \Ambar^{(n_3-i+2)}, \ i=2,\cdots,\left\lfloor\frac{n_3+1}{2} \right\rfloor.
	\end{cases}
\end{equation}
Conversely, for any given $\Abar\in\mathbb{C}^{n_1\times n_2\times n_3}$ satisfying (\ref{keyprofffttensor}), there exists a real tensor $\A\in\Rn$ such that (\ref{dftpro}) holds. Also, by using (\ref{eq_Fnorthogonal}), we have the following properties which will be used frequently:
\begin{align}\label{eq_proFnormNuclear}
 	\norm{\A}_F=\frac{1}{\sqrt{n_3}}\norm{\Ambar}_F, 
\end{align}
\begin{align}\label{eq_proinproduct}
 	\inproduct{\A}{\B}=\frac{1}{n_3}\inproduct{\Ambar}{\Bmbar}.
\end{align}

\subsection{T-product and T-SVD}

For $\A\in\Rn$, we define
\begin{equation*}
	\mcode{unfold}(\A) =
	\begin{bmatrix}
		\Am^{(1)} \\ \Am^{(2)} \\ \vdots \\ \Am^{(n_3)} 
	\end{bmatrix}, \ \mcode{fold}(\mcode{unfold}(\A))=\A,
\end{equation*}
where the $\unfold$ operator maps $\A$ to a matrix of size $n_1n_3\times n_2$ and $\fold$ is its inverse operator. 
\begin{defn} \textbf{(T-product)} \cite{kilmer2011factorization}
	Let $\A\in\Rn$ and $\B\in\mathbb{R}^{n_2\times l\times n_3}$. Then the t-product $\A*\B$ is defined to be a tensor of size $n_1\times l\times n_3$, 
	\begin{equation}\label{tproducdef}
		\A*\B = \mcode{fold}(\mcode{bcirc}(\A)\cdot\mcode{unfold}(\B)).
	\end{equation} 
\end{defn}
The t-product can be understood from two perspectives. First, in the original domain, a 3-way tensor of size $n_1\times n_2\times n_3$ can be regarded as an $n_1\times n_2$ matrix with each entry being a tube that lies in the third dimension. Thus, the t-product is analogous to the matrix multiplication except that the circular convolution replaces the multiplication operation between the elements. Note that the t-product reduces to the standard matrix multiplication when $n_3=1$. This is a key observation which makes our tensor RPCA model shown later involve the matrix RPCA as a special case. Second, the t-product is equivalent to the matrix multiplication in the Fourier domain; that is, $\C=\A*\B$ is equivalent to $\Cmbar=\Ambar\Bmbar$ due to (\ref{dftpro}). Indeed, $\C=\A*\B$ implies
\begin{align}
	  & \unfold(\C) \notag \\ 
	= & \mcode{bcirc}(\A)\cdot\mcode{unfold}(\B) \notag \\
	= & (\F_{n_3}^{-1}\otimes \bm{I}_{n_1}) \cdot ( (\F_{n_3}\otimes \bm{I}_{n_1})\cdot \mcode{bcirc}(\A) \cdot (\F_{n_3}^{-1}\otimes \bm{I}_{n_2})) \notag \\
	  & \cdot ((\F_{n_3}\otimes \bm{I}_{n_2})\cdot \mcode{unfold}(\B))\label{eqn_tproducomputproer} \\
	= & (\F_{n_3}^{-1}\otimes \bm{I}_{n_1})\cdot\Ambar\cdot\unfold(\Bbar),\notag
\end{align}
where (\ref{eqn_tproducomputproer}) uses (\ref{dftpro}). Left multiplying both sides with 
$(\F_{n_3}\otimes \bm{I}_{n_1})$ leads to $\unfold(\Cbar)=\Ambar\cdot\unfold(\Bbar)$. This is equivalent to $\Cmbar=\Ambar\Bmbar$. This property suggests an efficient way based on FFT to compute t-product instead of using (\ref{tproducdef}). See Algorithm \ref{alg_ttprod}.
\begin{algorithm}[!t]
	\caption{Tensor-Tensor Product}
	\textbf{Input:} $\A\in\Rn$, $\B\in\mathbb{R}^{n_2\times l\times n_3}$.\\
	\textbf{Output:} $\C = \A * \B\in\mathbb{R}^{n_1\times l\times n_3}$.
	\begin{enumerate}[1.]
		\item Compute $\Abar=\fft(\A,[\ ],3)$ and $\Bbar = \fft(\B,[\ ],3)$.
		\item Compute each frontal slice of $\Cbar$ by
		\begin{equation*}
			\Cmbar^{(i)} =
			\begin{cases}
				\Ambar^{(i)}\Bmbar^{(i)}, \quad & i=1,\cdots, \lceil \frac{n_3+1}{2}\rceil,\\
				\conj(\Cmbar^{(n_3-i+2)}), \quad & i=\lceil \frac{n_3+1}{2}\rceil+1,\cdots, n_3.
			\end{cases}
		\end{equation*}
		\item Compute $\C=\ifft(\Cbar,[\ ],3)$.
	\end{enumerate}
	\label{alg_ttprod}	
\end{algorithm} 

The t-product enjoys many similar properties to the matrix-matrix product. For example, the t-product is associative, \textit{i.e.}, $\A*(\B*\C) = (\A*\B)*\C$. We also need some other concepts on tensors extended from the matrix cases. 
\begin{defn} \textbf{(Conjugate transpose)} 
	The conjugate transpose of a tensor $\A\in\Cnnn$ is the tensor $\A^*\in\mathbb{C}^{n_2\times n_1\times n_3}$ obtained by conjugate transposing each of the frontal slices and then reversing the order of transposed frontal slices 2 through $n_3$.
\end{defn}
	The tensor conjugate transpose extends the tensor transpose \cite{kilmer2011factorization} for complex tensors. As an example, let $\A\in\mathbb{C}^{n_1\times n_2\times 4}$ and its frontal slices be $\Am_1$, $\Am_2$, $\Am_3$ and $\Am_4$. Then
\begin{align*}
	\A^* = \fold\left(\begin{bmatrix} \Am_1^* \\ \Am_4^* \\ \Am_3^* \\ \Am_2^* \end{bmatrix}\right).
\end{align*}
 
\begin{defn} \textbf{(Identity tensor)} \cite{kilmer2011factorization}
	The identity tensor $\I\in\mathbb{R}^{n\times n\times n_3}$ is the tensor with its first frontal
	slice being the $n \times n$ identity matrix, and other frontal slices being all zeros.
\end{defn}
It is clear that $\A *\I = \A$ and $\I * \A = \A$ given the appropriate dimensions. The tensor $\Ibar=\fft(\I,[\ ],3)$ is a tensor with each frontal slice being the identity matrix.
\begin{defn} \textbf{(Orthogonal tensor)} \cite{kilmer2011factorization}
	A tensor $\Q\in\mathbb{R}^{n\times n\times n_3}$ is orthogonal if it satisfies $\Q^**\Q=\Q*\Q^*=\I$.
\end{defn}
\begin{defn}\textbf{(F-diagonal Tensor)} \cite{kilmer2011factorization}
	A tensor is called f-diagonal if each of its frontal slices is a diagonal matrix. 
\end{defn}
\begin{thm} \textbf{(T-SVD)} \label{thmtsvd}
	Let $\A\in\Rn$. Then it can be factorized as
	\begin{equation}\label{eq_tsvd}
		\A=\U*\Sbm*\V^*,
	\end{equation}
	where $\U\in \mathbb{R}^{n_1\times n_1\times n_3}$, $\V\in\mathbb{R}^{n_2\times n_2\times n_3}$ are orthogonal, and $\Sbm\in\mathbb{R}^{n_1\times n_2\times n_3}$ is an f-diagonal tensor. 
\end{thm}
\begin{proof} 
	The proof is by construction. Recall that (\ref{dftpro}) holds and $\Ambar^{(i)}$'s satisfy the property (\ref{keyprofffttensor}). Then we construct the SVD of each $\Ambar^{(i)}$ in the following way. For $i=1,\cdots, \lceil \frac{n_3+1}{2}\rceil$, let $\Ambar^{(i)} = \Umbar^{(i)}\Smbar^{(i)}(\Vmbar^{(i)})^*$ be the full SVD of $\Ambar^{(i)}$. Here the singular values in $\Smbar^{(i)}$ are real. For $i=\lceil \frac{n_3+1}{2}\rceil+1,\cdots, n_3$, let $\Umbar^{(i)}=\conj(\Umbar^{(n_3-i+2)})$, $\Smbar^{(i)}=\Smbar^{(n_3-i+2)}$ and $\Vmbar^{(i)}=\conj(\Vmbar^{(n_3-i+2)})$. Then, it is easy to verify that $\Ambar^{(i)} = \Umbar^{(i)}\Smbar^{(i)}(\Vmbar^{(i)})^*$ gives the full SVD of $\Ambar^{(i)}$ for $i=\lceil \frac{n_3+1}{2}\rceil+1,\cdots, n_3$. Then,
	\begin{equation}\label{profTSVDeqn1}
		\Ambar = \Umbar \Smbar \Vmbar^*.
	\end{equation}
	By the construction of $\Umbar$, $\Smbar$ and $\Vmbar$, and Lemma \ref{lem_keyprofft}, we have that $(\F_{n_3}^{-1}\otimes \bm{I}_{n_1}) \cdot \Umbar \cdot (\F_{n_3}\otimes \bm{I}_{n_1})$, $(\F_{n_3}^{-1}\otimes \bm{I}_{n_1}) \cdot \Smbar \cdot (\F_{n_3}\otimes \bm{I}_{n_2})$ and $(\F_{n_3}^{-1}\otimes \bm{I}_{n_2}) \cdot \Vmbar \cdot (\F_{n_3}\otimes \bm{I}_{n_2})$ are real block circulant matrices. Then we can obtain an expression for $\bcirc(\A)$ by applying the appropriate matrix $(\F_{n_3}^{-1}\otimes \bm{I}_{n_1})$ to the left and the appropriate matrix $(\F_{n_3}\otimes \bm{I}_{n_2})$ to the right of each of the matrices in (\ref{profTSVDeqn1}), and folding up the result. This gives a decomposition of the form $\U*\Sbm*\V^*$, where $\U$, $\Sbm$ and $\V$ are real. 	
\end{proof}

\begin{figure}[!t]
	\centering
	\includegraphics[width=0.35\textwidth]{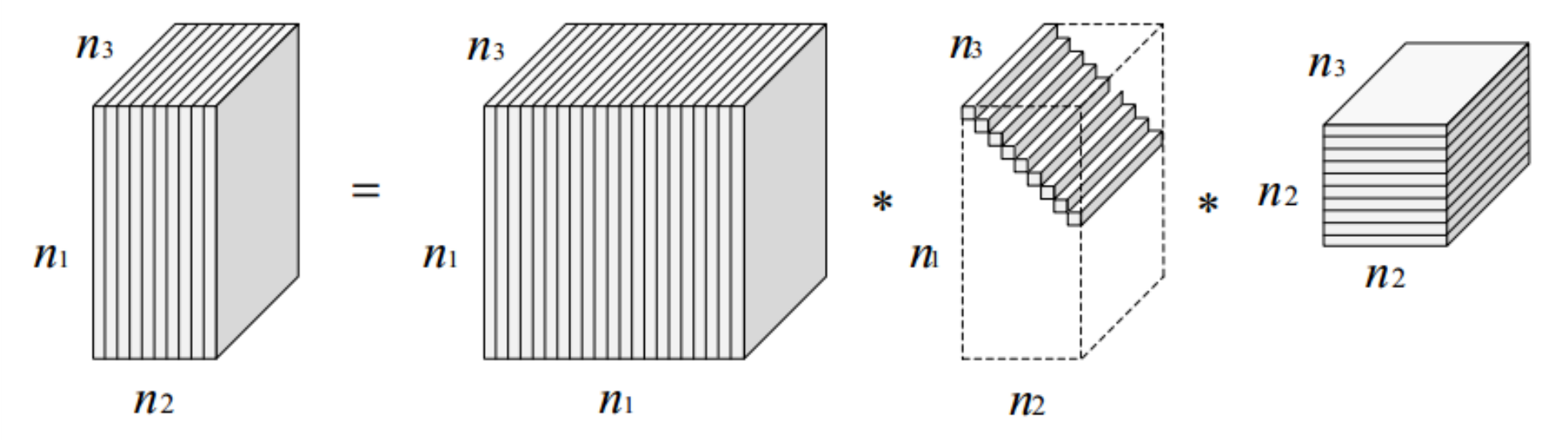}
	\caption{\small An illustration of the t-SVD of an $n_1\times n_2\times n_3$ tensor \cite{hao2013facial}.}
	\label{fig_tsvd}
	\vspace{-0.4cm}
\end{figure}

Theorem \ref{thmtsvd} shows that any 3 way tensor can be factorized into 3 components, including 2 orthogonal tensors and an f-diagonal tensor. See Figure \ref{fig_tsvd} for an intuitive illustration of the t-SVD factorization. T-SVD reduces to the matrix SVD when $n_3=1$. We would like to emphasize that the result of Theorem \ref{thmtsvd} was given first in \cite{kilmer2011factorization} and later in some other related works \cite{hao2013facial,martin2013order}. But their proof and the way for computing $\U$ and $\V$ are not rigorous. The issue is that their method cannot guarantee that $\U$ and $\V$ are real tensors. They construct each frontal slice $\Umbar^{(i)}$ (or $\Vmbar^{(i)}$) of $\Ubar$ (or $\Vbar$ resp.) from the SVD of $\Ambar^{(i)}$ independently for all $i=1,\cdots,n_3$. However, the matrix SVD is not unique. Thus, $\Umbar^{(i)}$'s and $\Vmbar^{(i)}$'s may not satisfy property 
(\ref{keyprofffttensor}) even though $\Ambar^{(i)}$'s do. In this case, the obtained $\U$ (or $\V$) from the inverse DFT of $\Ubar$ (or $\Vbar$ resp.) may not be real. Our proof above instead uses property (\ref{keyprofffttensor}) to construct $\U$ and $\V$ and thus avoids this issue. Our proof further leads to a more efficient way for computing t-SVD shown in Algorithm \ref{alg_tsvd}. 
\begin{algorithm}[!h]
	\caption{T-SVD}
	\textbf{Input:} $\A\in\Rn$.\\
	\textbf{Output:} T-SVD components $\U$, $\Sbm$ and $\V$ of $\A$.
	\begin{enumerate}[1.]
		\item Compute $\Abar=\fft(\A,[\ ],3)$.
		\item Compute each frontal slice of $\Ubar$, $\Sbar$ and $\Vbar$ from $\Abar$ by
		
		\textbf{for} $i=1,\cdots, \lceil 		 \frac{n_3+1}{2}\rceil$ \textbf{do}
		
		\hspace*{0.4cm} $[\Umbar^{(i)},\Smbar^{(i)},\Vmbar^{(i)}] = \text{SVD}(\Ambar^{(i)})$;
		
		\textbf{end for}
		
		\textbf{for} $i=\lceil \frac{n_3+1}{2}\rceil+1,\cdots, n_3$ \textbf{do}
		
		\hspace*{0.4cm}$\Umbar^{(i)} = \conj(\Umbar^{(n_3-i+2)})$;
		 
		\hspace*{0.4cm}$\Smbar^{(i)} = \Smbar^{(n_3-i+2)}$;
		
		\hspace*{0.4cm}$\Vmbar^{(i)} = \conj(\Vmbar^{(n_3-i+2)})$;
		
		\textbf{end for}
		
		\item Compute $\U=\ifft(\Ubar,[\ ],3)$, $\Sbm=\ifft(\Sbar,[\ ],3)$, and $\V=\ifft(\Vbar,[\ ],3)$.
	\end{enumerate}
	\label{alg_tsvd}	
\end{algorithm} 

It is known that the singular values of a matrix have the decreasing order property. Let $\A=\U*\Sbm*\V^*$ be the t-SVD of $\A\in\Rn$. The entries on the diagonal of the first frontal slice $\Sbm(:,:,1)$ of $\Sbm$ have the same decreasing property, \textit{i.e.}, 
\begin{align}\label{eqndecreasingvalue}
	\Sbm(1,1,1)\geq \Sbm(2,2,1)\geq \cdots \geq \Sbm(n',n',1) \geq 0,
\end{align}
where $n' = \min(n_1,n_2)$. The above property holds since the inverse DFT gives 
\begin{equation}\label{prosandsbar}
	\Sbm(i,i,1)=\frac{1}{n_3}\sum_{j=1}^{n_3}\Sbar(i,i,j),
\end{equation}
and the entries on the diagonal of $\Sbar(:,:,j)$ are the singular values of $\Abar(:,:,j)$. As will be seen in Section \ref{sec_TNN}, the tensor nuclear norm depends only on the first frontal slice $\Sbm(:,:,1)$. Thus, we call the entries on the diagonal of $\Sbm(:,:,1)$ as the singular values of $\A$.
\begin{defn} \textbf{(Tensor tubal rank)} \cite{kilmer2013third,zhang2014novel}
	For $\A\in\Rn$, the tensor {tubal rank}, denoted as $\rankt(\A)$, is defined as the number of nonzero singular tubes of $\Sbm$, where $\Sbm$ is from the t-SVD of $\A=\U*\Sbm*\V^*$. We can write
	\begin{align*}
		\rankt(\A) = & \#\{i,\Sbm(i,i,:)\neq\bm{0}\}.
	\end{align*} 
\end{defn}
By using property (\ref{prosandsbar}), the tensor tubal rank is determined by the first frontal slice $\Sbm(:,:,1)$ of $\Sbm$, \textit{i.e.},
\begin{align*}
	\rankt(\A) = \#\{i,\Sbm(i,i,1)\neq {0}\}.	\label{eqndefturank}	
\end{align*} 
Hence, the tensor tubal rank is equivalent to the number of nonzero singular values of $\A$. This property is the same as the matrix case. Define $\A_k=\sum_{i=1}^{k}\U(:,i,:)*\Sbm(i,i,:)*\V(:,i,:)^*$ for some $k<\min(n_1,n_2)$. Then $\A_k = \arg\min_{\rankt(\tilde{\A})\leq k} \norm{\A-\tilde{\A}}_F$, so $\A_k$ is the best approximation of $\A$ with the tubal rank at most $k$. It is known that the real color images can be well approximated by low-rank matrices on the three channels independently. If we treat a color image as a three way tensor with each channel corresponding to a frontal slice, then it can be well approximated by a tensor of low tubal rank. A similar observation was found in \cite{hao2013facial} with the application to facial recognition. Figure \ref{fig_lenna} gives an example to show that a color image can be well approximated by a low tubal rank tensor since most of the singular values of the corresponding tensor are relatively small.

In Section \ref{sec_TNN}, we will define a new tensor nuclear norm which is the convex surrogate of the tensor average rank defined as follows. This rank is closely related to the tensor tubal rank.

\begin{defn} \textbf{(Tensor average rank)} 
	For $\A\in\Rn$, the tensor average rank, denoted as $\ranka(\A)$, is defined as
	\begin{align}
		\ranka(\A) = \frac{1}{n_3} \rankm(\bcirc(\A)).
	\end{align}
\end{defn}
The above definition has a factor $\frac{1}{n_3}$. Note that this factor is crucial in this work as it guarantees that the convex envelope of the tensor average rank within a certain set is the tensor nuclear norm defined in Section \ref{sec_TNN}. The underlying reason for this factor is the t-product definition. Each element of $\A$ is repeated $n_3$ times in the block circulant matrix $\bcirc(\A)$ used in the t-product. Intuitively, this factor alleviates such an entries expansion issue.

\begin{figure}[!t]
	\centering
	\begin{subfigure}[b]{0.15\textwidth}
		\centering
		\includegraphics[width=\textwidth]{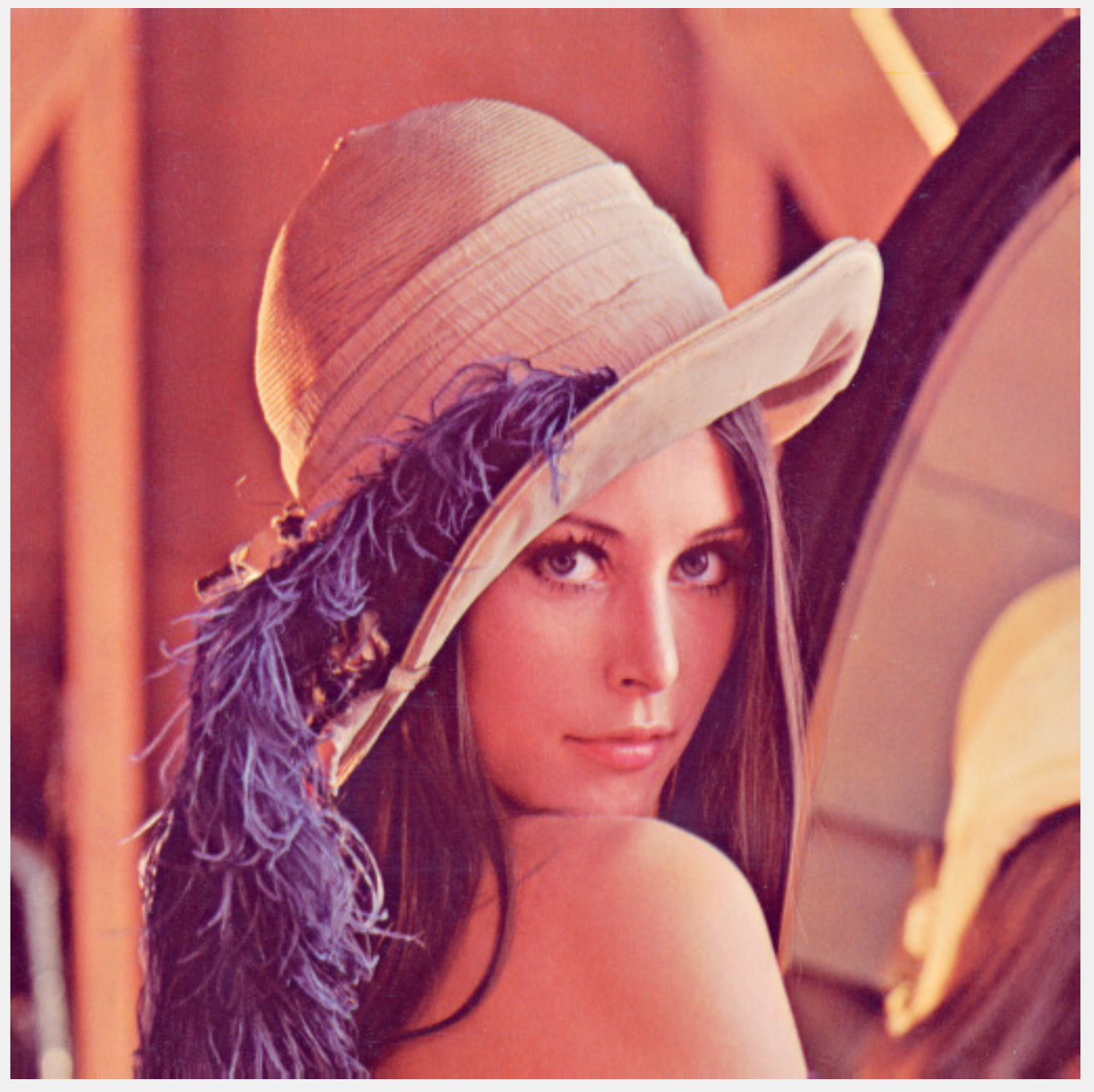}
		\caption{ }
	\end{subfigure} 
	\begin{subfigure}[b]{0.15\textwidth}
		\centering
		\includegraphics[width=\textwidth]{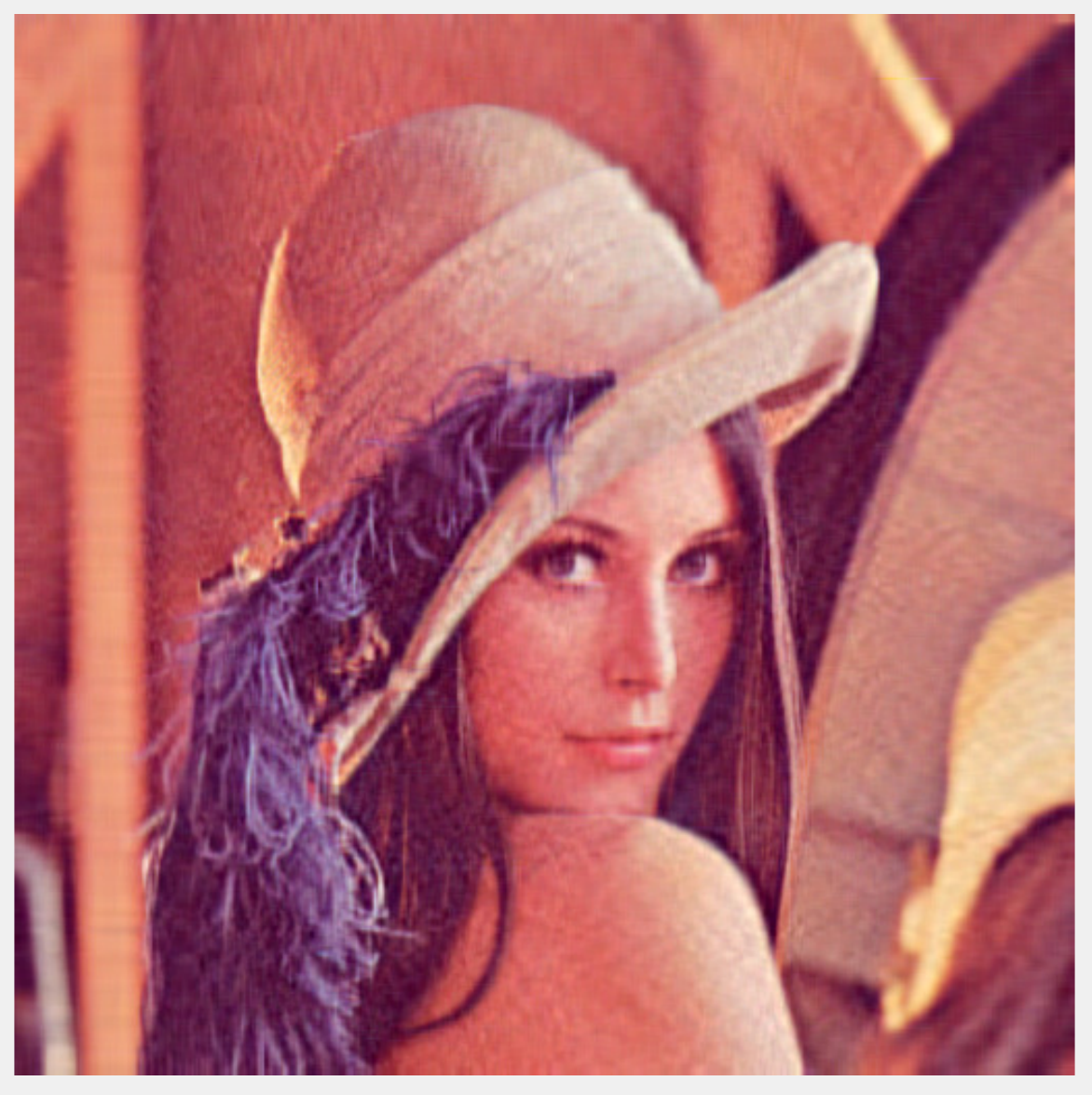}
		\caption{ }
	\end{subfigure}
	\begin{subfigure}[b]{0.165\textwidth}
		\centering
		\includegraphics[width=\textwidth]{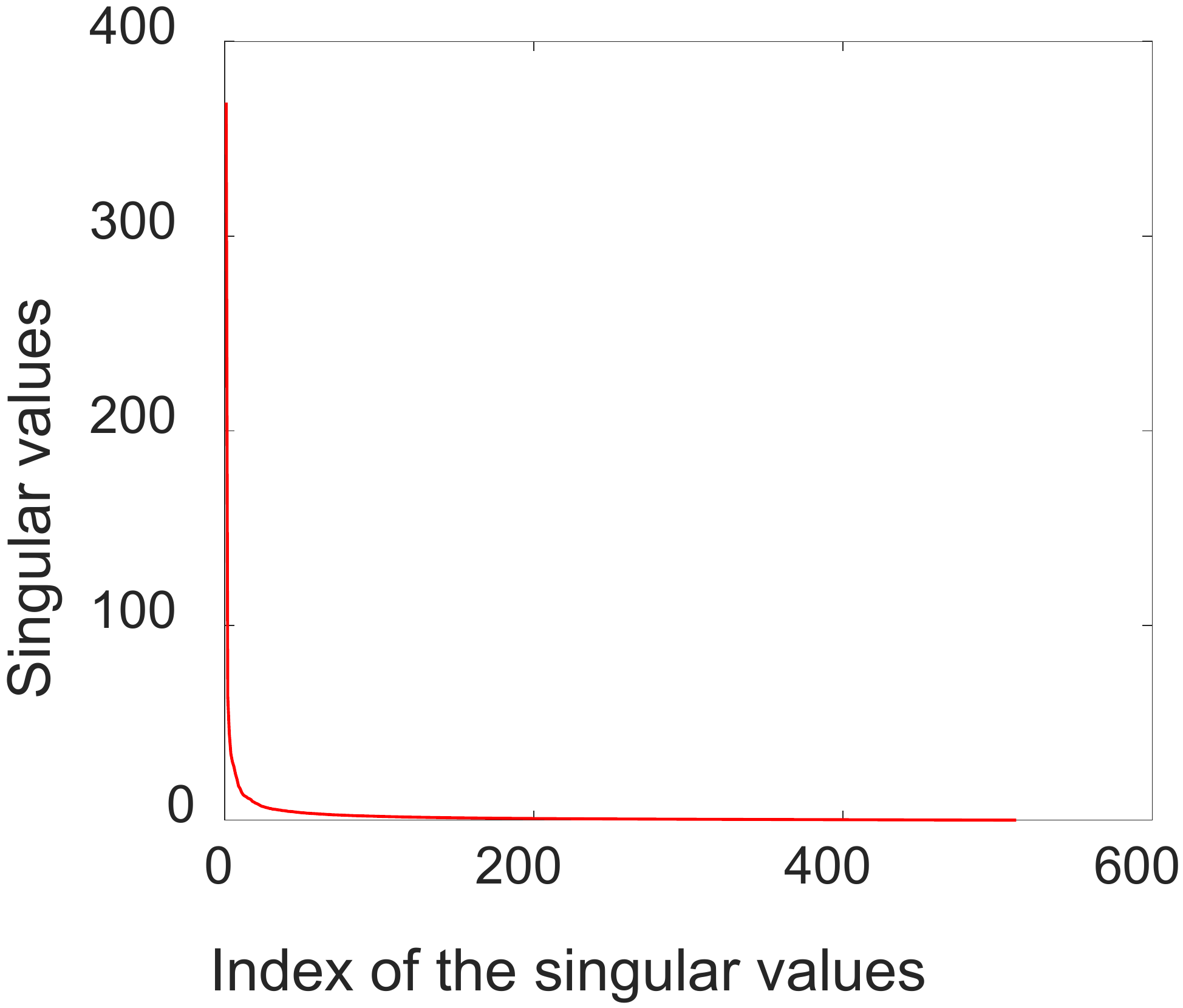}
		\caption{ }
	\end{subfigure}
	\caption{\small{Color images can be approximated by low tubal rank tensors. (a) A color image can be modeled as a tensor $\M \in \mathbb{R}^{512\times 512\times 3}$; (b) approximation by a tensor with tubal rank $r=50$; (c) plot of the singular values of $\M$.}}
	\label{fig_lenna}
\end{figure}

There are some connections between different tensor ranks and these properties imply that the low tubal rank or low average rank assumptions are reasonable for their applications in real visual data. 
First, $\ranka(\A) \leq \rankt(\A)$. Indeed, 
\begin{equation*}
	\ranka(\A) = \frac{1}{n_3} \rankm(\Ambar) \leq \max_{i=1,\cdots,n_3} \rankm(\Ambar^{(i)}) = \rankt(\A),
\end{equation*}
where the first equality uses (\ref{dftpro}). This implies that a low tubal rank tensor always has low average rank. Second, let $\text{rank}_{\text{tc}}(\A)=\left( \text{rank}(\Am^{\{1\}}), \text{rank}(\Am^{\{2\}}), \text{rank}(\Am^{\{3\}}) \right)$, where $\Am^{\{i\}}$ is the mode-$i$ matricization of $\A$, be the Tucker rank of $\A$. Then $\ranka(\A)\leq \rankm(\Am^{\{1\}})$. This implies that a tensor with low Tucker rank has low average rank. The low Tucker rank assumption used in some  applications, \textit{e.g.}, image completion \cite{liu2013tensor}, is applicable to the low average rank assumption. Third, if the CP rank of $\A$ is $r$, then its tubal rank is at most $r$ \cite{zhang2015exact}. Let $\A=\sum_{i=1}^{r} \aa_i^{(1)} \circ \aa_i^{(2)} \circ \aa_i^{(3)}$, where $\circ$ denotes the outer product, be the CP decomposition of $\A$. Then $\Abar=\sum_{i=1}^{r} \aa_i^{(1)} \circ \aa_i^{(2)} \circ \bar{\aa}_i^{(3)}$, where $\bar{\aa}_i^{(3)} = \fft(\aa_i^{(3)})$. So $\Abar$ has the CP rank at most $r$, and each frontal slice of $\Abar$ is the sum of $r$ rank-1 matrices. Thus, the tubal rank of $\A$ is at most $r$. In summary, we show that the low average rank assumption is weaker than the low Tucker rank and low CP rank assumptions. 

\section{Tensor Nuclear Norm (TNN)}
\label{sec_TNN}

In this section, we propose a new tensor nuclear norm which is a convex surrogate of tensor average rank. Based on t-SVD, one may have many different ways to define the tensor nuclear norm intuitively. We give a new and rigorous way to deduce the tensor nuclear norm from the t-product, such that the concepts and their properties are consistent with the matrix cases. This is important since it guarantees that the theoretical analysis of the tensor nuclear norm based tensor RPCA model in Section \ref{sec_tcTNN} can be done in a similar way to RPCA. Figure \ref{fig_tensoroperators} summarizes the way for the new definitions and their relationships. It begins with the known operator norm \cite{atkinson2009theoretical} and t-product. First, the tensor spectral norm is induced by the tensor operator norm by treating the t-product as an operator. Then the tensor nuclear norm is defined as the dual norm of the tensor spectral norm. Finally, we show that the tensor nuclear norm is the convex envelope of the tensor average rank within the unit ball of the tensor spectral norm.

\begin{figure}[!t]
	\centering
	\includegraphics[width=0.42\textwidth]{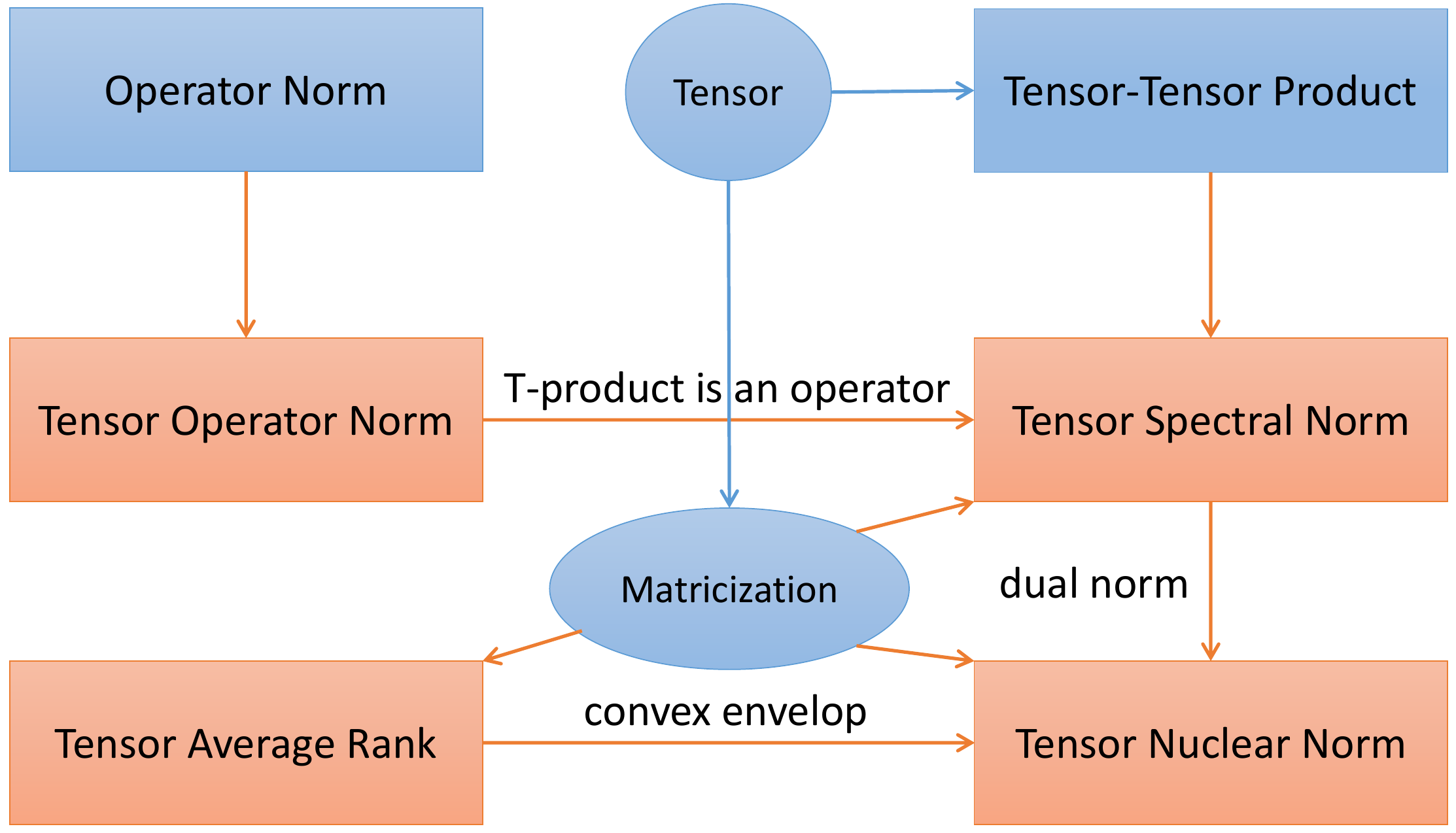}
	\caption{\small An illustration of the way to define the tensor nuclear norm and the relationship with other tensor concepts. First, the tensor operator norm is a special case of the known operator norm performed on the tensors. The tensor spectral norm is induced by the tensor operator norm by treating the tensor-tensor product as an operator. Then the tensor nuclear norm is defined as the dual norm of the tensor spectral norm. We also define the tensor average rank and show that its convex envelope is the tensor nuclear norm within the unit ball of the tensor spectral norm. As detailed in Section \ref{sec_TNN}, the tensor spectral norm, tensor nuclear norm and tensor average rank are also defined on the matricization of the tensor.}
	\label{fig_tensoroperators}
\end{figure}

Let us first recall the concept of operator norm \cite{atkinson2009theoretical}. Let $(V,\norm{\cdot}_V)$ and $(W,\norm{\cdot}_W)$ be normed linear spaces and $L: V\rightarrow W$ be the bounded linear operator between them, respectively. The operator norm of $L$ is defined as
\begin{align}\label{eq_generaloperatornnorm}
	\norm{L} = \sup_{\norm{\vv}_V\leq1} \norm{L(\vv)}_W.
\end{align}
Let $V=\mathbb{C}^{n_2}$, $W=\mathbb{C}^{n_1}$ and $L(\vv) = \Am\vv$, $\vv\in V$, where $\Am\in\mathbb{C}^{n_1\times n_2}$. Based on different choices of $\norm{\cdot}_V$ and $\norm{\cdot}_W$, many matrix norms can be induced by the operator norm in (\ref{eq_generaloperatornnorm}). For example, if $\norm{\cdot}_V$ and $\norm{\cdot}_W$ are $\norm{\cdot}_F$, then the operator norm (\ref{eq_generaloperatornnorm}) reduces to the matrix spectral norm.

Now, consider the normed linear spaces $(V,\norm{\cdot}_F)$ and $(W,\norm{\cdot}_F)$, where $V=\mathbb{R}^{n_2\times 1\times n_3}$, $W=\mathbb{R}^{n_1\times 1\times n_3}$, and $\LL: V\rightarrow W$ is a bounded linear operator. In this case, (\ref{eq_generaloperatornnorm}) reduces to the tensor operator norm
\begin{align}\label{tensoroperatornorm}
	\norm{\LL} = \sup_{\norm{\V}_F\leq1} \norm{\LL(\V)}_F.
\end{align}
As a special case, if $\LL(\V) = \A*\V$, where $\A\in\Rn$ and $\V\in V$, then the tensor operator norm (\ref{tensoroperatornorm}) gives the tensor spectral norm, denoted as $\norm{\A}$,
\begin{align}
	\norm{\A} 
	:= & \sup_{\norm{\V}_F\leq1} \norm{\A*\V}_F \notag\\
	 = & \sup_{\norm{\V}_F\leq1} \norm{\bcirc(\A)\cdot\unfold(\V)}_F \label{eq_deriveoper1}\\
	 = & \norm{\bcirc(\A)} \label{eq_deriveoper2},
\end{align}
where (\ref{eq_deriveoper1}) uses (\ref{tproducdef}), and (\ref{eq_deriveoper2}) uses the definition of matrix spectral norm.
\begin{defn}
	\textbf{(Tensor spectral norm)} The tensor spectral norm of $\A\in\Rn$ is defined as $\norm{\A} := \norm{\bcirc(\A)}$.
\end{defn}
\begin{table*}[!t]
	\centering
	\caption{Parallelism of sparse vector, low-rank matrix and low-rank tensor.}\label{tab_sparselowmlowt}
	\begin{tabular}{c||c|c|c}
		\hline
					   & Sparse vector     & Low-rank matrix  & Low-rank tensor (this work)\\ \hline
		Degeneracy of  & 1-D signal  $\x\in\mathbb{R}^n$      & 2-D correlated signals $\Xm\in\mathbb{R}^{n_1\times n_2}$ & 3-D correlated signals $\X\in\mathbb{R}^{n_1\times n_2\times n_3}$\\ \hline
		Parsimony concept & cardinality    & rank   		  & tensor average rank{\scriptsize\footnotemark}\\ \hline
		Measure  	   & $\ell_0$-norm $\norm{\x}_0$   & $\text{rank}(\Xm)$   & $\ranka(\X)$ \\ \hline
		Convex surrogate  & $\ell_1$-norm $\norm{\x}_1$ & nuclear norm $\norm{\Xm}_*$ & tensor nuclear norm $\norm{\X}_*$ \\ \hline
		Dual norm      & $\ell_\infty$-norm $\norm{\x}_\infty$ & spectral norm $\norm{\Xm}$ & tensor spectral norm $\norm{\X}$ \\ \hline
	\end{tabular}\\
	\footnotesize\footnotemark[2]{Strictly speaking, the tensor tubal rank, which bounds the tensor average rank, is also the parsimony concept of the low-rank tensor.} 
	\vspace{-0.5cm}
\end{table*}
By (\ref{eq_Fnorthogonal}) and (\ref{dftpro}), we have 
\begin{align}\label{eq_protsc}
	\norm{\A} = \norm{\bcirc(\A)} = \norm{\Ambar}.
\end{align}
This property is frequently used in this work. It is known that the matrix nuclear norm is the dual norm of the matrix spectral norm. Thus, we define the tensor nuclear norm, denoted as $\norm{\A}_*$, as the dual norm of the tensor spectral norm. For any $\B\in\Rn$ and $\Bmtilde\in\mathbb{C}^{n_1n_3\times n_2n_3}$, we have
\begin{align}
	\norm{\A}_* 
	:= & \sup_{\norm{\B}\leq1} \langle{\A},{\B}\rangle \label{eq_drivenuclearnorm0}\\
	 =&\sup_{\norm{\Bmbar}\leq1}\frac{1}{n_3}\langle{\Ambar},{\Bmbar}\rangle\label{eq_drivenuclearnorm1}\\
	 \leq &\frac{1}{n_3} \sup_{\norm{\Bmtilde}\leq1} |\langle{\Ambar},{\Bmtilde}\rangle|\label{eq_drivenuclearnorm2} \\
	 = &\frac{1}{n_3}\norm{\Ambar}_*,\label{eq_drivenuclearnorm3}\\
	 = & \frac{1}{n_3}\norm{\bcirc(\A)}_*,\label{eq_drivenuclearnorm4}
\end{align}
where (\ref{eq_drivenuclearnorm1}) is from (\ref{eq_proinproduct}), (\ref{eq_drivenuclearnorm2}) is due to the fact that $\Bmbar$ is a block diagonal matrix in $\mathbb{C}^{n_1n_3\times n_2n_3}$ while $\Bmtilde$ is an arbitrary matrix in $\mathbb{C}^{n_1n_3\times n_2n_3}$, (\ref{eq_drivenuclearnorm3}) uses the fact that the matrix nuclear norm is the dual norm of the matrix spectral norm, and (\ref{eq_drivenuclearnorm4}) uses (\ref{dftpro}) and~(\ref{eq_Fnorthogonal}). Now we show that there exists $\B\in\Rn$ such that the equality (\ref{eq_drivenuclearnorm2}) holds and thus $\norm{\A}_* = \frac{1}{n_3}\norm{\bcirc(\A)}_*$. Let $\A=\U*\Sbm*\V^*$ be the t-SVD of $\A$ and $\B=\U*\V^*$. We have
\begin{align}
	\langle{\A},{\B}\rangle 
	= & \langle{\U*\Sbm*\V^*},{\U*\V^*}\rangle \label{eqntnn1}\\
	= & \frac{1}{n_3} \inproduct{\overline{\U*\Sbm*\V^*}}{\overline{\U*\V^*}} \notag\\
	= & \frac{1}{n_3} \inproduct{\Umbar\Smbar\Vmbar^*}{\Umbar\Vmbar^*} = \frac{1}{n_3} \Tr(\Smbar)\notag\\
	= & \frac{1}{n_3} \norm{\Ambar}_* = \frac{1}{n_3} \norm{\bcirc(\A)}_*. \label{eqntnn2}
\end{align}
Combining (\ref{eq_drivenuclearnorm0})-(\ref{eq_drivenuclearnorm4}) and (\ref{eqntnn1})-(\ref{eqntnn2}) leads to $\norm{\A}_* = \frac{1}{n_3} \norm{\bcirc(\A)}_*$. On the other hand, by (\ref{eqntnn1})-(\ref{eqntnn2}), we have 
\begin{align}
	\norm{\A}_* 
	= & \langle{\U*\Sbm*\V^*},{\U*\V^*}\rangle \notag\\
 	= & \langle{\U^**\U*\Sbm},{\V^**\V}\rangle \notag\\
 	= & \langle{\Sbm},\I\rangle = \sum_{i=1}^{r} \Sbm(i,i,1), \label{defsipro}
\end{align}
where $r=\rankt(\A)$ is the tubal rank. Thus, we have the following definition of tensor nuclear norm.
\begin{defn}\label{defntnnours}
	\textbf{(Tensor nuclear norm)} Let $\A=\U*\Sbm*\V^*$ be the t-SVD of $\A\in\Rn$. The tensor nuclear norm of $\A$ is defined as
	\begin{equation*}
		\norm{\A}_*:= \inproduct{\Sbm}{\I} = \sum_{i=1}^{r} \Sbm(i,i,1),
	\end{equation*}
	where $r=\rankt(\A)$.
\end{defn}
From (\ref{defsipro}), it can be seen that only the information in the first frontal slice of $\Sbm$ is used when defining the tensor nuclear norm. Note that this is the first work which directly uses the singular values $\Sbm(:,:,1)$ of a tensor to define the tensor nuclear norm. Such a definition makes it consistent with the matrix nuclear norm. The above TNN definition is also different from existing works \cite{lu2016tensorrpca,zhang2014novel,semerci2014tensor}.
 
It is known that the matrix nuclear norm $\norm{\Am}_*$ is the convex envelope of the matrix rank $\rankm(\Am)$ within the set $\{\Am |\norm{\Am}\leq1\}$~\cite{fazel2002matrix}. Now we show that the tensor average rank and tensor nuclear norm have the same relationship. 
\begin{thm}\label{thm_convexenvelope}
	On the set $\{\A\in\Rn |\norm{\A}\leq1\}$, the convex envelope of the tensor average rank $\ranka(\A)$ is the tensor nuclear norm $\norm{\A}_*$.
\end{thm}
We would like to emphasize that the proposed tensor spectral norm, tensor nuclear norm and tensor ranks are not arbitrarily defined. They are rigorously induced by the t-product and t-SVD. These concepts and their relationships are consistent with the matrix cases. This is important for the proofs, analysis and computation in optimization. Table \ref{tab_sparselowmlowt} summarizes the parallel concepts in sparse vector, low-rank matrix and low-rank tensor. With these elements in place, the existing proofs of low-rank matrix recovery provide a template for the more general case of low-rank tensor recovery.
 
Also, from the above discussions, we have the property
\begin{align}\label{eqpropertytnn1}
	\norm{\A}_* = \frac{1}{n_3} \norm{\bcirc(\A)}_* = \frac{1}{n_3} \norm{\Ambar}_*. 
\end{align}
It is interesting to understand the tensor nuclear norm from the perspectives of $\bcirc(\A)$ and $\Ambar$. The block circulant matrix can be regarded as a new way of matricization of $\A$ in the original domain. The frontal slices of $\A$ are arranged in a circulant way, which is expected to preserve more spacial relationships across frontal slices, compared with previous matricizations along a single dimension. Also, the block diagonal matrix $\Ambar$ can be regarded as a matricization of $\A$ in the Fourier domain. Its blocks on the diagonal are the frontal slices of $\Abar$, which contains the information across frontal slices of $\A$ due to the DFT on $\A$ along the third dimension. So $\bcirc(\A)$ and $\Ambar$ play a similar role to matricizations of $\A$ in different domains. Both of them capture the spacial information within and across frontal slices of $\A$. This intuitively supports our tensor nuclear norm definition.

Let $\Am=\Um\Sm\Vm^*$ be the skinny SVD of $\Am$. It is known that any subgradient of the nuclear norm at $\Am$ is of the form $\Um\Vm^*+\Wm$, where $\Um^*\Wm=\0$, $\Wm\Vm=\0$ and $\norm{\Wm}\leq 1$ \cite{watson1992characterization}. Similarly, for $\A\in\Rn$ with tubal rank $r$, we also have the skinny t-SVD, \textit{i.e.}, $\A=\U*\Sbm*\V^*$, where $\U\in\mathbb{R}^{n_1\times r\times n_3}$, $\Sbm\in\mathbb{R}^{r\times r\times n_3}$, and $\V\in\mathbb{R}^{n_2\times r\times n_3}$, in which $\U^**\U=\I$ and $\V^**\V=\I$. The skinny t-SVD will be used throughout this paper. With skinny t-SVD, we introduce the subgradient of the tensor nuclear norm, which plays an important role in the proofs. 

\begin{thm}\textbf{(Subgradient of tensor nuclear norm)}\label{thmsubgradient}
	Let $\A\in\Rn$ with $\rankt(\A)=r$ and its skinny t-SVD be $\A=\U*\Sbm*\V^*$. The subdifferential (the set of subgradients) of $\norm{\A}_*$ is $\partial \norm{\A}_*= \{\U*\V^*+\W |\U^**\W=\0, \W*\V=\0,\norm{\W}\leq 1 \}$.
\end{thm}
 
\section{Exact Recovery Guarantee of TRPCA}\label{sec_tcTNN}

With TNN defined above, we now consider the exact recovery guarantee of TRPCA in (\ref{trpca}). The problem we study here is to recover a low tubal rank tensor $\LL_0$ from highly corrupted measurements $\X=\LL_0+\Sbm_0$. In this section, we show that under certain assumptions, the low tubal rank part $\LL_0$ and sparse part $\Sbm_0$ can be exactly recovered by solving convex program (\ref{trpca}). We will also give the optimization detail for solving (\ref{trpca}).

\subsection{Tensor Incoherence Conditions}

Recovering the low-rank and sparse components from their sum suffers from an identifiability issue. For example, the tensor $\X$, with $x_{ijk}=1$ when $i=j=k=1$ and zeros everywhere else, is both low-rank and sparse. One is not able to identify the low-rank component and the sparse component in this case. To avoid such pathological situations, we need to assume that the low-rank component $\LL_0$ is not sparse. To this end, we assume $\LL_0$ to satisfy some incoherence conditions. We denote $\mathring{\mathfrak{e}}_i$ as the tensor column basis, which is a tensor of size $n_1\times 1\times n_3$ with its $(i,1,1)$-th entry equaling  1 and the rest equaling 0 \cite{zhang2015exact}. We also define the tensor tube basis $\dot{\mathfrak{e}}_k$, which is a tensor of size $1\times 1\times n_3$ with its $(1,1,k)$-th entry equaling  1 and the rest equaling 0.

\begin{defn} \textbf{(Tensor Incoherence Conditions)} 
	For $\LL_0\in\Rn$, assume that $\rankt(\LL_0)=r$ and it has the skinny t-SVD $\LL_0=\U*\Sbm*\V^*$, where $\U\in\mathbb{R}^{n_1\times r\times n_3}$ and
	$\V\in\mathbb{R}^{n_2\times r\times n_3}$ satisfy $\U^**\U=\I$ and $\V^**\V=\I$, and 
	$\Sbm\in\mathbb{R}^{r\times r\times n_3}$ is an f-diagonal tensor. Then $\LL_0$ is said to satisfy the tensor incoherence conditions with parameter $\mu$ if
	\begin{align}
		\max_{i=1,\cdots,n_1} \norm{\U^**\mathring{\mathfrak{e}}_i}_F\leq\sqrt{\frac{\mu r}{n_1n_3}}, \label{tic1}\\
		\max_{j=1,\cdots,n_2} \norm{\V^**\mathring{\mathfrak{e}}_j}_F\leq\sqrt{\frac{\mu r}{n_2n_3}},\label{tic2}
	\end{align}
	\begin{equation}
		\norm{\U*\V^*}_\infty\leq \sqrt{\frac{\mu r}{n_1n_2n_3^2}}.\label{tic3}
	\end{equation}	
\end{defn}
The exact recovery guarantee of RPCA \cite{RPCA} also requires some incoherence conditions. Due to property (\ref{eq_proFnormNuclear}), conditions (\ref{tic1})-(\ref{tic2}) have equivalent matrix forms in the Fourier domain, and they are intuitively similar to the matrix incoherence conditions (1.2) in \cite{RPCA}. But the joint incoherence condition (\ref{tic3}) is more different from the matrix case (1.3) in \cite{RPCA}, since it does not have an equivalent matrix form in the Fourier domain. As observed in \cite{chen2013incoherence}, the joint incoherence condition is not necessary for low-rank matrix completion. However, for RPCA, it is unavoidable for polynomial-time algorithms. In our proofs, the joint incoherence (\ref{tic3}) condition is necessary. Another identifiability issue arises if the sparse tensor $\Sbm_0$ has low tubal rank. This can be avoided by assuming that the support of $\Sbm_0$ is uniformly distributed.

\subsection{Main Results}
Now we show that the convex program (\ref{trpca}) is able to perfectly recover the low-rank and sparse components. We define $n_{(1)}=\max(n_1,n_2)$ and $n_{(2)}=\min ({n_1,n_2})$.
\begin{thm}\label{thm1}
	Suppose that $\LL_0\in \mathbb{R}^{n\times n\times n_3}$ obeys (\ref{tic1})-(\ref{tic3}). Fix any $\nss$ tensor $\M$ of signs. Suppose that the support set $\Omegat$ of $\Sbm_0$ is uniformly distributed among all sets of cardinality $m$, and that $\sgn{[\Sbm_0]_{ijk}}=[\M]_{ijk}$ for all $(i,j,k)\in\Omegat$. Then, there exist universal constants $c_1, c_2>0$ such that with probability at least $1-c_1(nn_3)^{-c_2}$ (over the choice of support of $\Sbm_0$), $(\LL_0,\Sbm_0)$ is the unique minimizer to (\ref{trpca}) with $\lambda = 1/\sqrt{nn_3}$, provided that
	\begin{equation}
		\rankt(\LL_0)\leq \frac{\rho_r nn_3}{\mu(\log(nn_3))^{2}} \text{ and } m\leq \rho_sn^2n_3,
	\end{equation}
	where $\rho_r$ and $\rho_s$ are positive constants. If $\LL_0\in\Rn$ has rectangular frontal slices, TRPCA with $\lambda = 1/\sqrt{\none n_3}$ succeeds with probability at least $1-c_1(\none n_3)^{-c_2}$, provided that $\rankt(\LL_0)\leq \frac{ \rho_r \ntwo n_3}{\mu(\log(\none n_3))^{2}} \text{ and } m\leq \rho_sn_1n_2 n_3$.
\end{thm}
The above result shows that for incoherent $\LL_0$, the perfect recovery is guaranteed with high probability for $\rankt(\LL_0)$ on the order of $nn_3/(\mu(\log n n_3)^2)$ and a number of nonzero entries in $\Sbm_0$ on the order of $n^2n_3$. For $\Sbm_0$, we make only one assumption on the random location distribution, but no assumption about the magnitudes or signs of the nonzero entries. Also TRPCA is parameter free. The mathematical analysis implies that the parameter $\lambda = 1/\sqrt{nn_3}$ leads to the correct recovery. Moreover, since the t-product of 3-way tensors reduces to the standard matrix-matrix product when $n_3=1$, the tensor nuclear norm reduces to the matrix nuclear norm. Thus, RPCA is a special case of TRPCA and the guarantee of RPCA in Theorem 1.1 in \cite{RPCA} is a special case of our Theorem \ref{thm1}. Both our model and theoretical guarantee are consistent with RPCA. Compared with SNN \cite{huang2014provable}, our tensor extension of RPCA is much more simple and elegant.

The detailed proof of Theorem \ref{thm1} can be found in the supplementary material. It is interesting to understand our proof from the perspective of the following equivalent formulation
\begin{equation}\label{trpcaeqnmixed}
	\min_{\LL,\ \E} \ \frac{1}{n_3}\left(\norm{\Lmbar}_*+\lambda\norm{\bcirc(\E)}_1\right), \ \st \ \X=\LL+\E,
\end{equation}
where (\ref{eqpropertytnn1}) is used. Program (\ref{trpcaeqnmixed}) is a mixed model since the low-rank regularization is performed on the Fourier domain while the sparse regularization is performed on the original domain. Our proof of Theorem \ref{thm1} is also conducted based on the interaction between both domains. By interpreting the tensor nuclear norm of $\LL$ as the matrix nuclear norm of $\Lmbar$ (with a factor $\frac{1}{n_3}$) in the Fourier domain, we are then able to use some existing properties of the matrix nuclear norm in the proofs. The analysis for the sparse term is kept on the original domain since the $\ell_1$-norm has no equivalent form in the Fourier domain. Though both two terms of the objective function of (\ref{trpcaeqnmixed}) are given on two matrices ($\Lmbar$ and $\bcirc(\E)$), the analysis for model (\ref{trpcaeqnmixed}) is very different from that of matrix RPCA. The matrices $\Lmbar$ and $\bcirc(\E)$ can be regarded as two matricizations of the tensor objects $\LL$ and $\E$, respectively. Their structures are more complicated than those in matrix RPCA, and thus make the proofs different from 
\cite{RPCA}. For example, our proofs require proving several bounds of norms on random tensors. Theses results and proofs, which have to use the properties of block circulant matrices and the Fourier transformation, are completely new. Some proofs are challenging due to the dependent structure of $\bcirc(\E)$ for $\E$ with an independent elements assumption. Also, TRPCA is of a different nature from the tensor completion problem \cite{zhang2015exact}. The proof of the exact recovery of TRPCA is more challenging since the $\ell_1$-norm (and its dual norm $\ell_\infty$-norm used in (\ref{tic3})) has no equivalent formulation in the Fourier domain.

It is worth mentioning that this work focuses on the analysis for 3-way tensors. But it is not difficult to generalize our model in (\ref{trpca}) and results in Theorem \ref{thm1} to the case of order-$p$ ($p\geq3$) tensors, by using the t-SVD for order-$p$ tensors in \cite{martin2013order}. 

When considering the application of TRPCA, the way for constructing a 3-way tensor from data is important. The reason is that the t-product is orientation dependent, and so is the tensor nuclear norm. Thus, the value of TNN may be different if the tensor is rotated. For example, a 3-channel color image can be formatted as 3 different sizes of tensors. Therefore, when using TRPCA which is based on TNN, one has to format the data into tensors in a proper way by leveraging some priori knowledge, \textit{e.g.}, the low tubal rank property of the constructed tensor.
	
\begin{algorithm}[!t]
	\caption{Tensor Singular Value Thresholding (t-SVT)}
	\textbf{Input:} $\Y\in\Rn$, $\tau>0$.\\
	\textbf{Output:} $\mathcal{D}_\tau(\Y)$ as defined in (\ref{eqn_tsvt}).
	\begin{enumerate}[1.]
		\item Compute $\Ybar=\fft(\Y,[\ ],3)$.
		\item Perform matrix SVT on each frontal slice of $\Ybar$ by
		
		\textbf{for} $i=1,\cdots, \lceil \frac{n_3+1}{2}\rceil$ \textbf{do}
		
		\hspace*{0.4cm} $[\Um,\Sm,\Vm] = \text{SVD}(\Ymbar^{(i)})$;
		
		\hspace*{0.4cm}
		$\Wmbar^{(i)} = \Um \cdot (\Sm-\tau)_+ \cdot \Vm^*$;
		
		\textbf{end for}
		
		\textbf{for} $i=\lceil \frac{n_3+1}{2}\rceil+1,\cdots, n_3$ \textbf{do}
		
		\hspace*{0.4cm}$\Wmbar^{(i)} = \conj(\Wmbar^{(n_3-i+2)})$;
		
		\textbf{end for}
		
		\item Compute $\mathcal{D}_\tau(\Y)=\ifft(\Wbar,[\ ],3)$.
	\end{enumerate}
	\label{alg_tsvt}	
\end{algorithm} 

\subsection{Tensor Singular Value Thresholding}

Problem (\ref{trpca}) can be solved by the standard Alternating Direction Method of Multiplier (ADMM) \cite{lu2018unified}. A key step is to compute the proximal operator of TNN
\begin{equation}\label{potnn}
	\min_{\X\in\Rn} \ \tau \norm{\X}_*+\frac{1}{2}\norm{\X - \Y}_F^2.
\end{equation}
We show that it also has a closed-form solution as the proximal operator of the matrix nuclear norm. 
Let $\Y= \U * \Sbm *\V^*$ be the tensor SVD of $\Y\in\Rn$. For each $\tau>0$, we define the tensor Singular Value Thresholding (t-SVT) operator as follows
\begin{equation}\label{eqn_tsvt}
	\mathcal{D}_{\tau}(\Y) = \U * \Sbm_\tau *\V^*,
\end{equation}
where 
\begin{equation}\label{staudef}
	\Sbm_\tau = \ifft((\Sbar -\tau)_+,[\ ],3).
\end{equation}
Note that $\Sbar$ is a  real tensor. Above $t_+$ denotes the positive part of $t$, \textit{i.e.}, $t_+ = \max(t,0)$. That is, this operator simply applies a soft-thresholding rule to the singular values $\Sbar$ (not $\Sbm$) of the frontal slices of $\Ybar$, effectively shrinking these towards zero. The t-SVT operator is the proximity operator associated with TNN. 
\begin{thm}\label{thmtsvt}
	For any $\tau>0$ and $\Y \in\Rn$, the tensor singular value thresholding operator (\ref{eqn_tsvt}) obeys 
	\begin{equation}\label{thmeqnsvt}
		\mathcal{D}_{\tau}(\Y) = \arg\min_{\X\in\Rn} \ \tau \norm{\X}_*+\frac{1}{2}\norm{\X - \Y}_F^2.
	\end{equation}
\end{thm}
\begin{proof}
	The required solution to (\ref{thmeqnsvt}) is a real tensor and thus we first show that $\mathcal{D}_{\tau}(\Y)$ in (\ref{eqn_tsvt}) is real. Let $\Y= \U * \Sbm *\V^*$ be the tensor SVD of $\Y$. We know that the frontal slices of $\Sbar$ satisfy the property (\ref{keyprofffttensor}) and so do the frontal slices of $(\Sbar-\tau)_+$. By Lemma \ref{lem_keyprofft}, $\Sbm_\tau$ in (\ref{staudef}) is real. Thus, $\mathcal{D}_{\tau}(\Y)$ in (\ref{eqn_tsvt}) is real. Secondly, by using properties (\ref{eqpropertytnn1}) and (\ref{eq_proFnormNuclear}), problem (\ref{thmeqnsvt}) is equivalent to 
	\begin{align}
		\arg\min_{\X}  & \ \frac{1}{n_3}(\tau\norm{\Xmbar}_*+\frac{1}{2}\norm{\Xmbar-\Ymbar}_F^2) 	\notag \\
		=\arg\min_{\X} & \ \frac{1}{n_3}\sumi(\tau\norm{\Xmbar^{(i)}}_*+\frac{1}{2}\norm{\Xmbar^{(i)}-\Ymbar^{(i)}}_F^2). \label{eqnthmtsvteqnform}
	\end{align}
	By Theorem 2.1 in \cite{cai2010singular}, we know that the $i$-th frontal slice of $\overline{\mathcal{D}_\tau(\Y)}$ solves the $i$-th subproblem of (\ref{eqnthmtsvteqnform}). Hence, $\mathcal{D}_\tau(\Y)$ solves problem (\ref{thmeqnsvt}).
\end{proof}

Theorem \ref{thmtsvt} gives the closed-form of the t-SVT operator $\mathcal{D}_\tau(\Y)$, which is a natural extension of the matrix SVT \cite{cai2010singular}. Note that $\mathcal{D}_\tau(\Y)$ is real when $\Y$ is real. By using property (\ref{keyprofffttensor}),  Algorithm \ref{alg_tsvt} gives an efficient way for computing $\mathcal{D}_\tau(\Y)$.

With t-SVT, we now give the details of ADMM to solve (\ref{trpca}). The augmented Lagrangian function of (\ref{trpca}) is
\begin{align*}
	L(\LL,\E,\Y,\mu) = & \norm{\LL}_* + \lambda\norm{\E}_1+ \inproduct{\Y}{\LL+\E-\X} \\
					   & + \frac{\mu}{2}\norm{\LL+\E-\X}_F^2.
\end{align*}
Then $\LL$ and $\E$ can be updated by minimizing the augmented Lagrangian function $L$ alternately. Both subproblems have closed-form solutions. See Algorithm \ref{alg1} for the whole procedure. The main per-iteration cost lies in the update of $\LL_{k+1}$, which requires computing FFT and $\lceil \frac{n_3+1}{2}\rceil$ SVDs of $n_1\times n_2$ matrices. The per-iteration complexity is $O\left(n_1n_2n_3\log n_3+\none\ntwo^2n_3\right)$.

\begin{algorithm}[t!]
	\caption{Solve (\ref{trpca}) by ADMM}
	\textbf{Initialize:} $\LL_0=\Sbm_0=\Y_0={0}$, $\rho=1.1$, $\mu_0=1\e{-3}$, $\mu_{\max}=1\e{10}$, $\epsilon=1\e{-8}$. \\ 
	\textbf{while} not converged \textbf{do}
	\begin{enumerate}[1.]
		\item Update $\LL_{k+1}$ by
		\begin{equation*}
			\LL_{k+1} = \argmin_{\LL} \ \norm{\LL}_*+\frac{\mu_k}{2}\normlarge{\LL+\E_{k}-\X+\frac{\Y_{k}}{\mu_k}}_F^2;
		\end{equation*}
		\item Update $\E_{k+1}$ by
		\begin{equation*}
			\E_{k+1} = \argmin_{\E} \ \lambda\norm{\E}_1+\frac{\mu_k}{2}\normlarge{\LL_{k+1}+\E-\X+\frac{\Y_{k}}{\mu_k}}_F^2;
		\end{equation*}
		\item $\Y_{k+1}=\Y_{k}+\mu_k(\LL_{k+1}+\E_{k+1}-\X)$;
		\item Update $\mu_{k+1}$ by $\mu_{k+1}=\min(\rho\mu_k,\mu_{\max})$;
		\item Check the convergence conditions
		\begin{align*}
			& \norm{\LL_{k+1}-\LL_{k}}_\infty\leq\epsilon, \ \norm{\E_{k+1}-\E_{k}}_\infty\leq\epsilon, \\
			& \norm{\LL_{k+1}+\E_{k+1}-\X}_\infty\leq\epsilon.
		\end{align*}
	\end{enumerate}
	\textbf{end while}
	\label{alg1}	
\end{algorithm}

\section{Experiments}\label{sec_exp}

In this section, we conduct numerical experiments to verify our main results in Theorem \ref{thm1}. We first investigate the ability of the convex TRPCA model (\ref{trpca}) to recover tensors with varying tubal rank and different levels of sparse noises. We then apply it for image recovery and background modeling. As suggested by Theorem \ref{thm1}, we set $\lambda=1/\sqrt{\none n_3}$ in all the experiments\footnote{Codes of our method available at \url{https://github.com/canyilu}.}. But note that it is possible to further improve the performance by tuning $\lambda$ more carefully. The suggested value in theory provides a good guide in practice. All the simulations are conducted on a PC with an Intel Xeon E3-1270 3.60GHz CPU and 64GB memory.
 

\subsection{Exact Recovery from Varying Fractions of Error}

We first verify the correct recovery guarantee of Theorem \ref{thm1} on randomly generated problems. We simply consider the tensors of size $n\times n \times n$, with varying dimension $n=$100 and 200. We generate a tensor with tubal rank $r$ as a product $\LL_0=\PP*\Q^*$, where $\PP$ and $\Q$ are ${n\times r\times n}$ tensors with entries independently sampled from $\mathcal{N}(0,1/n)$ distribution. The support set $\Omegat$ (with size $m$) of $\E_0$ is chosen uniformly at random. For all $(i,j,k)\in\Omegat$, let $[\E_0]_{ijk}=\M_{ijk}$, where $\M$ is a tensor with independent Bernoulli $\pm 1$ entries.

\begin{table}[]
	\scriptsize 
	\centering
	\caption{\small Correct recovery for random problems of varying sizes.}\vspace{-0.2cm}
	$r=\rankt(\LL_0)=0.05n$, $m=\norm{\E_0}_0=0.05n^3$
	\begin{tabular}{c|c|c|c|c|c|c}
		\hline
		$n$ & $r$ & $m$ & $\rankt(\Lhat)$ & $\|\Shat\|_0$ & $\frac{\norm{\Lhat-\LL_0}_F}{\norm{\LL_0}_F}$ & $\frac{\norm{\Ehat-\E_0}_F}{\norm{\E_0}_F}$ \\ \hline\hline
		100 & 5 & $5\e{4}$ & 5 & 50,029 & $2.6\e{-7}$ & $5.4\e{-10}$\\ \hline
		200 & 10 & $4\e{5}$ & 10 & 400,234 & $5.9\e{-7}$ & $6.7\e{-10}$\\ \hline
	\end{tabular}
	\vspace{0.1cm}
	
	$r=\rankt(\LL_0)=0.05n$, $m=\norm{\E_0}_0=0.1n^3$
	\begin{tabular}{c|c|c|c|c|c|c}
		\hline
		$n$ & $r$ & $m$ & $\rankt(\Lhat)$ & $\|\Shat\|_0$ & $\frac{\norm{\Lhat-\LL_0}_F}{\norm{\LL_0}_F}$ & $\frac{\norm{\Ehat-\E_0}_F}{\norm{\E_0}_F}$    \\ \hline\hline
		100 & 5 & $1\e{5}$ & 5 & 100,117 & $4.1\e{-7}$ & $8.2\e{-10}$\\ \hline
		200 & 10 & $8\e{5}$ & 10 & 800,901 & $4.4\e{-7}$ & $4.5\e{-10}$\\ \hline
	\end{tabular}
	\vspace{0.1cm}
	
	$r=\rankt(\LL_0)=0.1n$, $m=\norm{\E_0}_0=0.1n^3$
	\begin{tabular}{c|c|c|c|c|c|c}
		\hline
		$n$ & $r$ & $m$ & $\rankt(\Lhat)$ & $\|\Shat\|_0$ & $\frac{\norm{\Lhat-\LL_0}_F}{\norm{\LL_0}_F}$ & $\frac{\norm{\Ehat-\E_0}_F}{\norm{\E_0}_F}$        \\ \hline\hline
		100  & 10 & $1\e{5}$ & 10 & 101,952 & $4.8\e{-7}$ & $1.8\e{-9}$\\ \hline
		200  & 20 & $8\e{5}$ & 20 & 815,804  & $4.9\e{-7}$ & $9.3\e{-10}$\\ \hline
	\end{tabular}
	\vspace{0.1cm}
	
	$r=\rankt(\LL_0)=0.1n$, $m=\norm{\E_0}_0=0.2n^3$
	\begin{tabular}{c|c|c|c|c|c|c}
		\hline
		$n$ & $r$ & $m$  & $\rankt(\Lhat)$ & $\|\Ehat\|_0$  & $\frac{\norm{\Lhat-\LL_0}_F}{\norm{\LL_0}_F}$ & $\frac{\norm{\Ehat-\E_0}_F}{\norm{\E_0}_F}$  \\ \hline\hline
		100  & 10 & $2\e{5}$ & 10 & 200,056 & $7.7\e{-7}$ & $4.1\e{-9}$ \\ \hline
		200  & 20 & $16\e{5}$ & 20 & 1,601,008 & $1.2\e{-6}$ & $3.1\e{-9}$ \\ \hline
	\end{tabular}\label{tab_recov}
	\vspace{-0.3cm}
\end{table}

\begin{figure}[!t]
	\centering
	\begin{subfigure}[b]{0.22\textwidth}
		\centering
		\includegraphics[width=\textwidth]{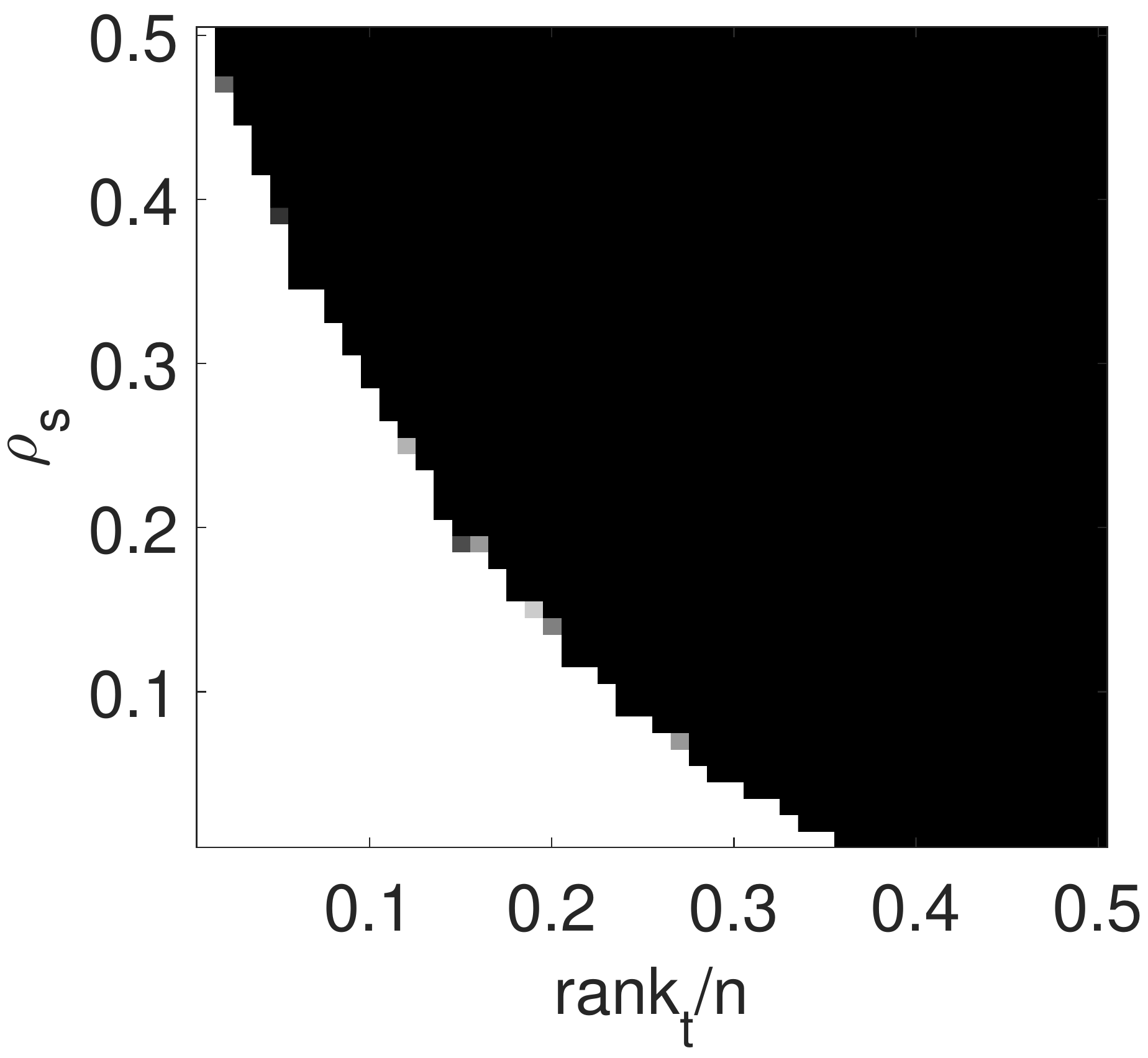}
		\caption{TRPCA, Random Signs}
	\end{subfigure}
	\begin{subfigure}[b]{0.22\textwidth}
		\centering
		\includegraphics[width=\textwidth]{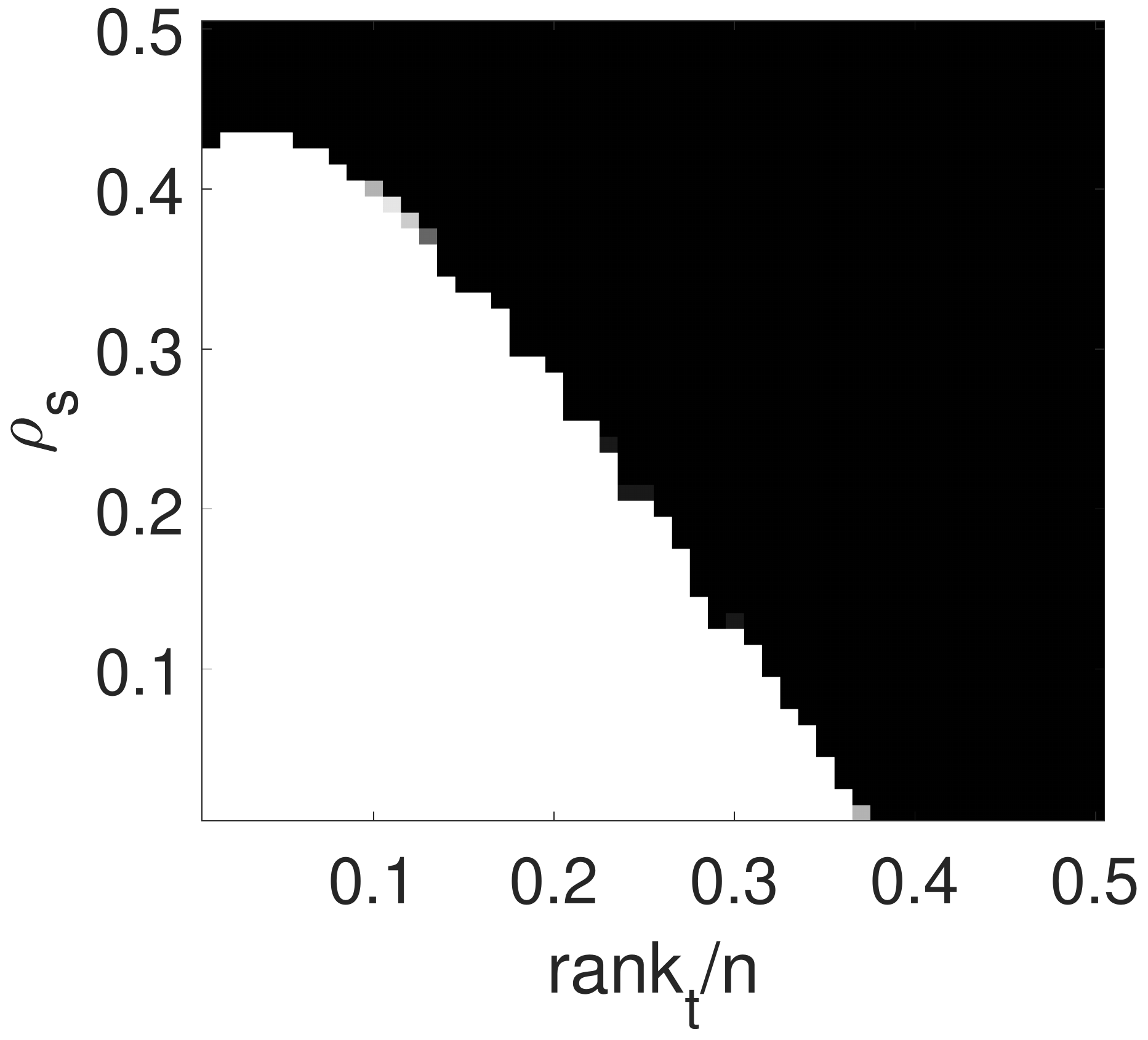}
		\caption{\canyi{TRPCA, Coherent Signs}}
	\end{subfigure}
	\caption{\small{Correct recovery for varying tubal ranks of $\LL_0$ and sparsities of $\E_0$. Fraction of correct recoveries across 10 trials, as a function of $\rankt(\LL_0)$ (x-axis) and sparsity of $\E_0$ (y-axis). \canyi{Left: $\text{sgn}(\E_0)$ random. Right: $\E_0 = \Pomega\text{sgn}(\LL_0)$.}}}
	\label{fig_sparsityvsrank}
	\vspace{-0.4cm}
\end{figure}

Table \ref{tab_recov} reports the recovery results based on varying choices of the tubal rank $r$ of $\LL_0$ and the sparsity $m$ of $\E_0$. It can be seen that our convex program (\ref{trpca}) gives the correct tubal rank estimation of $\LL_0$ in all cases and also the relative errors ${\|{\hat{\LL}-\LL_0}\|_F}/{\norm{\LL_0}_F}$ are very small, less than $10^{-5} $. The sparsity estimation of $\E_0$ is not as exact as the rank estimation, but note that the relative errors ${\|{\Ehat-\E_0}\|_F}/{\norm{\E_0}_F}$ are all very small, less than $10^{-8}$ (actually much smaller than the relative errors of the recovered low-rank component). These results well verify the correct recovery phenomenon as claimed in Theorem \ref{thm1}. 


\begin{figure*}[!t]
	\centering
	\begin{subfigure}[b]{1\textwidth}
		\centering
		\includegraphics[width=\textwidth]{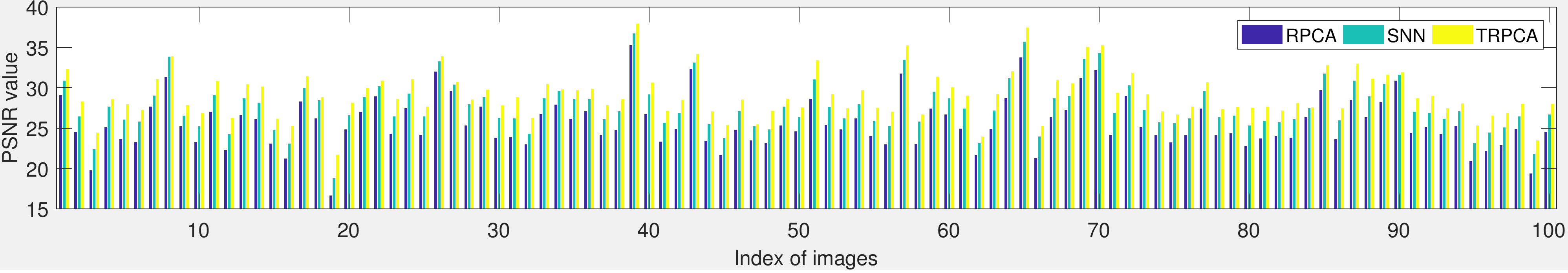}
	\end{subfigure}
	\begin{subfigure}[b]{1\textwidth}
		\centering
		\includegraphics[width=\textwidth]{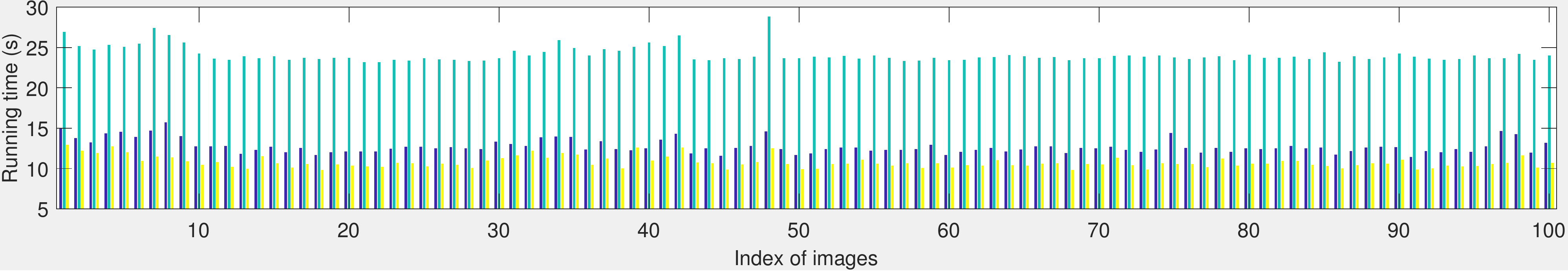}
	\end{subfigure}
	\caption{\small Comparison of the PSNR values (top) and running time (bottom) obtained by RPCA, SNN and TRPCA on 100 images.}
	\label{fig_img_res_psnr_time}
	\vspace{-0.4cm}
\end{figure*}

\subsection{Phase Transition in Tubal Rank and Sparsity}

The results in Theorem \ref{thm1} show the perfect recovery for incoherent tensor with $\rankt(\LL_0)$ on the order of $nn_3/(\mu(\log n n_3)^2)$ and the sparsity of $\E_0$ on the order of $n^2n_3$. Now we examine the recovery phenomenon with varying tubal rank of $\LL_0$ from varying sparsity of $\E_0$. \canyi{We consider the tensor $\LL_0$ of size $\mathbb{R}^{n\times n\times n_3}$, where $n=100$ and $n_3=50$. We generate $\LL_0=\PP*\Q^*$, where $\PP$ and $\Q$ are ${n\times r\times n_3}$ tensors with entries independently sampled from a $\mathcal{N}(0,1/n)$ distribution. For the sparse component $\E_0$, we consider two cases. In the first case, we assume a Bernoulli model for the support of the sparse term $\E_0$, with random signs: each entry of $\E_0$ takes on value 0 with probability $1-\rho$, and values $\pm1$ each with probability $\rho/2$. The second case chooses the support $\bm\Omega$ in accordance with the Bernoulli model, but this time sets $\E_0=\Pomega \text{sgn}(\LL_0)$.} We set $\frac{r}{n}=[0.01:0.01:0.5]$ and $\rho_s=[0.01:0.01:0.5]$. For each $(\frac{r}{n},\rho_s)$-pair, we simulate 10 test instances and declare a trial to be successful if the recovered $\hat{\LL}$ satisfies ${\|{\hat{\LL}-\LL_0}\|_F}/{\norm{\LL_0}_F}\leq 10^{-3}$. Figure \ref{fig_sparsityvsrank} plots the fraction of correct recovery for each pair $(\frac{r}{n},\rho_s)$ (black = $0\%$ and white = 100$\%$). \canyi{It can be seen that there is a large region in which the recovery is correct in both cases.} Intuitively, the experiment shows that the recovery is correct when the tubal rank of $\LL_0$ is relatively low and the errors $\E_0$ is relatively sparse. \canyi{Figure \ref{fig_sparsityvsrank} (b) further shows that the signs of $\E_0$ are  not important: recovery can be guaranteed as long as its support is chosen uniformly at random. These observations are consistent with Theorem \ref{thm1}. Similar observations can be found in the matrix RPCA case (see Figure 1 in \cite{RPCA}).}


\begin{figure*}[!t]
	\centering	
	\begin{subfigure}[b]{0.195\textwidth}
		\centering
		\includegraphics[width=\textwidth]{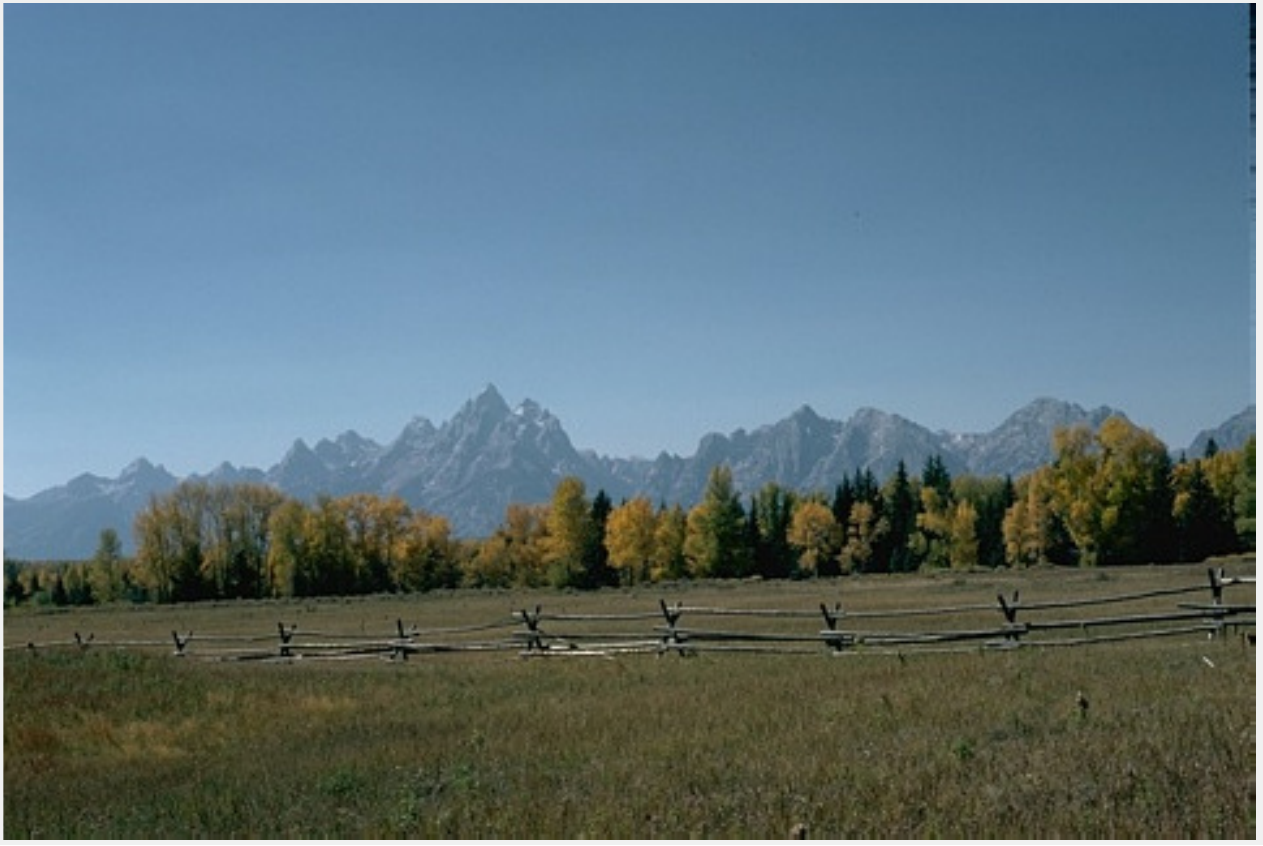}
	\end{subfigure}
	\begin{subfigure}[b]{0.195\textwidth}
		\centering
		\includegraphics[width=\textwidth]{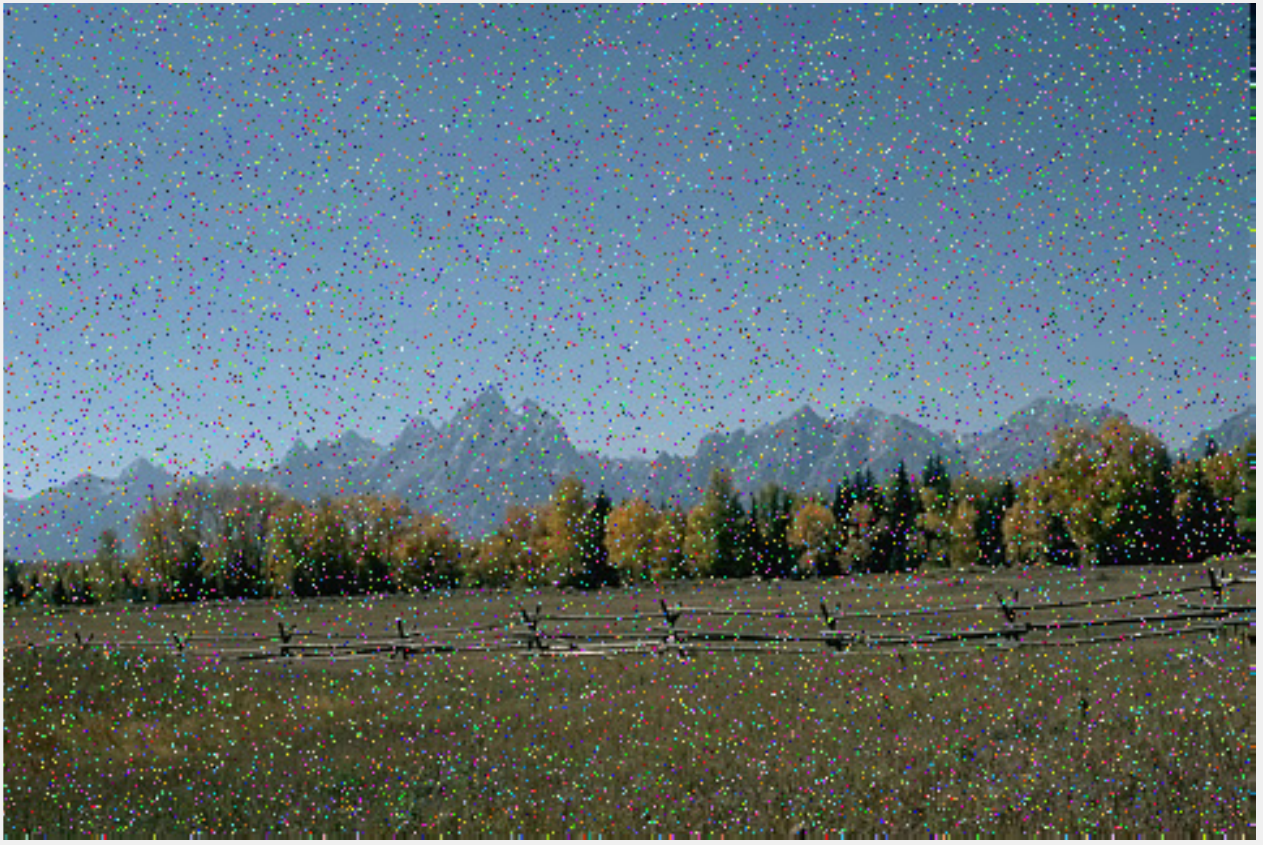}
	\end{subfigure}
	\begin{subfigure}[b]{0.195\textwidth}
		\centering
		\includegraphics[width=\textwidth]{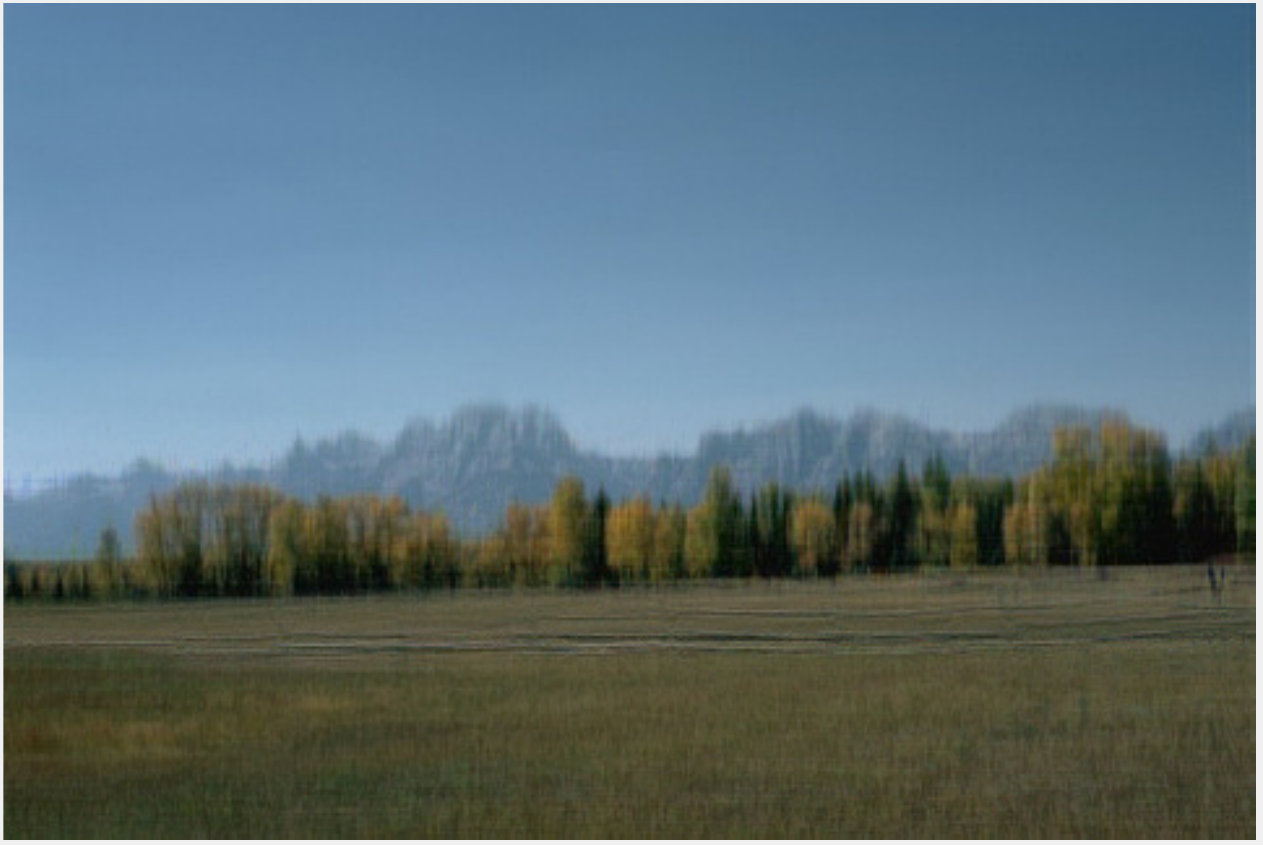}
	\end{subfigure}
	\begin{subfigure}[b]{0.195\textwidth}
		\centering
		\includegraphics[width=\textwidth]{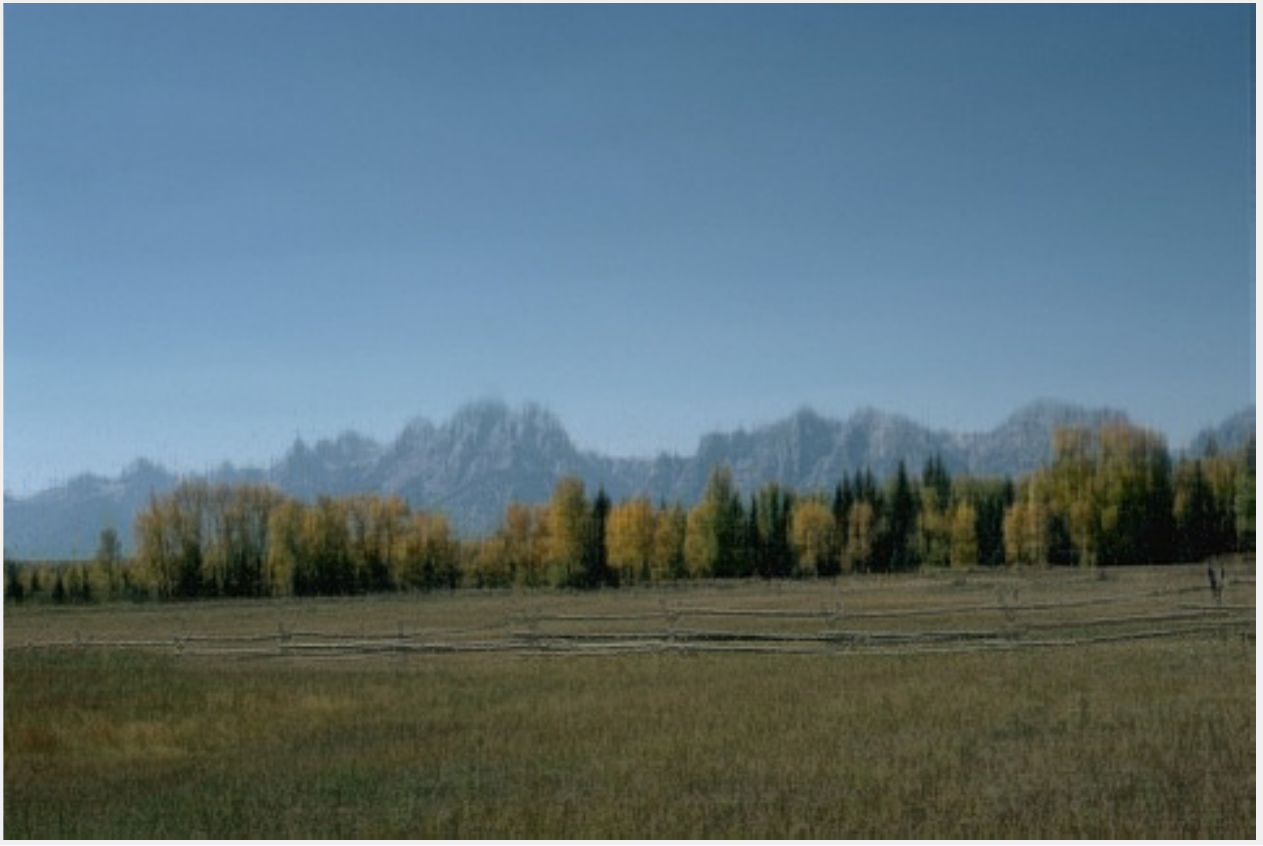}
	\end{subfigure}
	\begin{subfigure}[b]{0.195\textwidth}
		\centering
		\includegraphics[width=\textwidth]{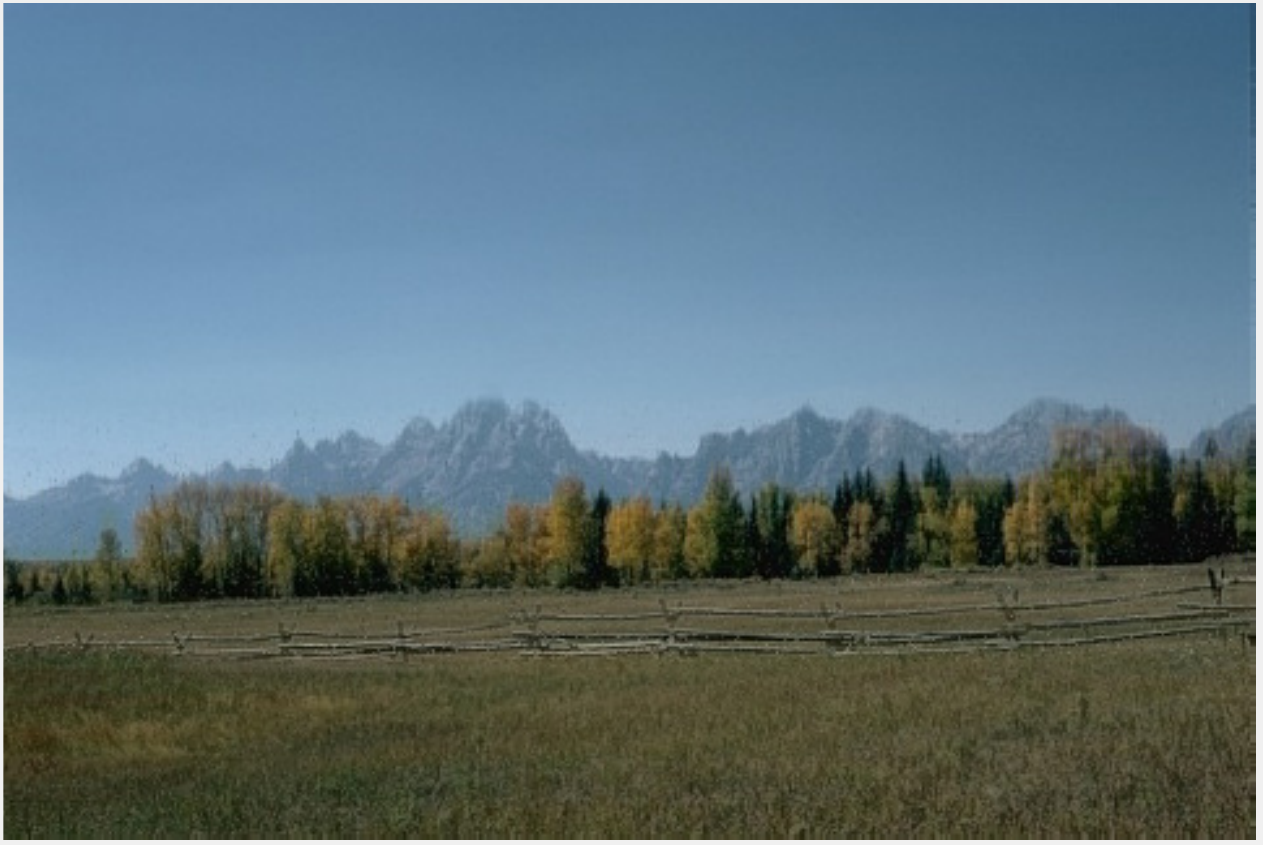}
	\end{subfigure}	
	\begin{subfigure}[b]{0.195\textwidth}
		\centering
		\includegraphics[width=\textwidth]{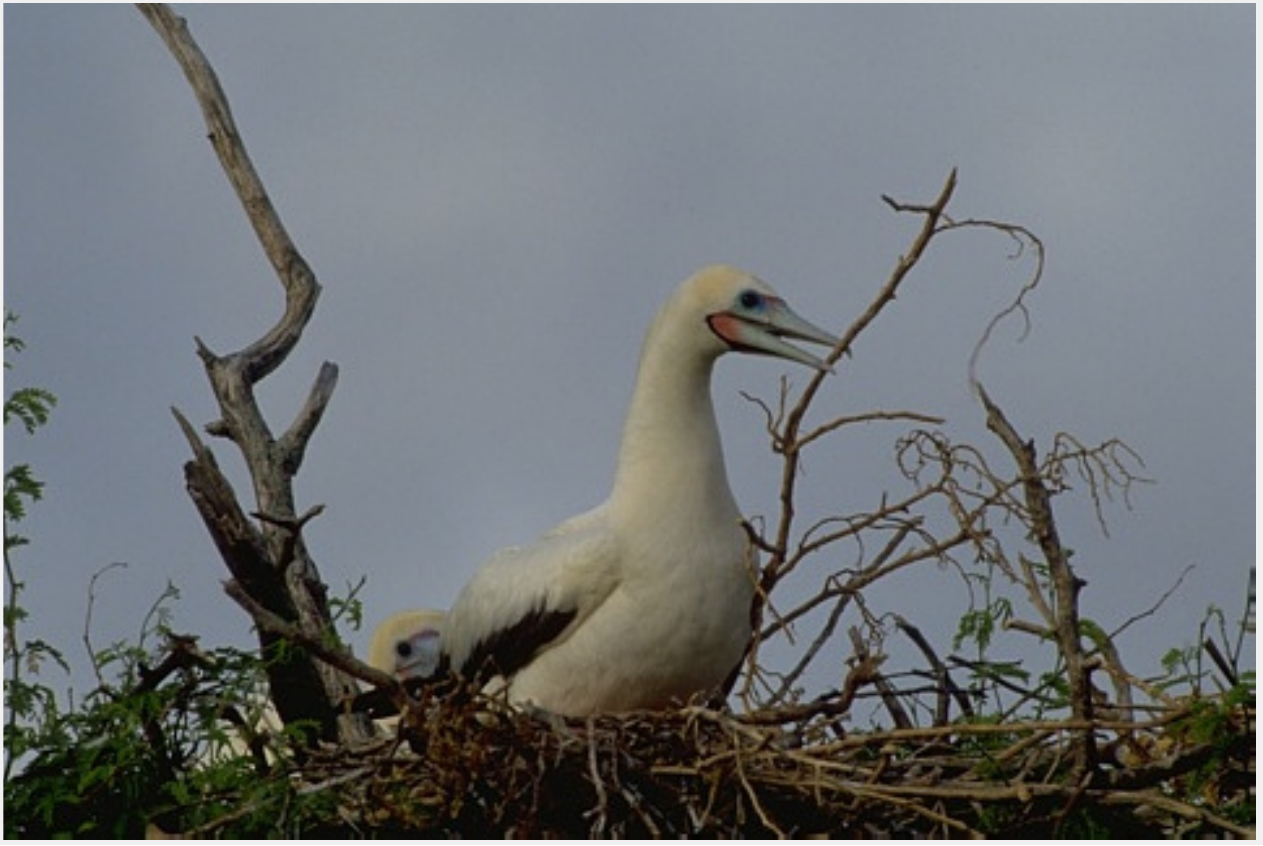}
	\end{subfigure}
	\begin{subfigure}[b]{0.195\textwidth}
		\centering
		\includegraphics[width=\textwidth]{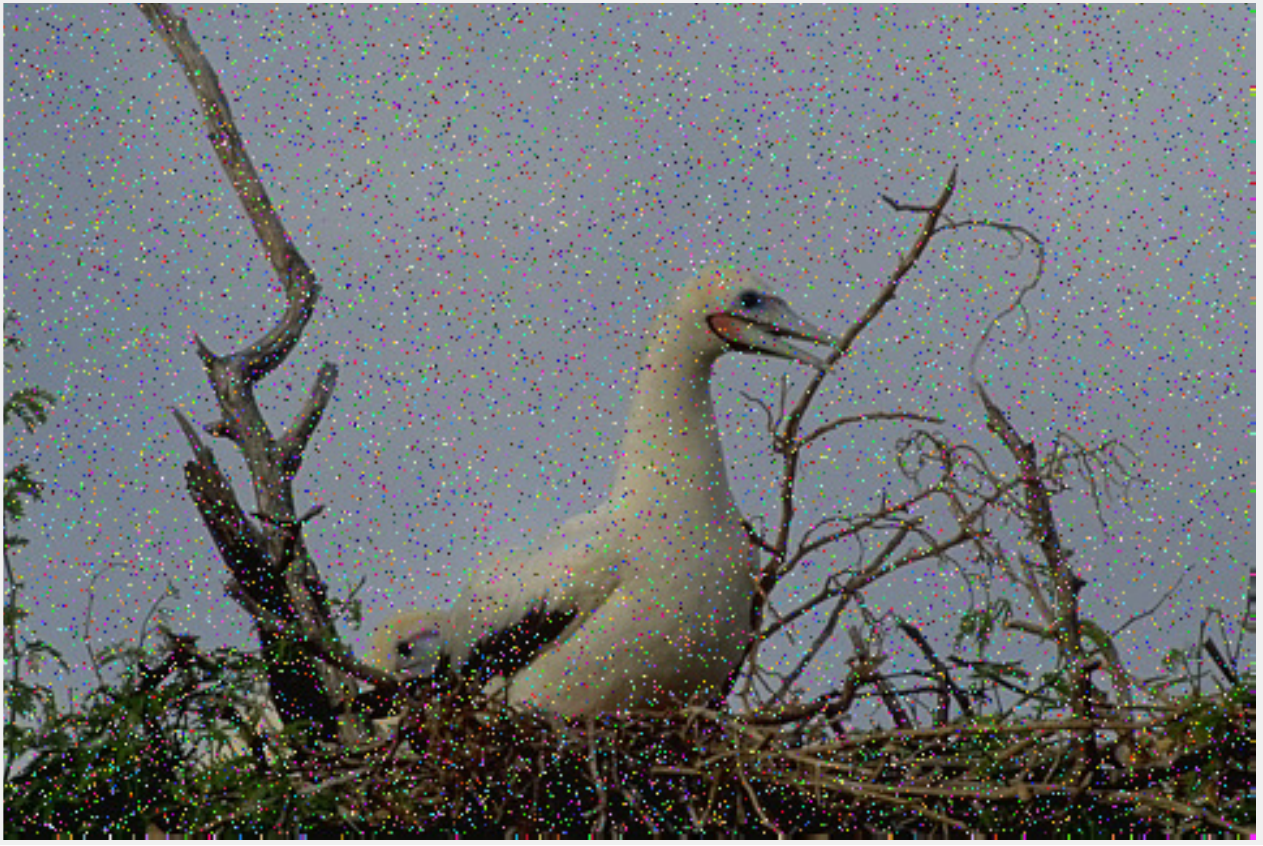}
	\end{subfigure}
	\begin{subfigure}[b]{0.195\textwidth}
		\centering
		\includegraphics[width=\textwidth]{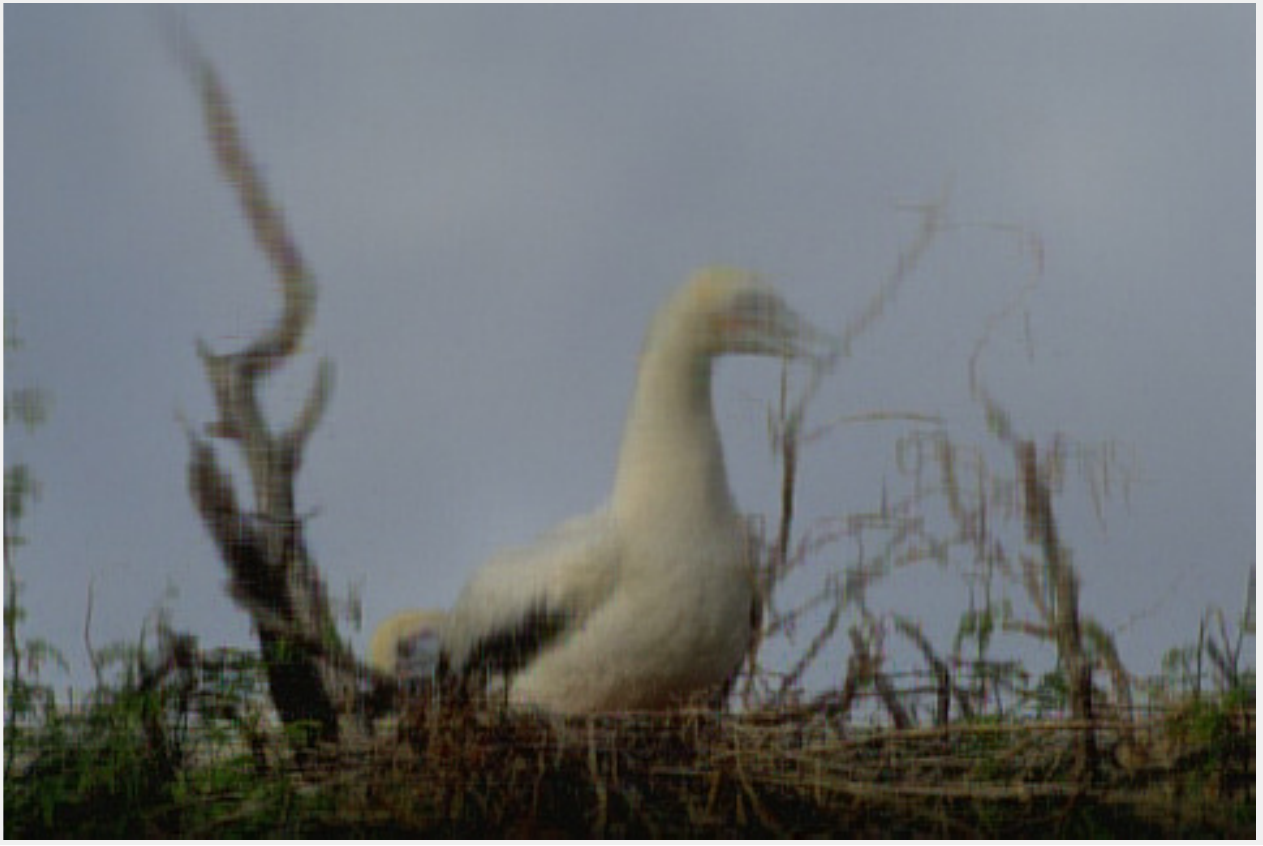}
	\end{subfigure}
	\begin{subfigure}[b]{0.195\textwidth}
		\centering
		\includegraphics[width=\textwidth]{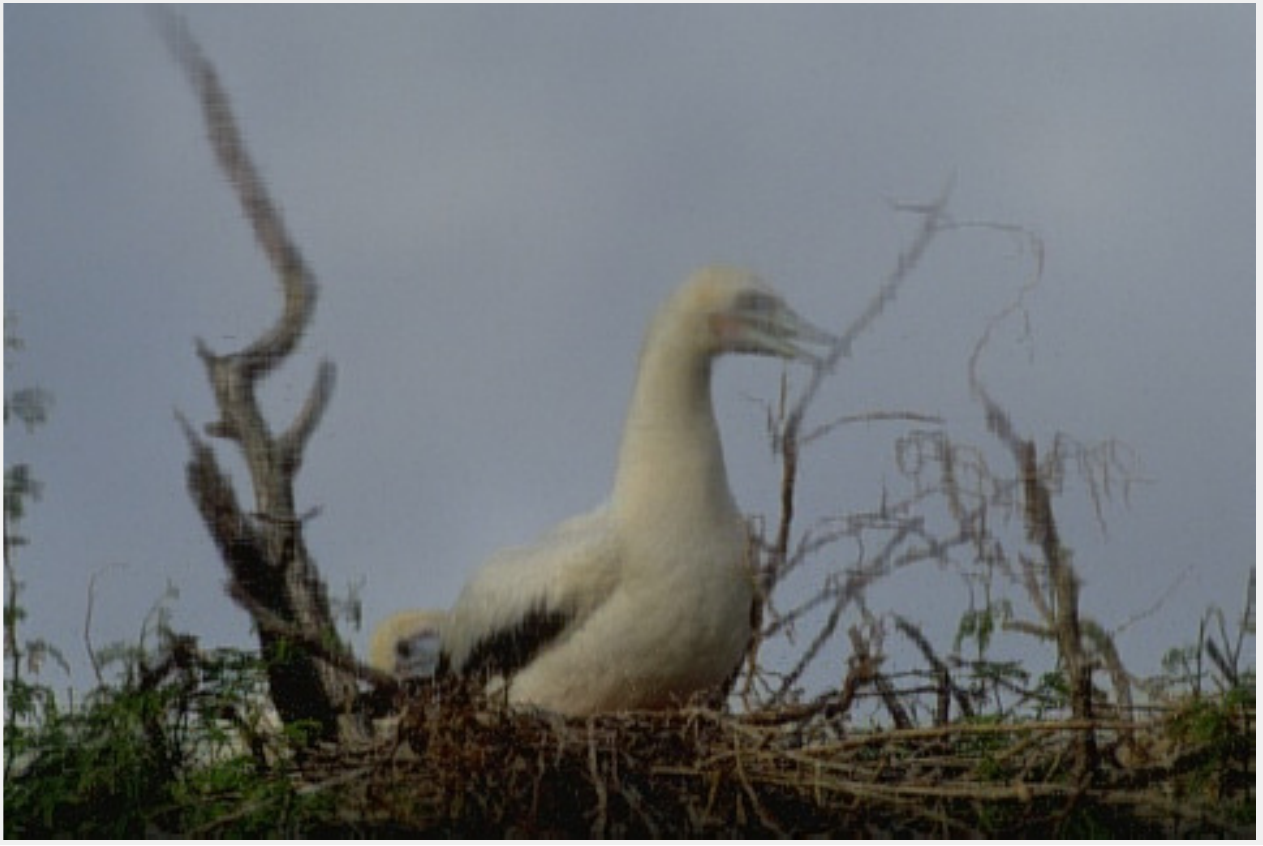}
	\end{subfigure}
	\begin{subfigure}[b]{0.195\textwidth}
		\centering
		\includegraphics[width=\textwidth]{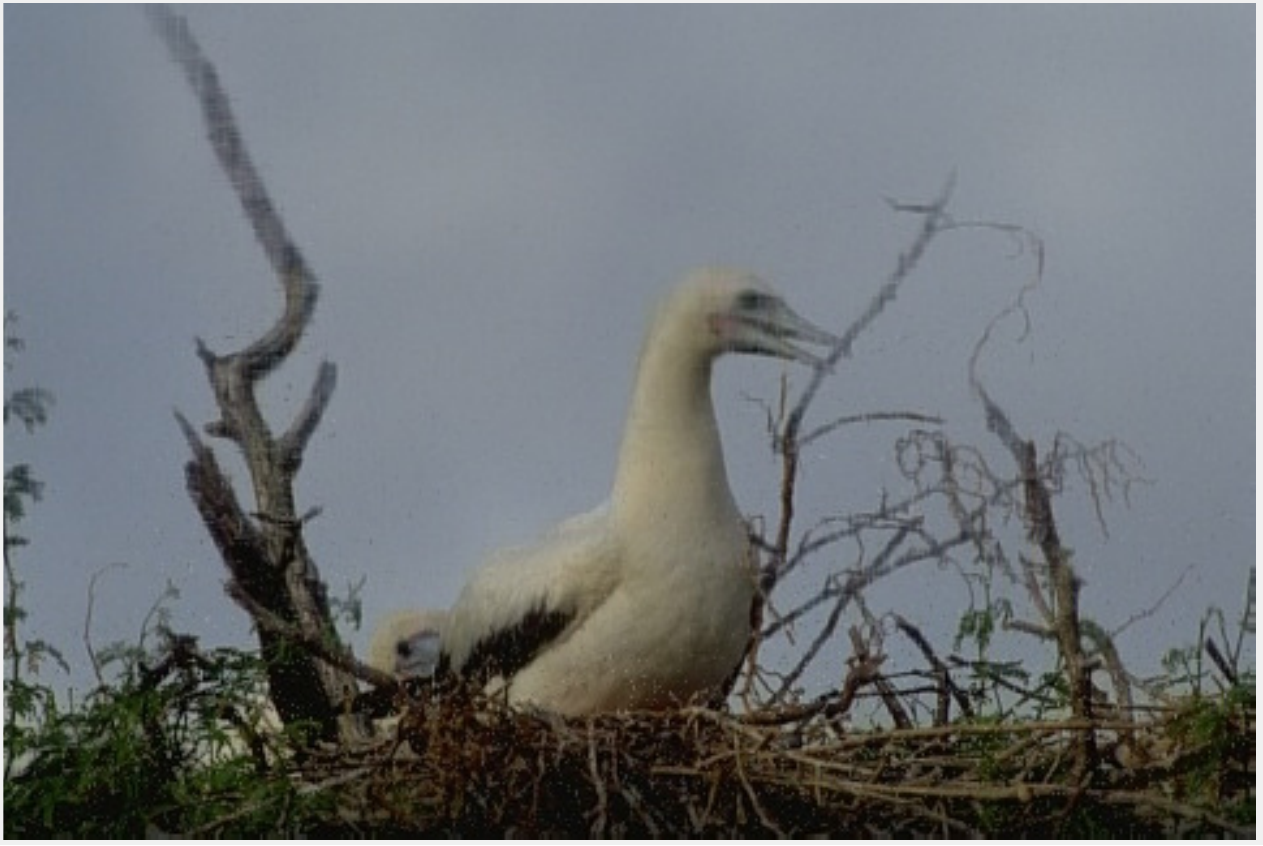}
	\end{subfigure}
	\begin{subfigure}[b]{0.195\textwidth}
		\centering
		\includegraphics[width=\textwidth]{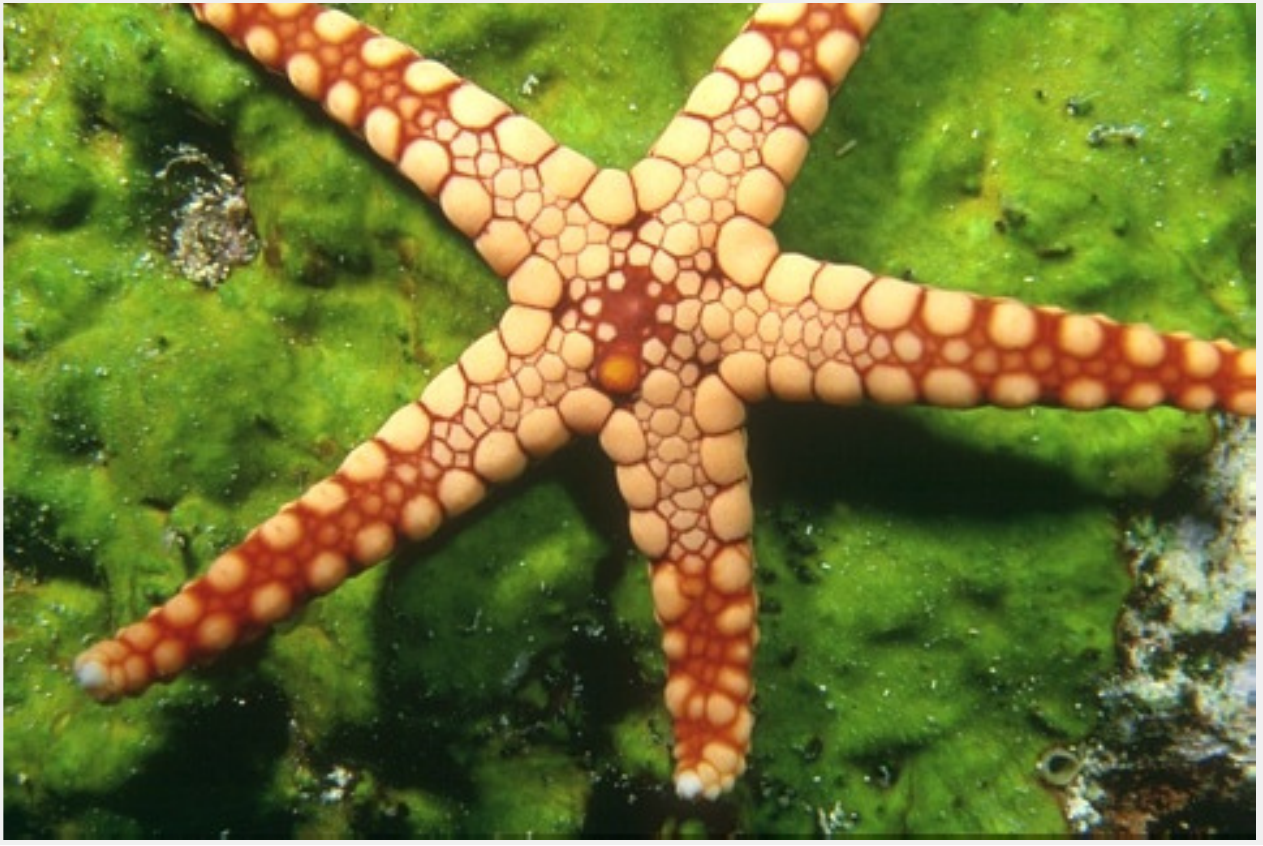}
	\end{subfigure}
	\begin{subfigure}[b]{0.195\textwidth}
		\centering
		\includegraphics[width=\textwidth]{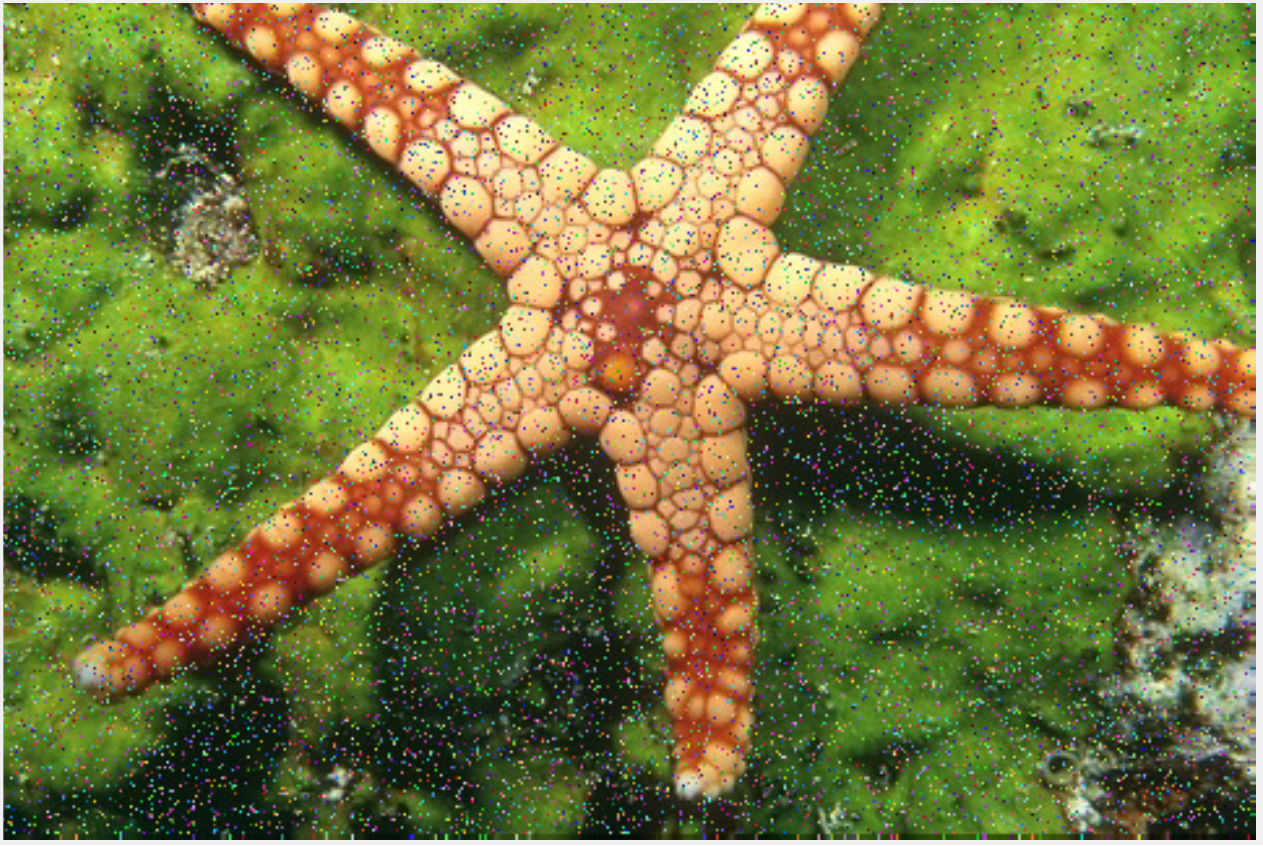}
	\end{subfigure}
	\begin{subfigure}[b]{0.195\textwidth}
		\centering
		\includegraphics[width=\textwidth]{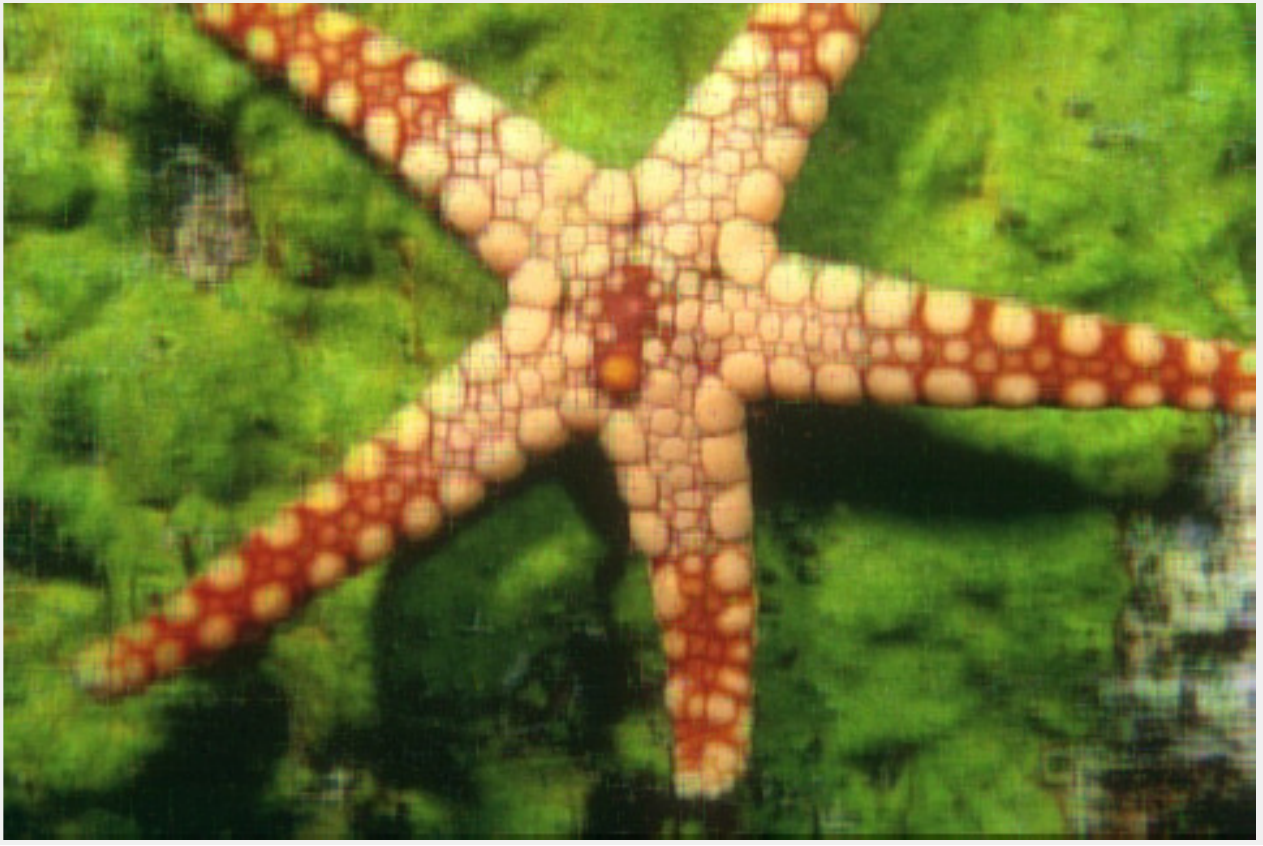}
	\end{subfigure}
	\begin{subfigure}[b]{0.195\textwidth}
		\centering
		\includegraphics[width=\textwidth]{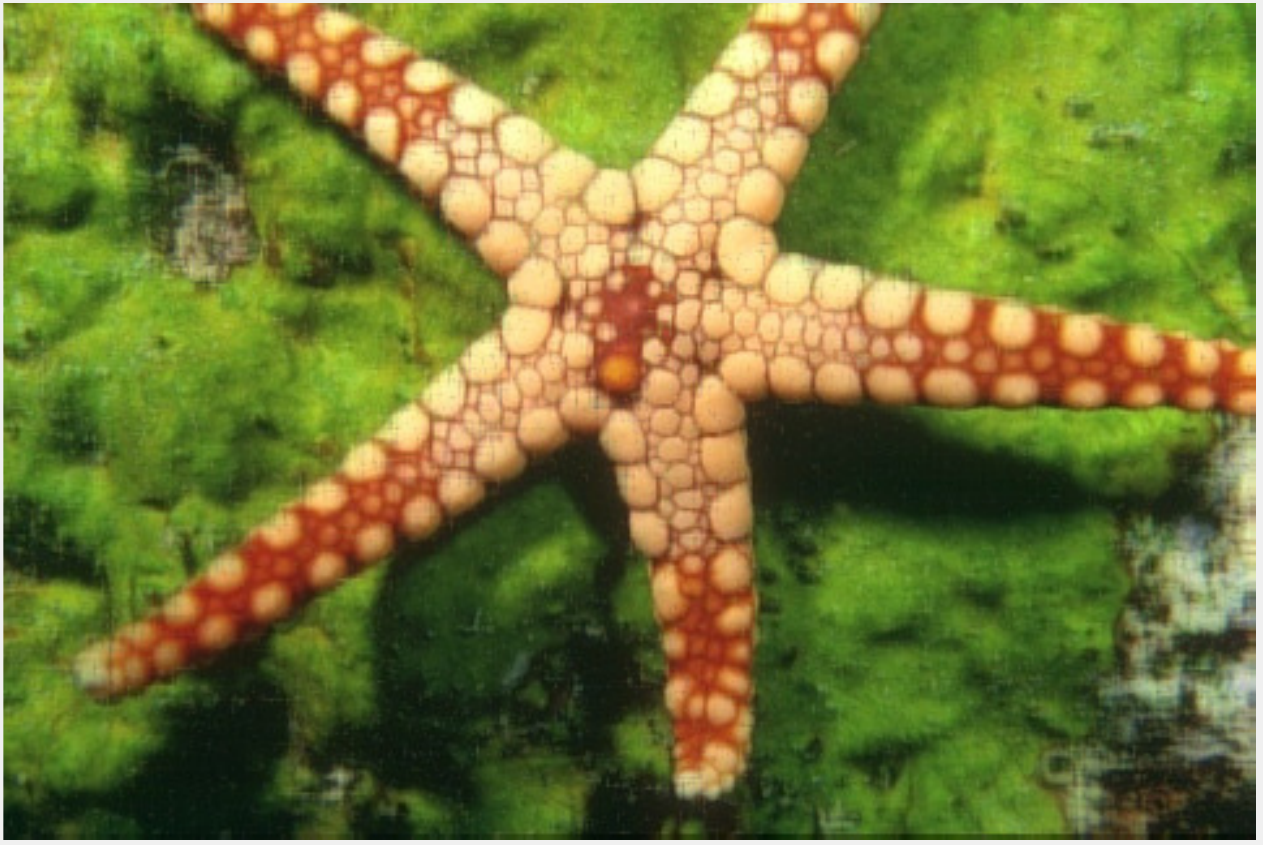}
	\end{subfigure}
	\begin{subfigure}[b]{0.195\textwidth}
		\centering
		\includegraphics[width=\textwidth]{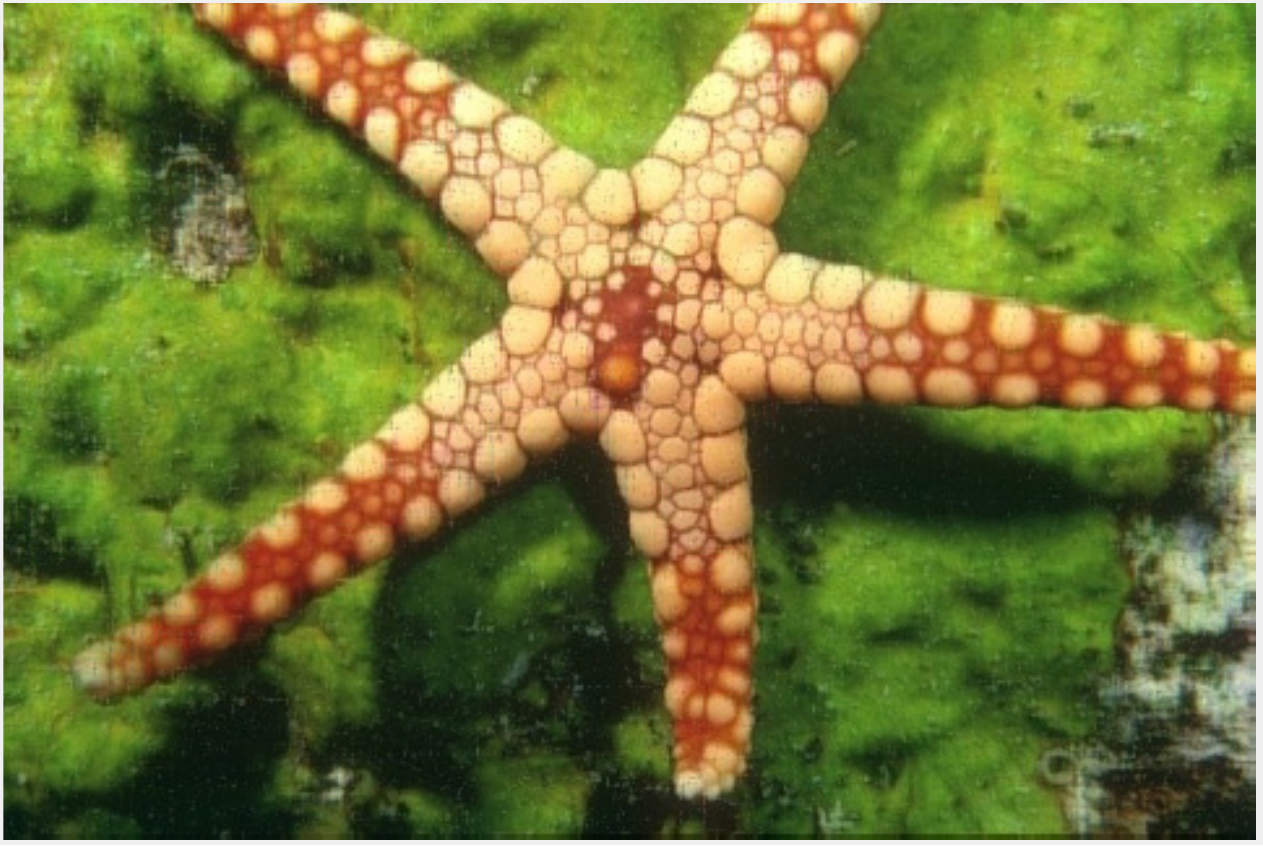}
	\end{subfigure}
	\begin{subfigure}[b]{0.195\textwidth}
		\centering
		\includegraphics[width=\textwidth]{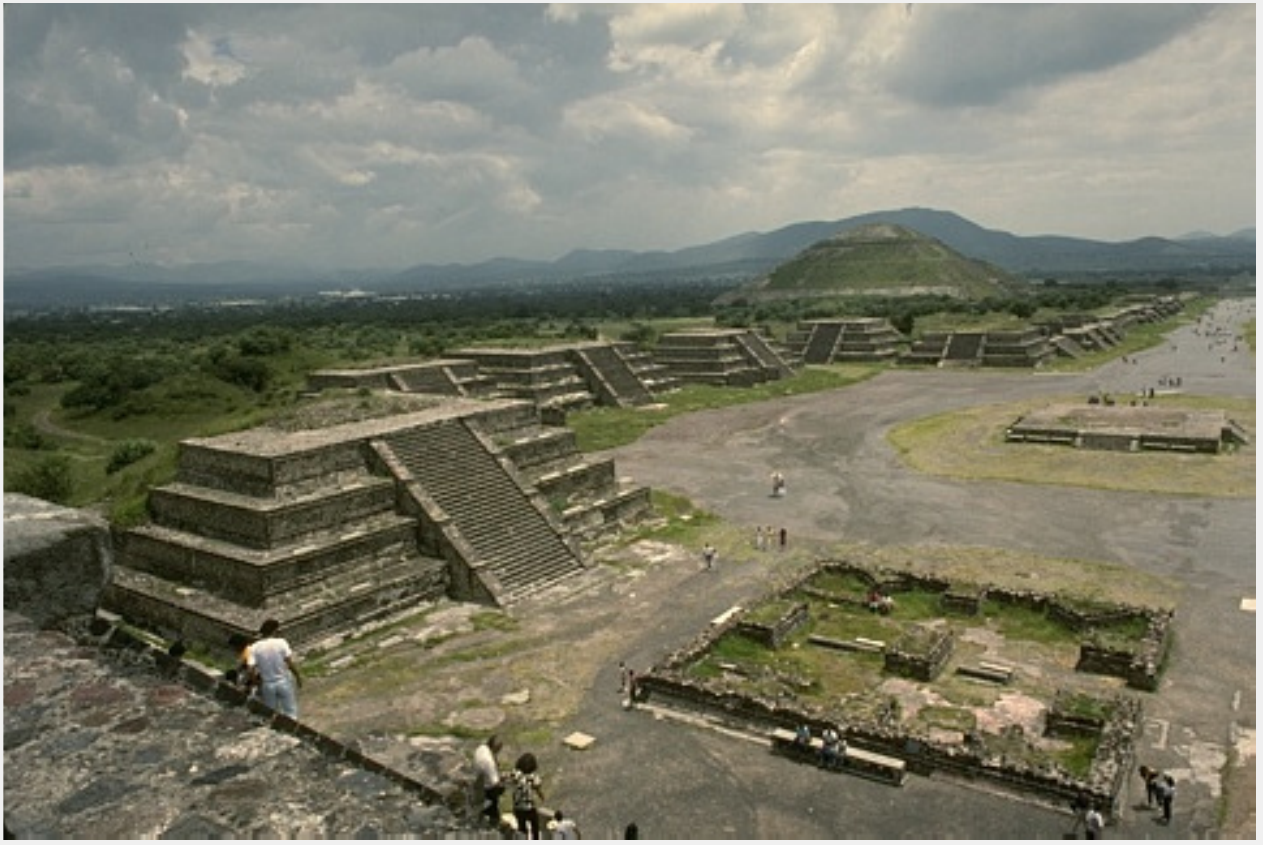}
	\end{subfigure}
	\begin{subfigure}[b]{0.195\textwidth}
		\centering
		\includegraphics[width=\textwidth]{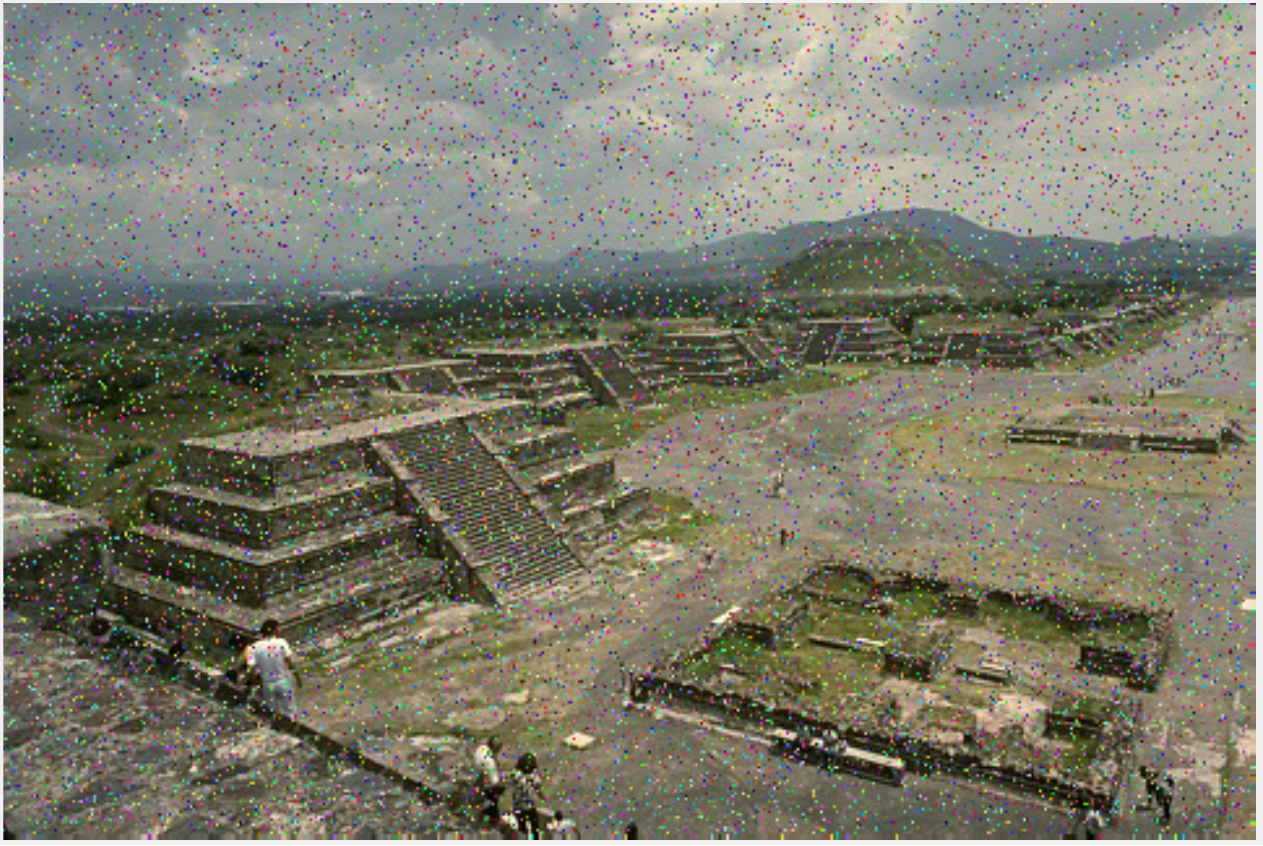}
	\end{subfigure}
	\begin{subfigure}[b]{0.195\textwidth}
		\centering
		\includegraphics[width=\textwidth]{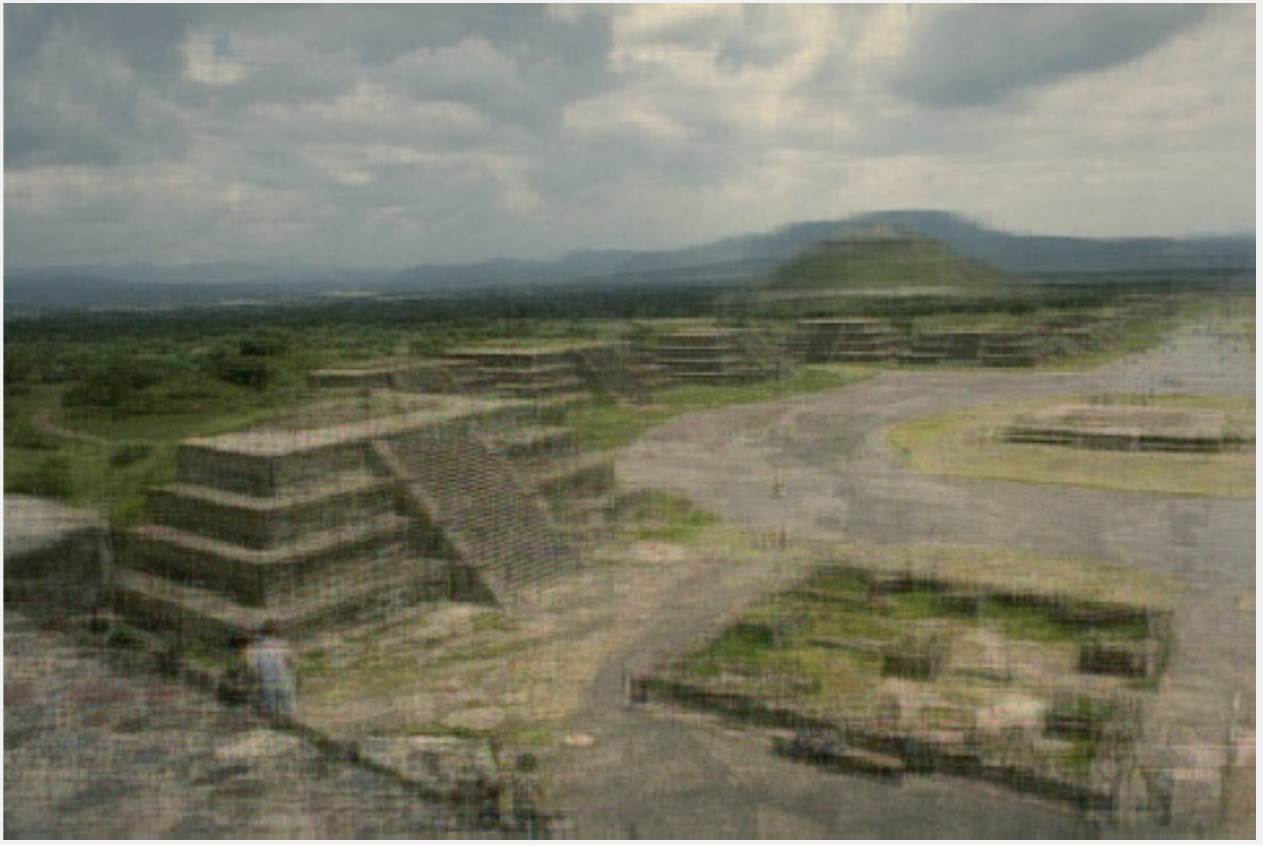}
	\end{subfigure}
	\begin{subfigure}[b]{0.195\textwidth}
		\centering
		\includegraphics[width=\textwidth]{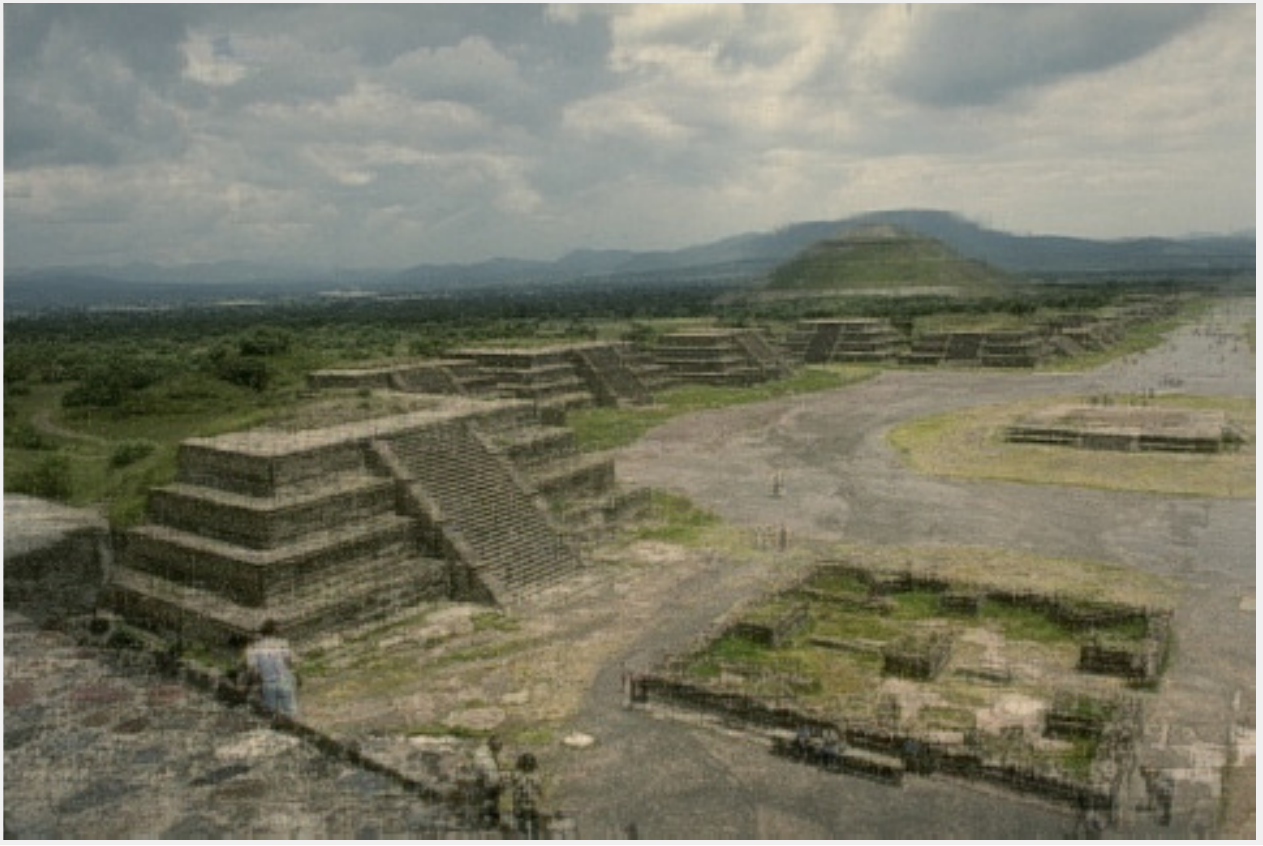}
	\end{subfigure}
	\begin{subfigure}[b]{0.195\textwidth}
		\centering
		\includegraphics[width=\textwidth]{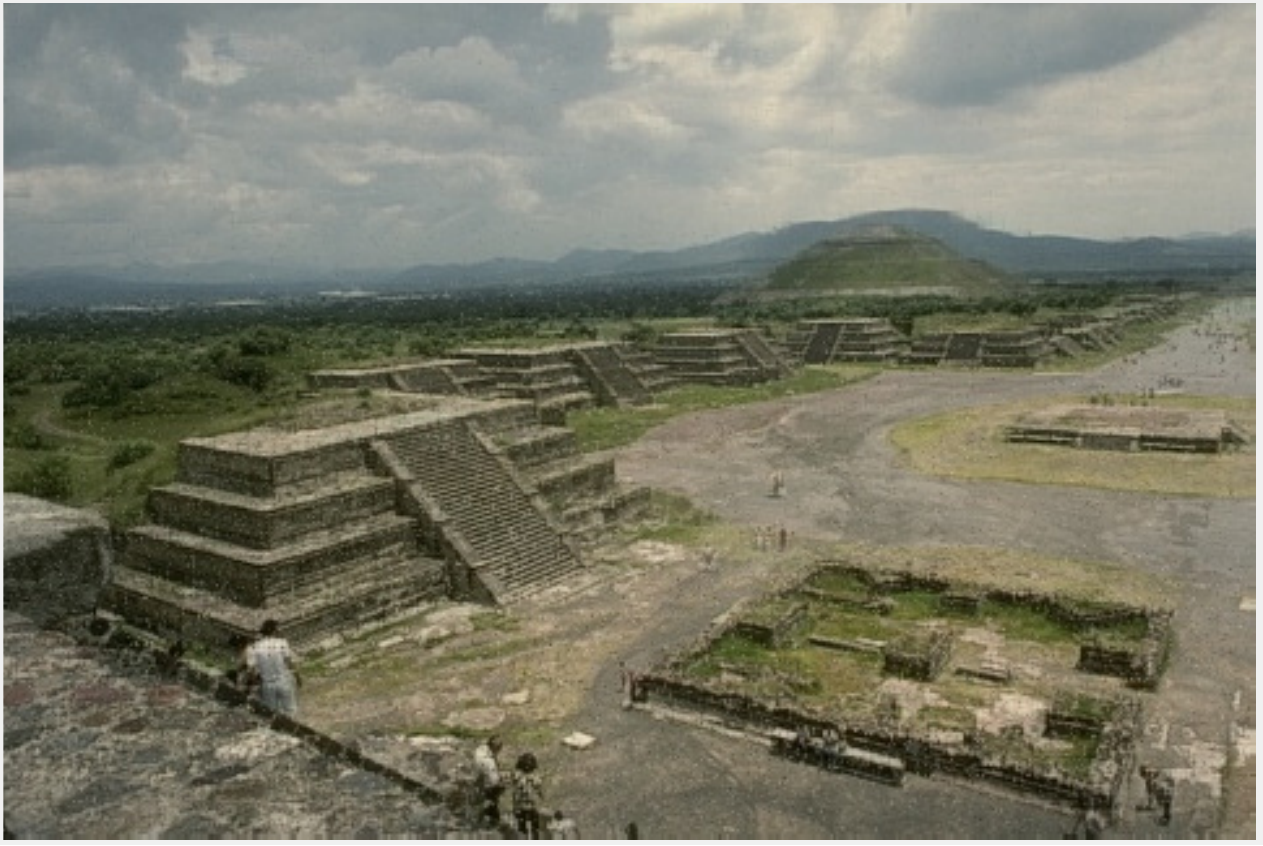}
	\end{subfigure}
	\begin{subfigure}[b]{0.195\textwidth}
		\centering
		\includegraphics[width=\textwidth]{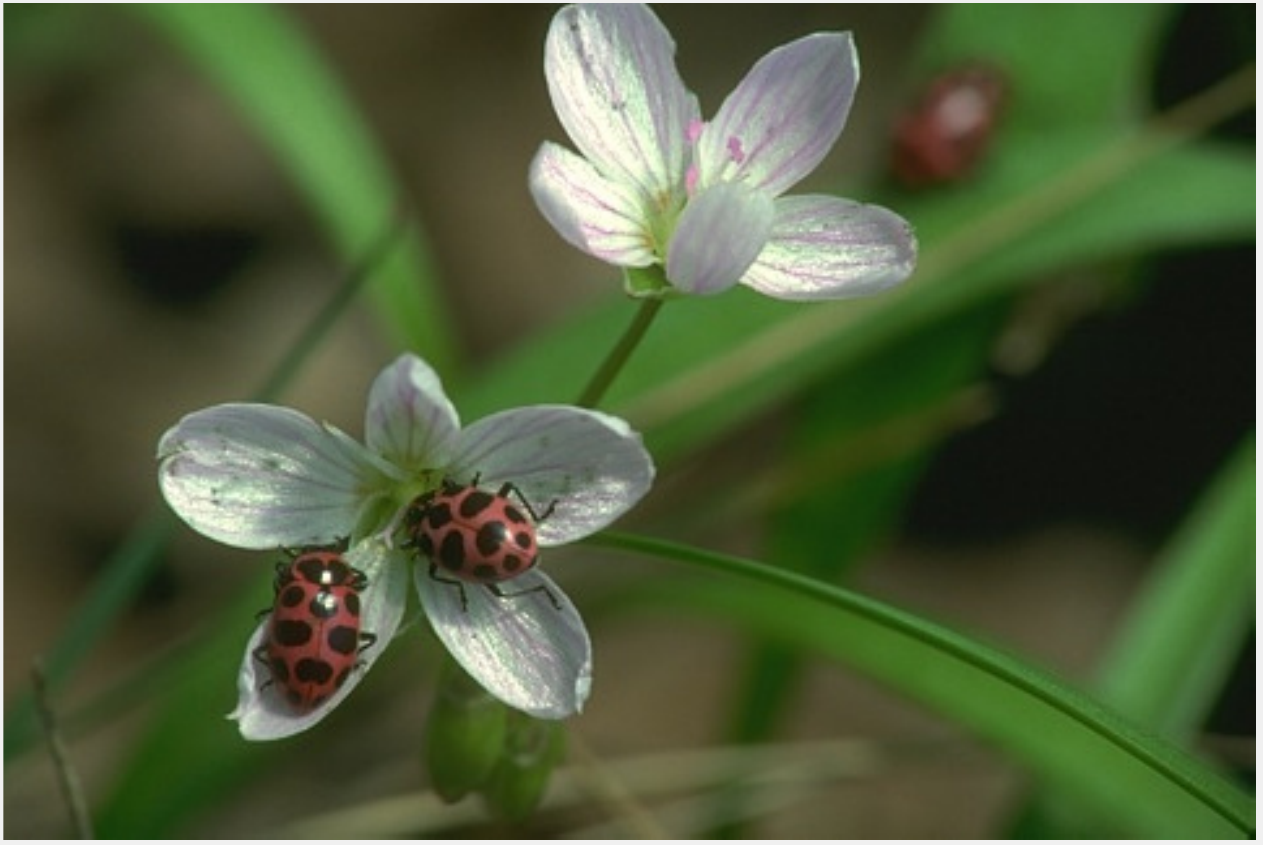}
	\end{subfigure}
	\begin{subfigure}[b]{0.195\textwidth}
		\centering
		\includegraphics[width=\textwidth]{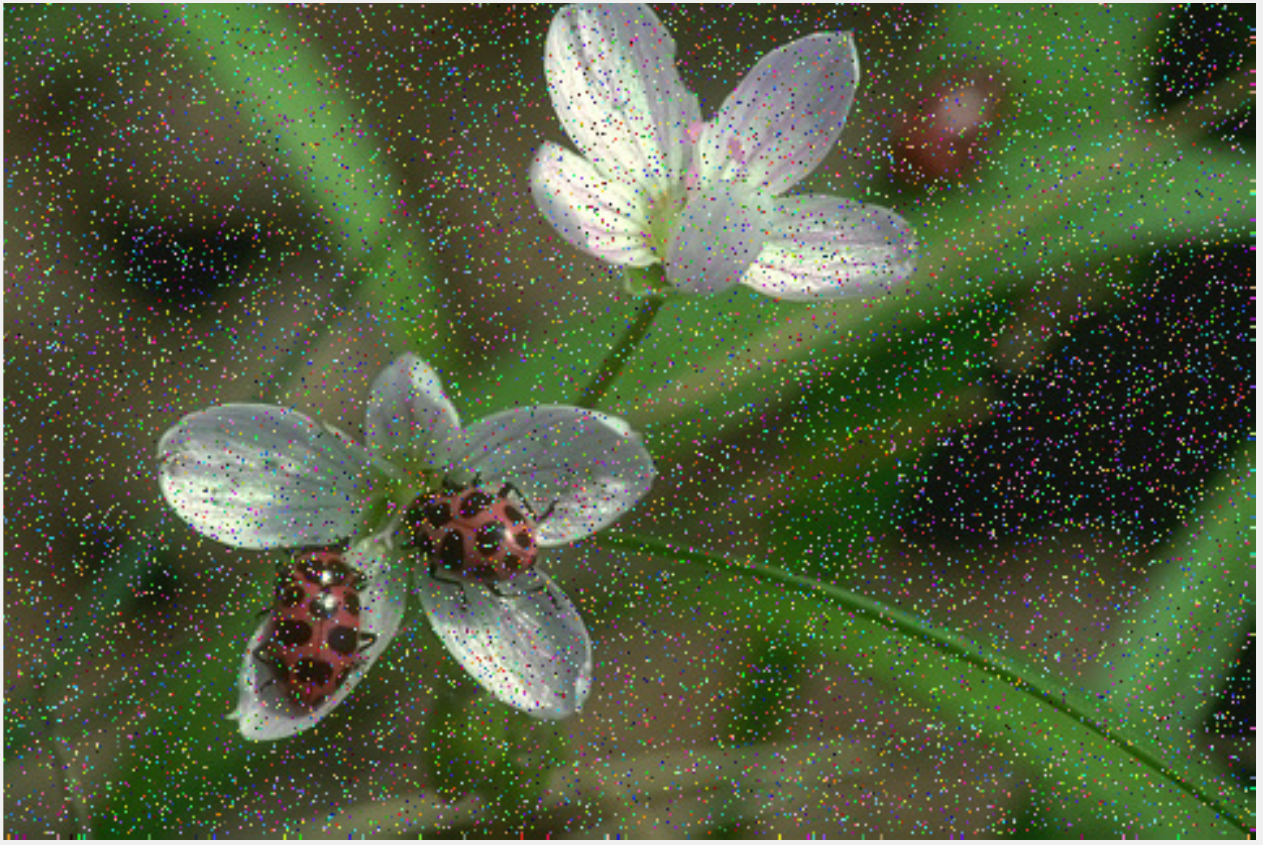}
	\end{subfigure}
	\begin{subfigure}[b]{0.195\textwidth}
		\centering
		\includegraphics[width=\textwidth]{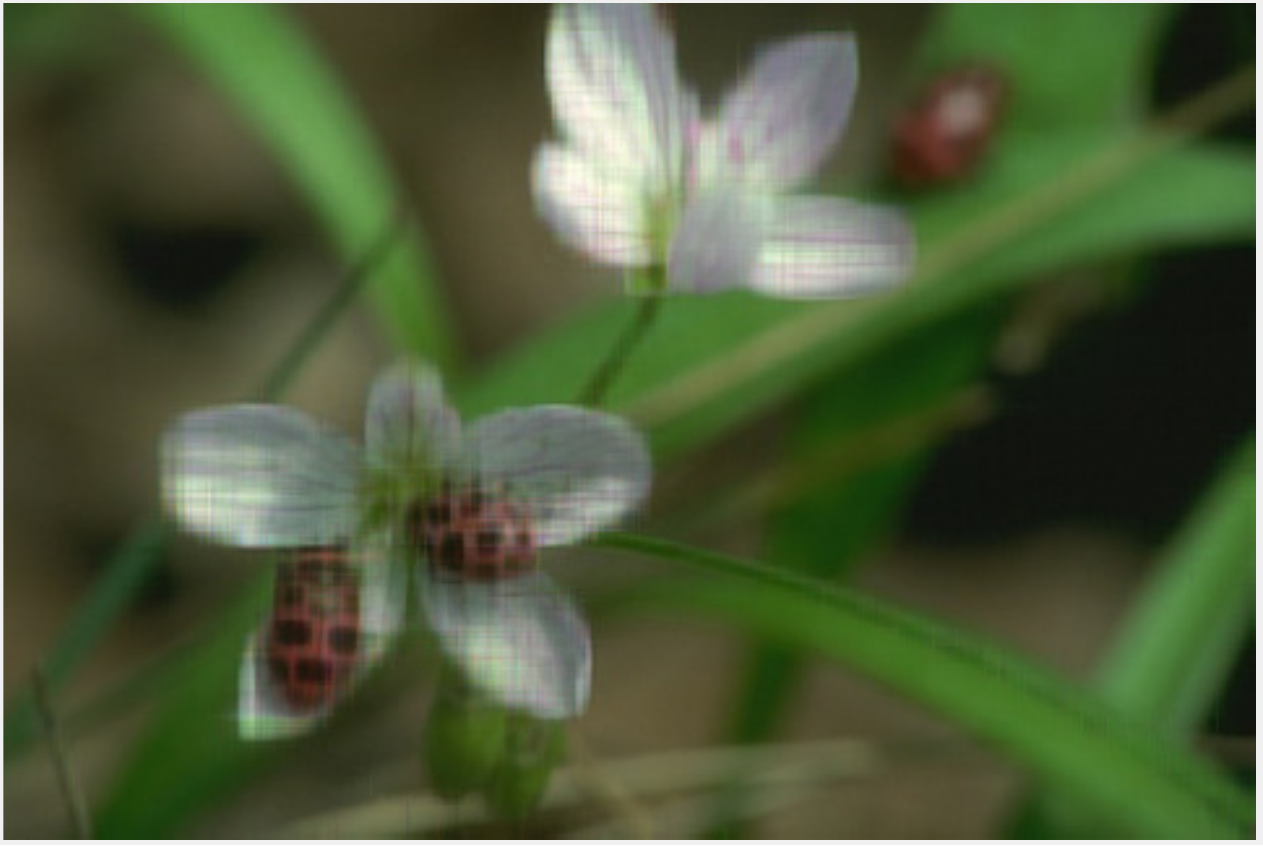}
	\end{subfigure}
	\begin{subfigure}[b]{0.195\textwidth}
		\centering
		\includegraphics[width=\textwidth]{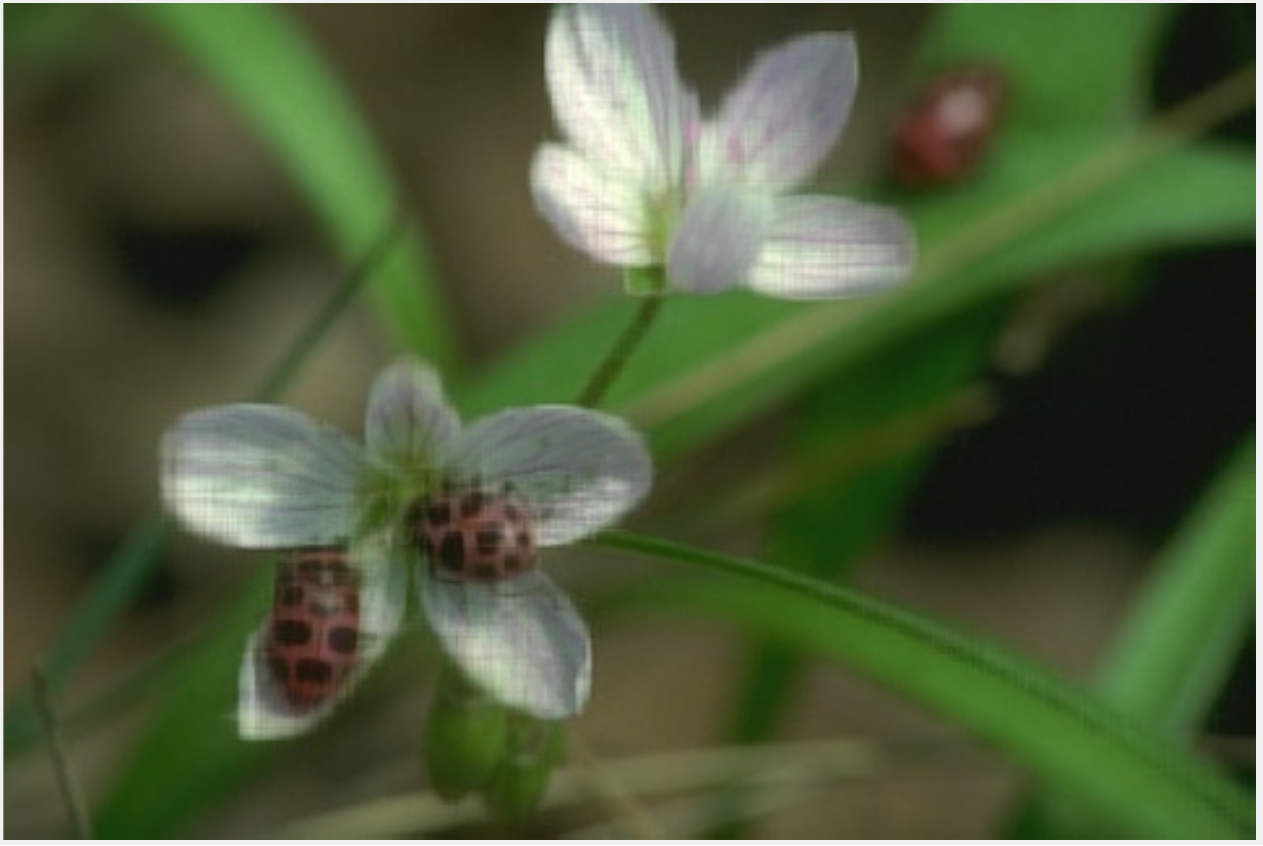}
	\end{subfigure}
	\begin{subfigure}[b]{0.195\textwidth}
		\centering
		\includegraphics[width=\textwidth]{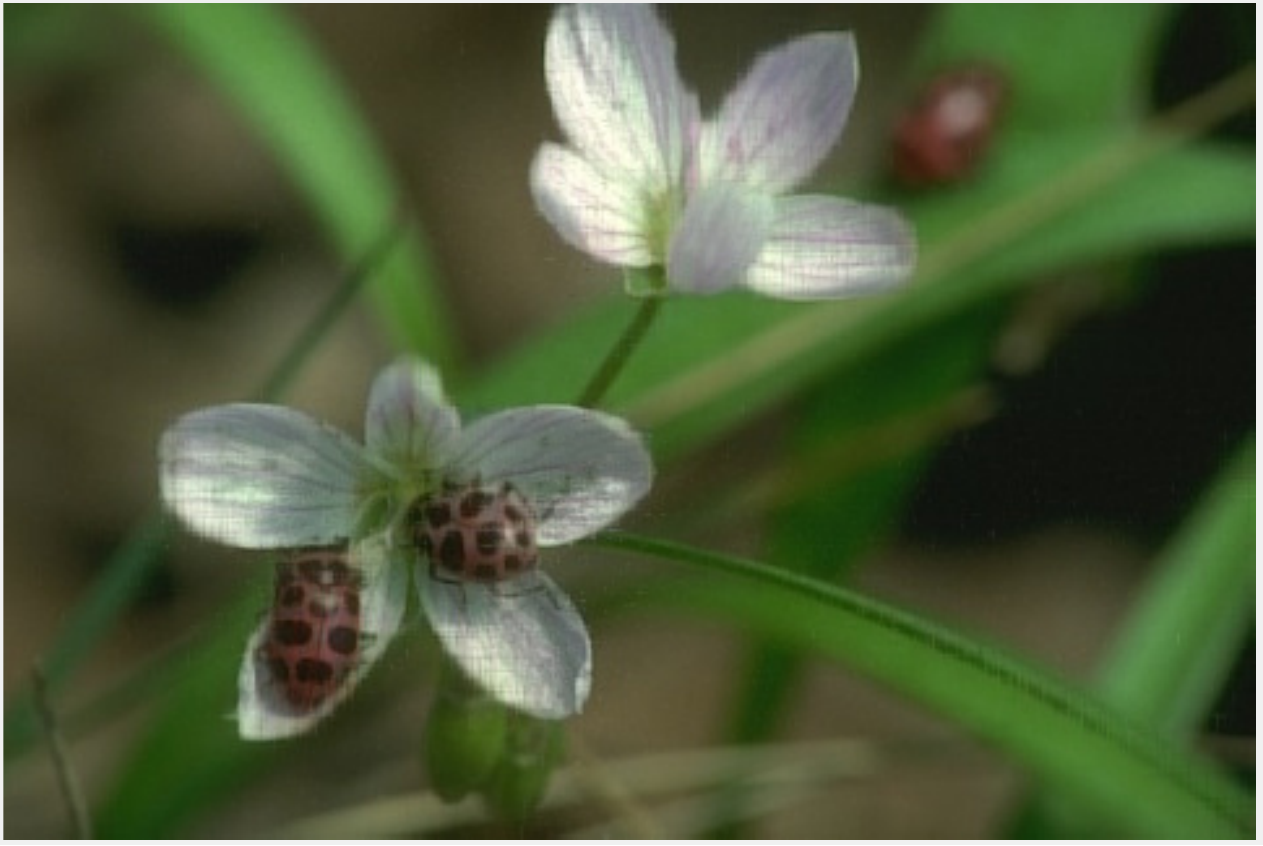}
	\end{subfigure}
	\begin{subfigure}[b]{0.195\textwidth}
		\centering
		\includegraphics[width=\textwidth]{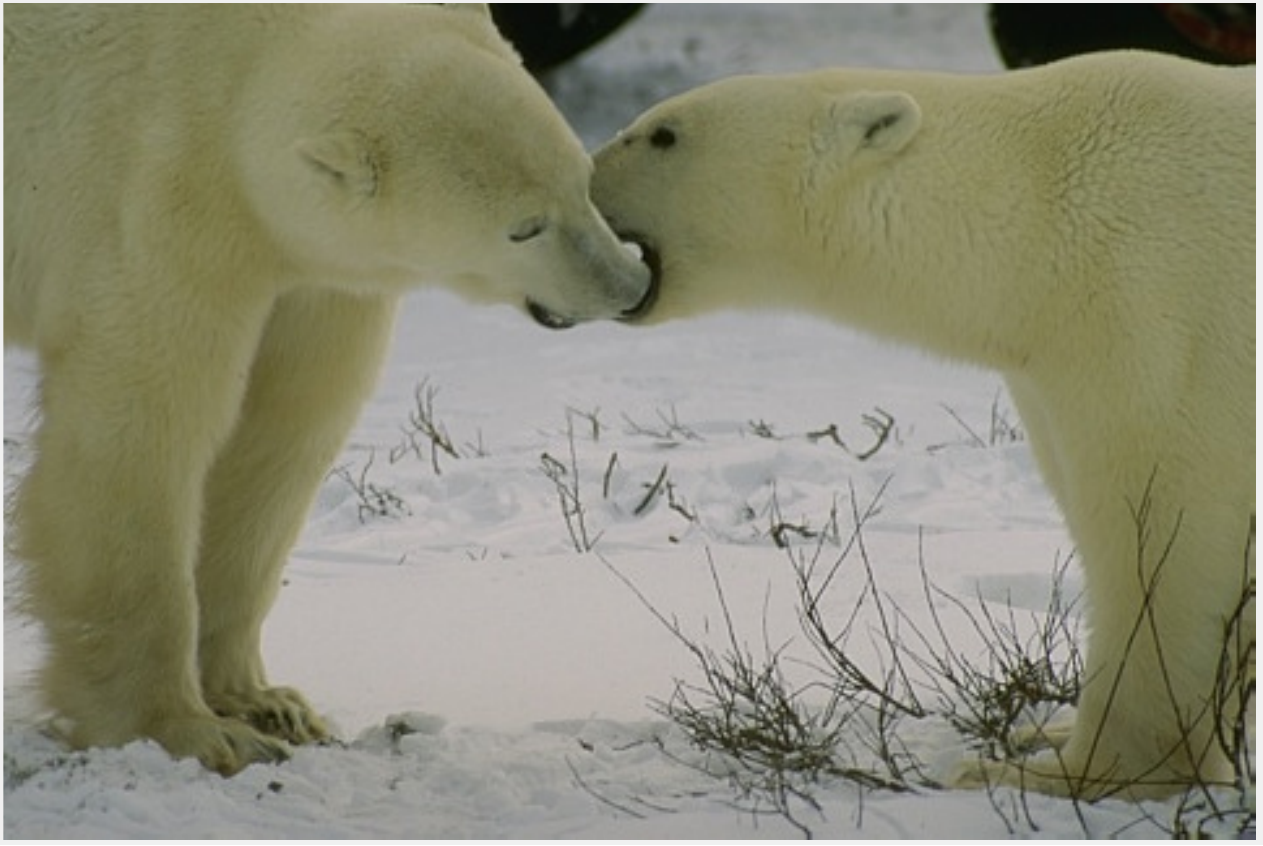}
		\caption{\small Orignal image}
	\end{subfigure}
	\begin{subfigure}[b]{0.195\textwidth}
		\centering
		\includegraphics[width=\textwidth]{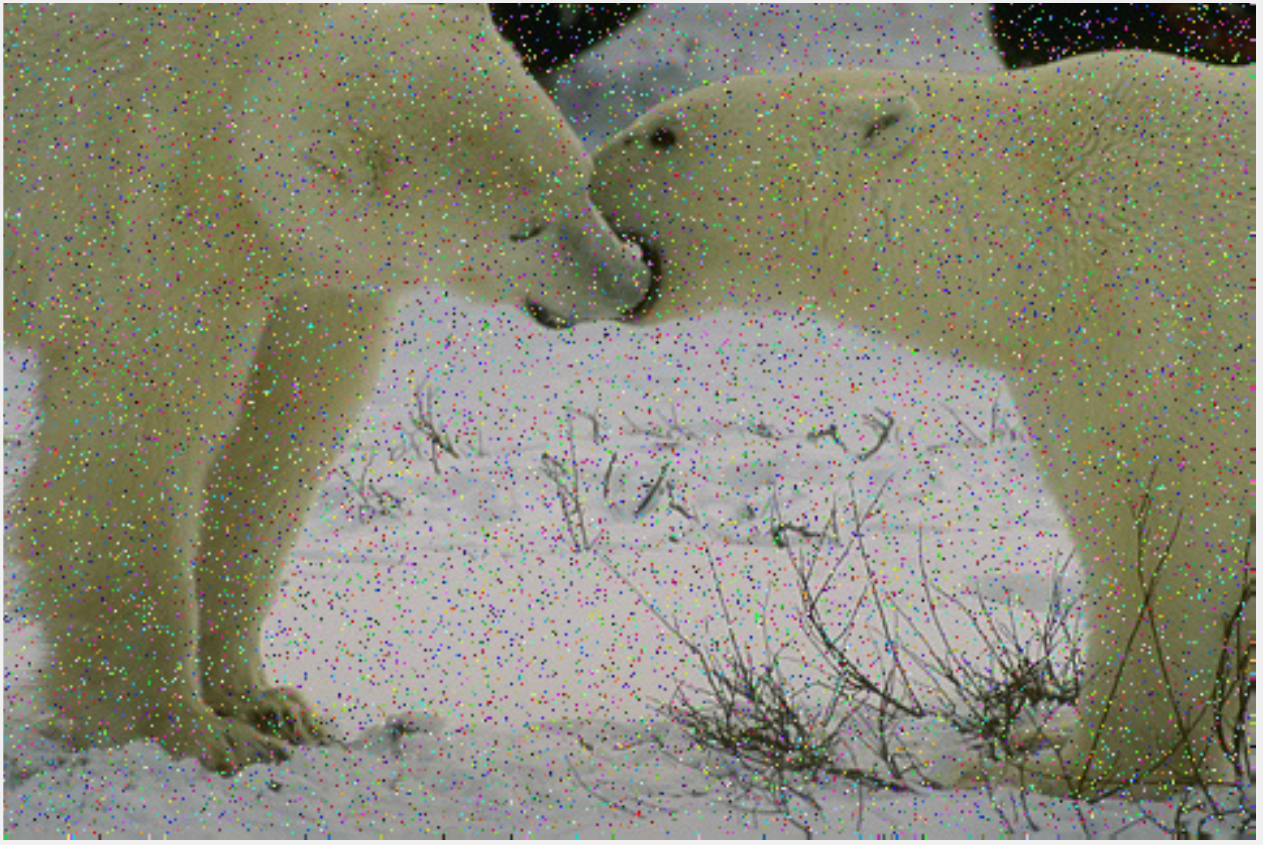}
		\caption{\small Observed image}
	\end{subfigure}
	\begin{subfigure}[b]{0.195\textwidth}
		\centering
		\includegraphics[width=\textwidth]{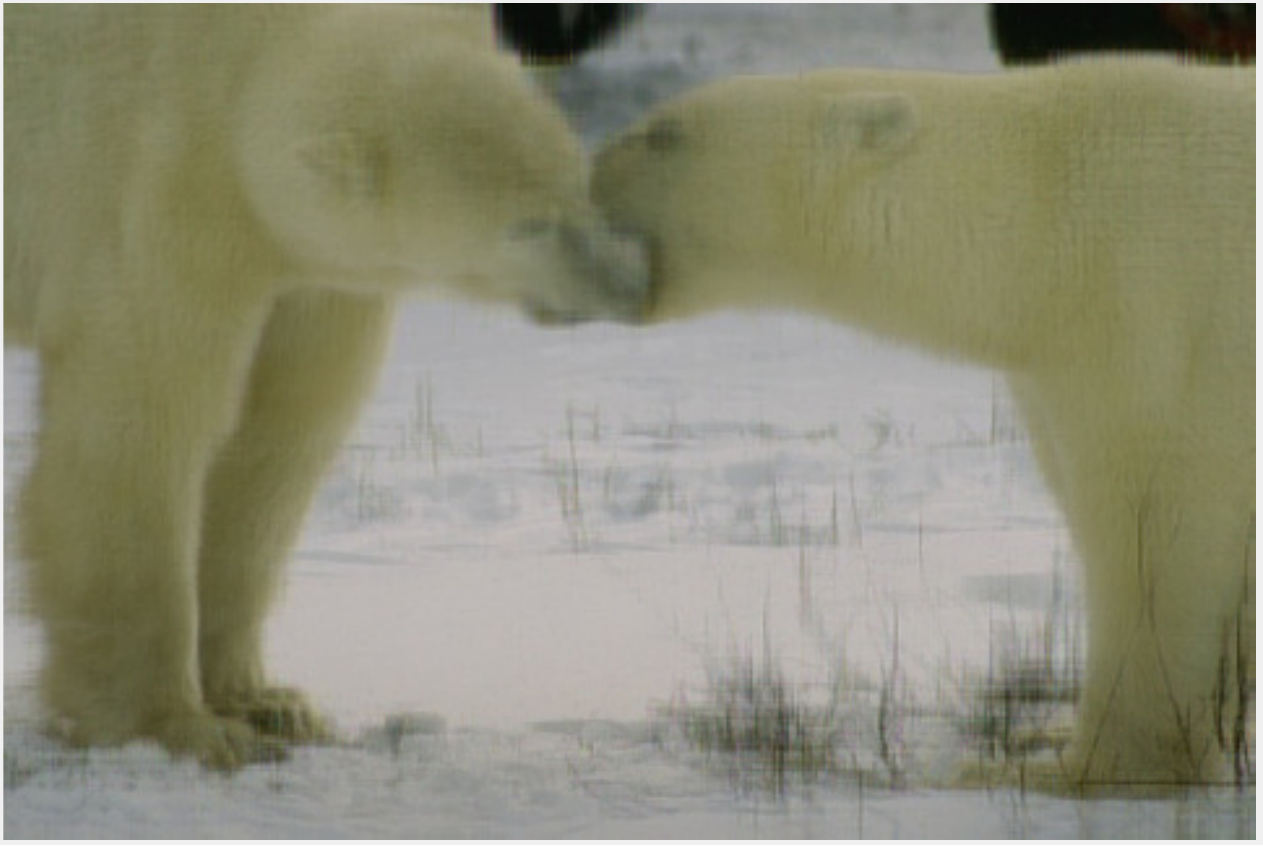}
		\caption{\small RPCA}
	\end{subfigure}
	\begin{subfigure}[b]{0.195\textwidth}
		\centering
		\includegraphics[width=\textwidth]{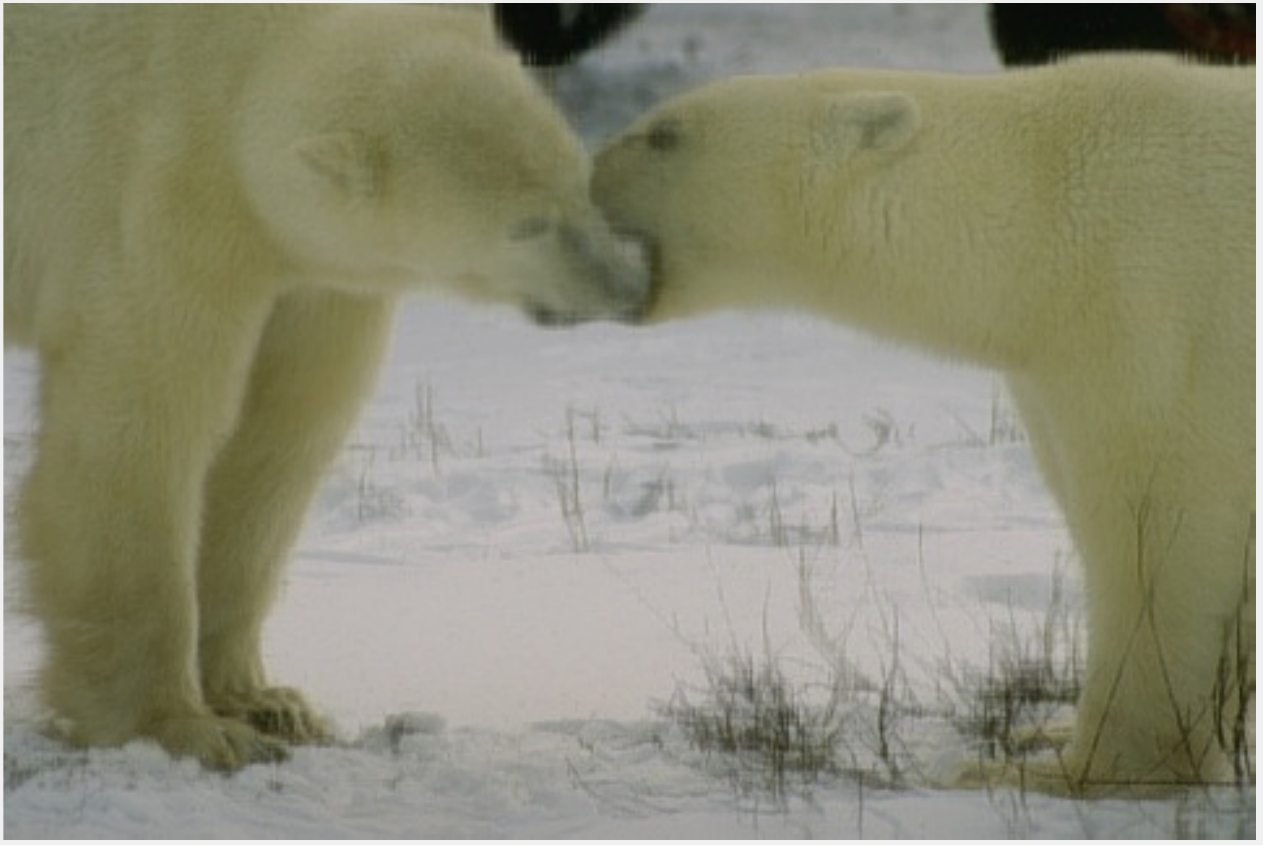}
		\caption{\small SNN}
	\end{subfigure}
	\begin{subfigure}[b]{0.195\textwidth}
		\centering
		\includegraphics[width=\textwidth]{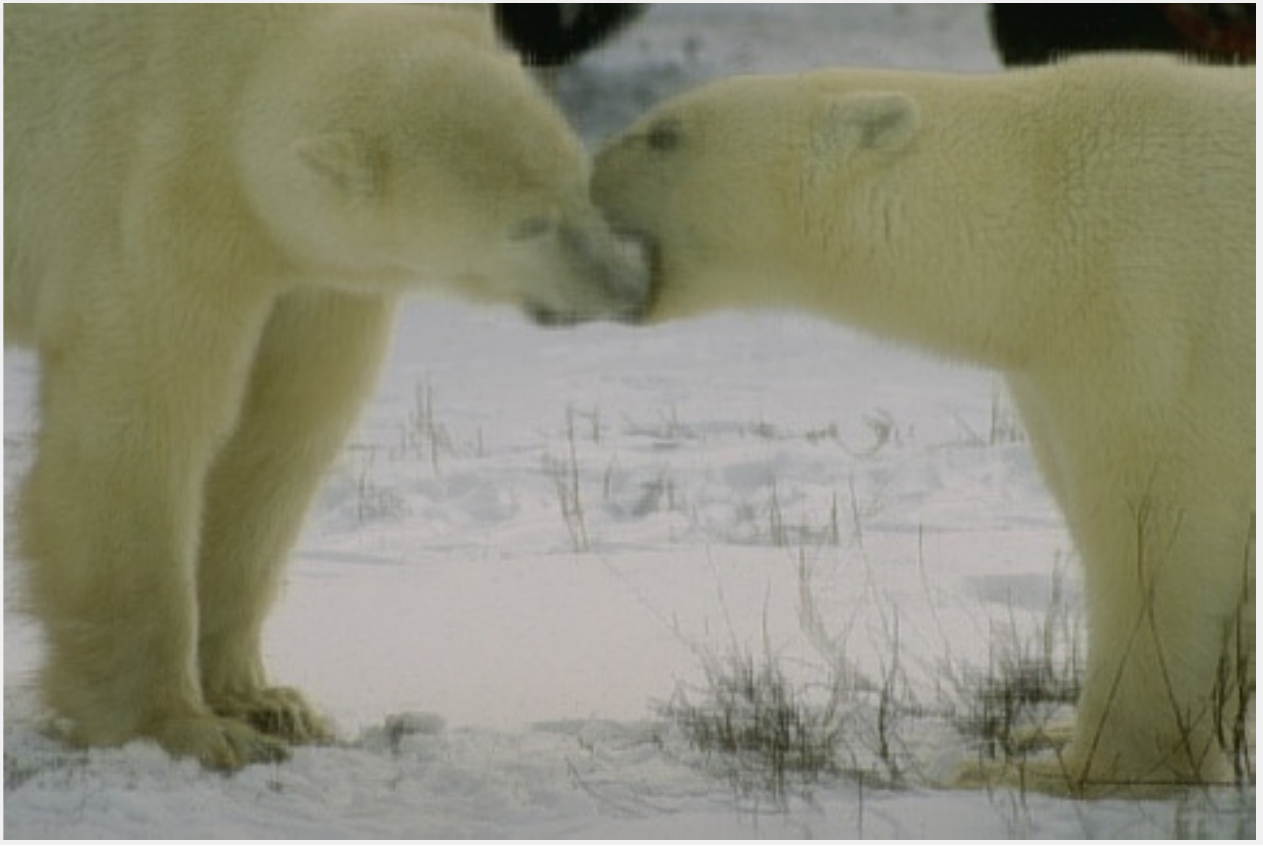}
		\caption{\small TRPCA}
	\end{subfigure}
	\\
	\vspace{0.8em}\small 
	\centering
\begin{minipage}[c]{0.48\textwidth}%
	\centering
	\begin{tabular}{c|c|c|c|c|c|c}
		\hline
		Index	& 1		& 2	    & 3	    & 4	    & 5	    & 6    	\\ \hline
		RPCA	& 29.10 & 24.53 & 25.12 & 24.31 & 27.50 & 26.77 \\ \hline
		SNN 	& 30.91 & 26.45 & 27.66 & 26.45 & 29.26 & 28.19 \\ \hline
		TRPCA	& \textbf{32.33} & \textbf{28.30} & \textbf{28.59} & \textbf{28.62} & \textbf{31.06} & \textbf{30.16} \\ \hline
	\end{tabular}
	\caption*{\small (f) Comparison of the PSNR values on the above 6 images.}
\end{minipage}	
\begin{minipage}[c]{0.48\textwidth}%
	\centering
	\begin{tabular}{c|c|c|c|c|c|c}
		\hline
		Index 	& 1	    & 2    & 3    & 4    & 5    & 6    \\ \hline
		RPCA 	& 14.98 &13.79 &14.35 &12.45 &12.72 &15.73 \\ \hline
		SNN 	& 26.93 &25.20 &25.33 &23.47 &23.38 &28.16 \\ \hline
		TRPCA 	& \textbf{12.96} &\textbf{12.24} &\textbf{12.76} &\textbf{10.70} &\textbf{10.64} &\textbf{14.31} \\ \hline
	\end{tabular}
	\caption*{\small (g) Comparison of the running time (s) on the above 6 images.}
\end{minipage}
	\caption{\small{Recovery performance comparison on 6 example images. (a) Original image; (b) observed image; (c)-(e) recovered images by RPCA, SNN and TRPCA, respectively; (f) and (g) show the comparison of PSNR values and running time (second) on the above 6 images. }}\label{fig_imageinpr}
	\vspace{-0.3cm}
\end{figure*}

\subsection{Application to Image Recovery}

We apply TRPCA to image recovery from the corrupted images with random noises.The motivation is that the color images can be approximated by low rank matrices or tensors \cite{liu2013tensor}. We will show that the recovery performance of TRPCA is still satisfactory with the suggested parameter in theory on real data. 

We use 100 color images from the Berkeley Segmentation Dataset \cite{martin2001database} for the test. The sizes of images are $321\times 481$ or $481\times 321$. For each image, we randomly set $10\%$ of pixels to random values in [0, 255], and the positions of the corrupted pixels are unknown. All the 3 channels of the images are corrupted at the same positions (the corruptions are on the whole tubes). This problem is more challenging than the corruptions on 3 channels at different positions. See Figure \ref{fig_imageinpr} (b) for some sample images with noises. We compare our TRPCA model with RPCA \cite{RPCA} and SNN \cite{huang2014provable} which also own the theoretical recovery guarantee. For RPCA, we apply it on each channel separably and combine the results to obtain the recovered image. The parameter $\lambda$ is set to $\lambda = 1/\sqrt{\max{(n_1,n_2)}}$ as suggested in theory. For SNN in (\ref{eqsnn}), we find that it does not perform well when $\lambda_i$'s are set to the values suggested in theory \cite{huang2014provable}. We empirically set $\bm{\lambda}=[15, 15, 1.5]$ in (\ref{eqsnn}) to make SNN perform well in most cases. For our TRPCA, we format a $n_1\times n_2$ sized image as a tensor of size $n_1\times n_2\times 3$. We find that such a way of tensor construction usually performs better than some other ways. This may be due to the noises which present on the tubes. We set $\lambda = 1/\sqrt{3\max{(n_1,n_2)}}$ in TRPCA. We use the Peak Signal-to-Noise Ratio (PSNR), defined as
\begin{equation*}\label{psnreq}
	\text{PSNR} = 10\log_{10}\left( \frac{\norm{\M}_\infty^2}{\frac{1}{n_1n_2n_3} {\|\hat{\mathbf{\X}}-\M\|_F^2}} \right), 
\end{equation*}
to evaluate the recovery performance. 

Figure \ref{fig_img_res_psnr_time} gives the comparison of the PSNR values and running time on all 100 images. Some examples with the recovered images are shown in Figure \ref{fig_imageinpr}. From these results, we have the following observations. First, both SNN and TRPCA perform much better than the matrix based RPCA. The reason is that RPCA performs on each channel independently, and thus is not able to use the information across channels. The tensor methods instead take advantage of the multi-dimensional structure of data. Second, TRPCA outperforms SNN in most cases. This not only demonstrates the superiority of our TRPCA, but also validates our recovery guarantee in Theorem \ref{thm1} on image data. Note that SNN needs some additional effort to tune the weighted parameters $\lambda_i$'s empirically. Different from SNN which is a loose convex surrogate of the sum of Tucker rank, our TNN is a tight convex relaxation of the tensor average rank, and the recovery performance of the obtained optimal solutions has the tight recovery guarantee as RPCA. Third, we use the standard ADMM to solve RPCA, SNN and TRPCA. Figure \ref{fig_img_res_psnr_time} (bottom) shows that TRPCA is as efficient as RPCA, while SNN requires the highest cost in this experiment.

\subsection{Application to Background Modeling}

In this section, we consider the background modeling problem which aims to separate the foreground objects from the background. The frames of the background are highly \canyi{correlated} and thus can be modeled as a low rank tensor. The moving foreground objects occupy only a fraction of image pixels and thus can be treated as sparse errors. We solve this problem by using RPCA, SNN and TRPCA. \canyi{We consider four color videos, \textit{Hall} (144$\times 176$, 300), \textit{WaterSurface} (128$\times$160, 300), \textit{ShoppingMall} (256$\times$320, 100) and \textit{ShopCorridor} (144$\times$192, 200), where the numbers in the parentheses denote the frame size and the frame number. For each sequence with color frame size $h\times w$ and frame number $k$, we reshape it to a $(3hw) \times k$ matrix and use it in RPCA. To use SNN and TRPCA, we reshape the video to a $(hw)\times 3\times k$ tensor\footnote{We observe that this way of tensor construction performs well for TRPCA, despite one has some other ways.}.} The parameter of SNN in (\ref{eqsnn}) is set to $\bm{\lambda}=[10,\ 0.1,\ 1]\times 20$ in this experiment. 

Figure \ref{fig_bg_res} shows the performance and running time comparison of RPCA, SNN and TRPCA \canyi{on the four sequences}. It can be seen that the low rank components identify the main illuminations as background, while the sparse parts correspond to the motion in the scene. \canyi{Generally, our TRPCA performs the best. RPCA does not perform well on the \textit{Hall} and \textit{WaterSurface} sequences using the default parameter.} Also, TRPCA is as efficient as RPCA and SNN requires much higher computational cost. The efficiency of TRPCA is benefited from our faster way for computing tensor SVT in Algorithm \ref{alg_tsvt} which is the key step for solving TRPCA. 

\captionsetup{position=top}
\begin{figure}[!t]
	\centering
	\begin{subfigure}[b]{0.118\textwidth}
		\centering
		\includegraphics[width=\textwidth]{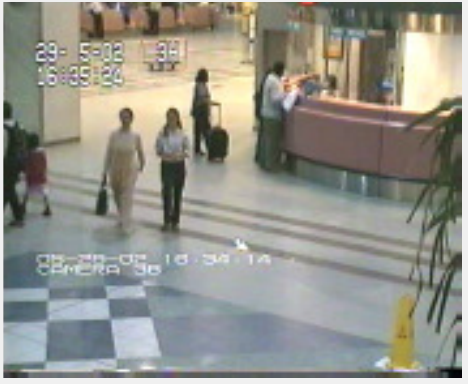}
	\end{subfigure}
	\begin{subfigure}[b]{0.118\textwidth}
		\centering
		\includegraphics[width=\textwidth]{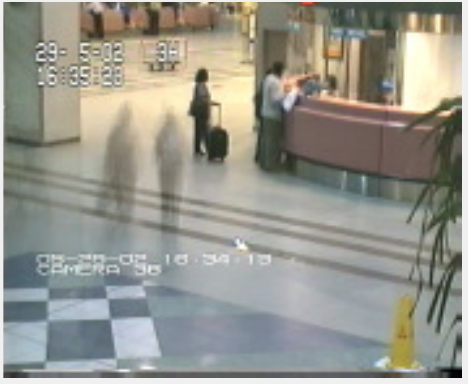}
	\end{subfigure}
	\begin{subfigure}[b]{0.118 \textwidth}
		\centering
		\includegraphics[width=\textwidth]{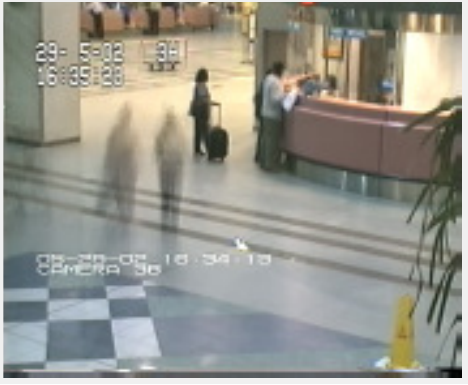}
	\end{subfigure}
	\begin{subfigure}[b]{0.118 \textwidth}
		\centering
		\includegraphics[width=\textwidth]{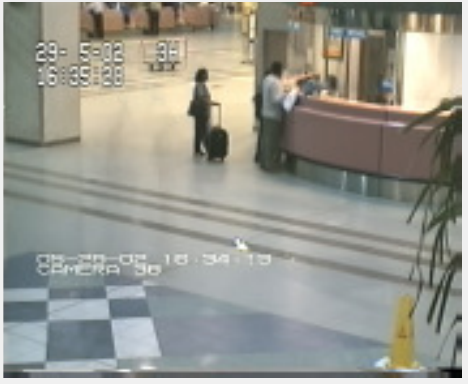}
	\end{subfigure}
	
	\begin{subfigure}[b]{0.118\textwidth}
	\end{subfigure}
	\hfill
	\begin{subfigure}[b]{0.118\textwidth}
		\centering
		\includegraphics[width=\textwidth]{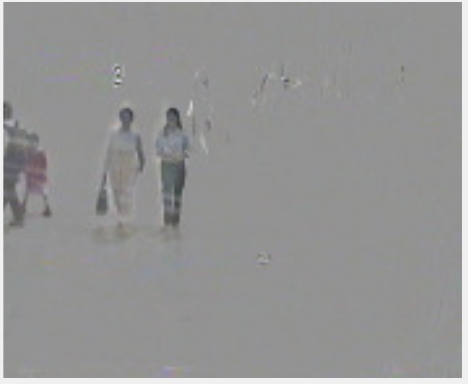}
	\end{subfigure}
	\begin{subfigure}[b]{0.118 \textwidth}
		\centering
		\includegraphics[width=\textwidth]{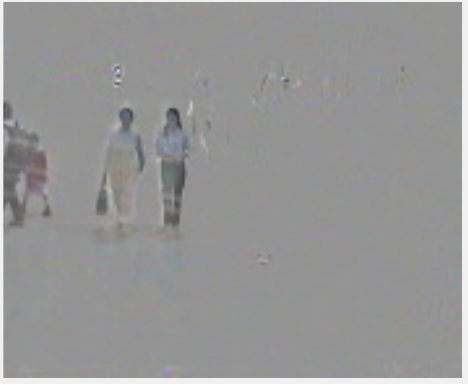}
	\end{subfigure}
	\begin{subfigure}[b]{0.118 \textwidth}
		\centering
		\includegraphics[width=\textwidth]{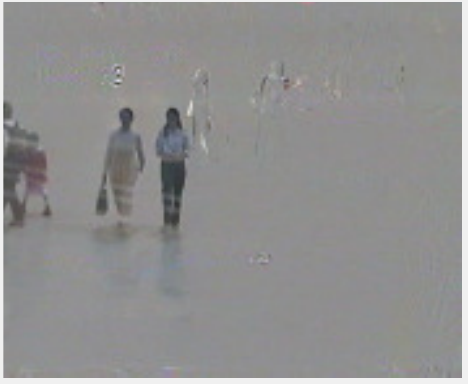}
	\end{subfigure}
	
	\begin{subfigure}[b]{0.118\textwidth}
		\centering
		\includegraphics[width=\textwidth]{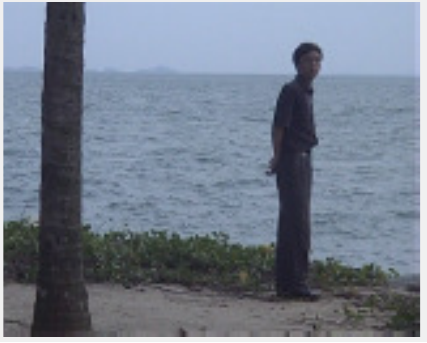}
	\end{subfigure}
	\begin{subfigure}[b]{0.118\textwidth}
		\centering
		\includegraphics[width=\textwidth]{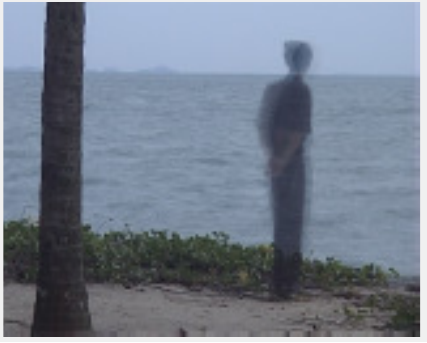}
	\end{subfigure}
	\begin{subfigure}[b]{0.118 \textwidth}
		\centering
		\includegraphics[width=\textwidth]{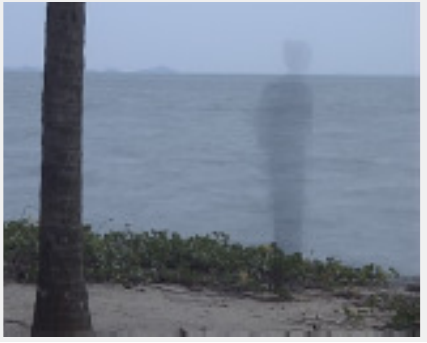}
	\end{subfigure}
	\begin{subfigure}[b]{0.118 \textwidth}
		\centering
		\includegraphics[width=\textwidth]{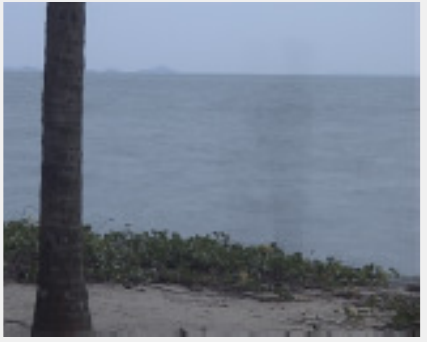}
	\end{subfigure}
	
	\begin{subfigure}[b]{0.118\textwidth}

	\end{subfigure}
	\hfill
	\begin{subfigure}[b]{0.118\textwidth}
		\centering
		\includegraphics[width=\textwidth]{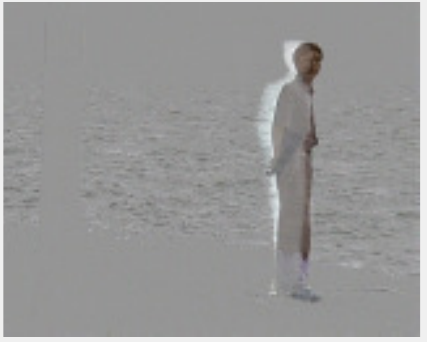}		
	\end{subfigure}
	\begin{subfigure}[b]{0.118 \textwidth}
		\centering
		\includegraphics[width=\textwidth]{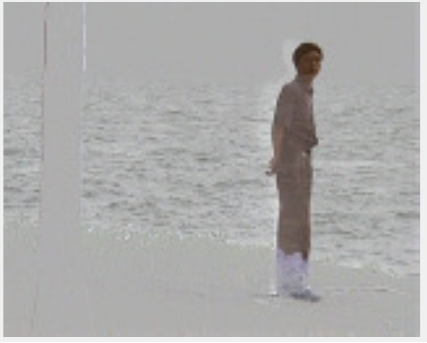}
	\end{subfigure}
	\begin{subfigure}[b]{0.118 \textwidth}
		\centering
		\includegraphics[width=\textwidth]{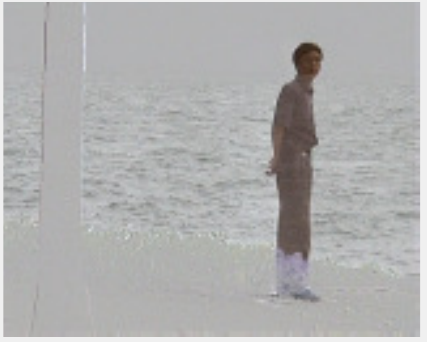}
	\end{subfigure}

	\begin{subfigure}[b]{0.118\textwidth}
		\centering
		\includegraphics[width=\textwidth]{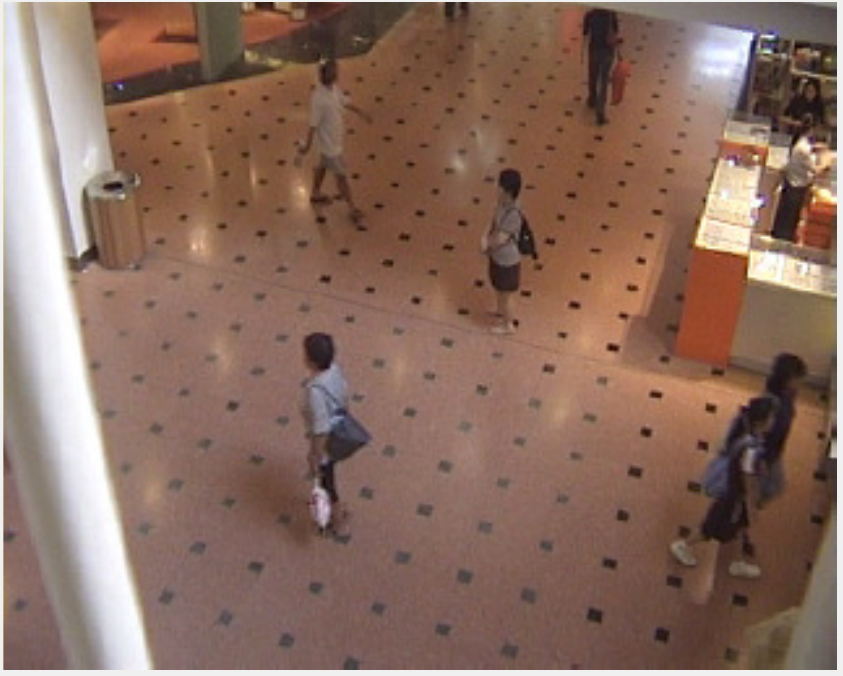}
	\end{subfigure}
	\begin{subfigure}[b]{0.118\textwidth}
		\centering
		\includegraphics[width=\textwidth]{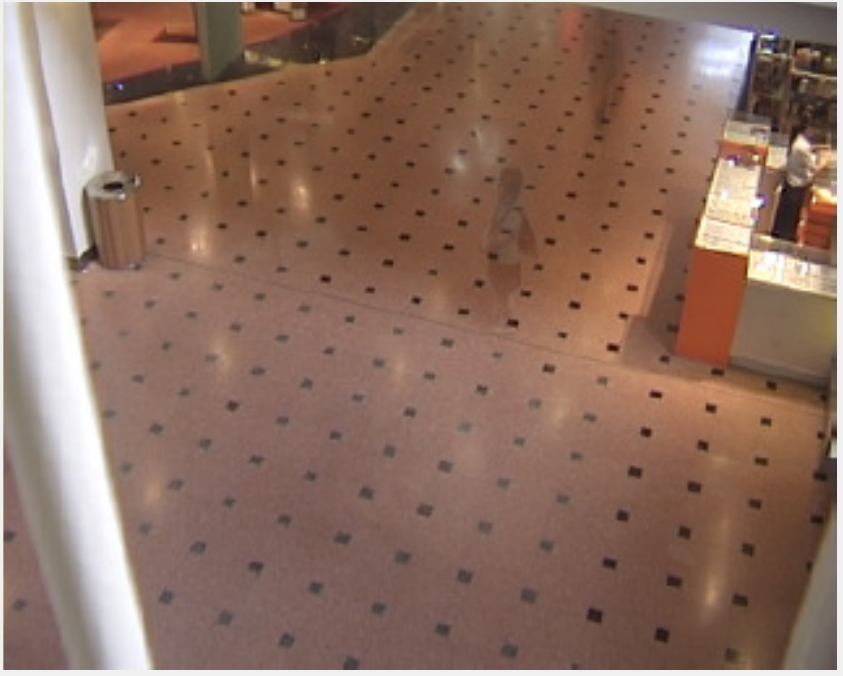}
	\end{subfigure}
	\begin{subfigure}[b]{0.118 \textwidth}
		\centering
		\includegraphics[width=\textwidth]{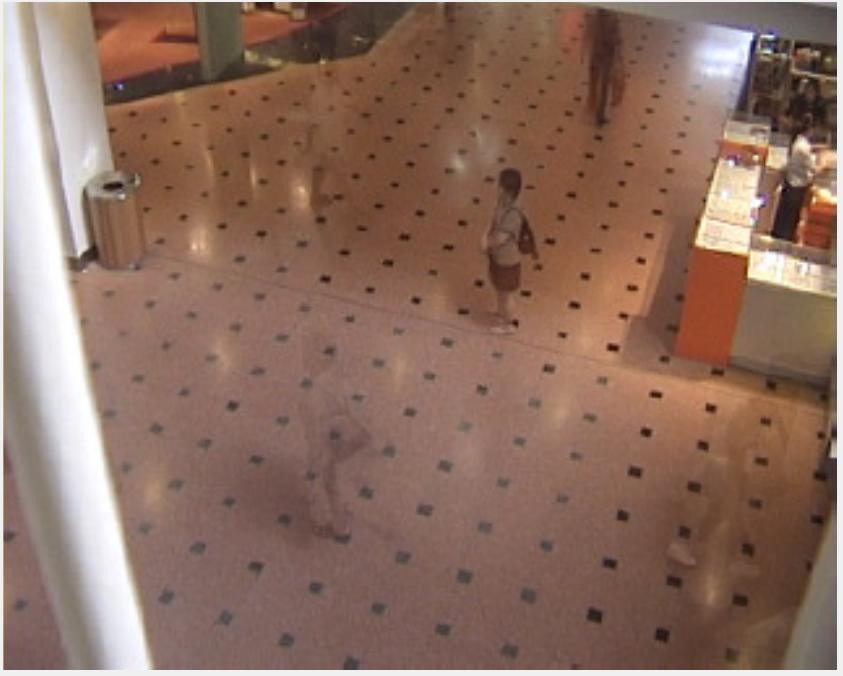}
	\end{subfigure}
	\begin{subfigure}[b]{0.118 \textwidth}
		\centering
		\includegraphics[width=\textwidth]{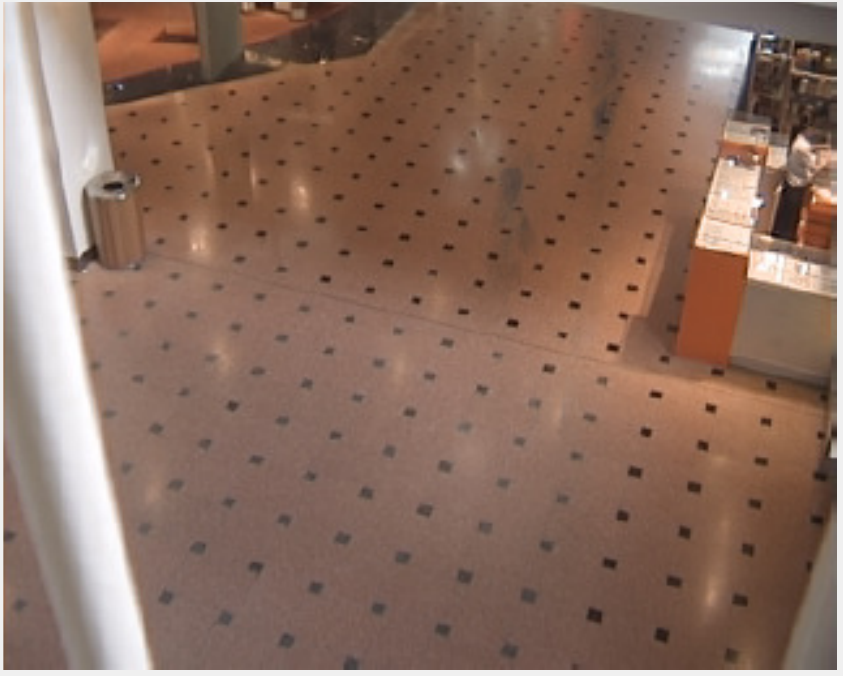}
	\end{subfigure}
	\begin{subfigure}[b]{0.118\textwidth}		
	\end{subfigure}
	\hfill
	\begin{subfigure}[b]{0.118\textwidth}
		\centering
		\includegraphics[width=\textwidth]{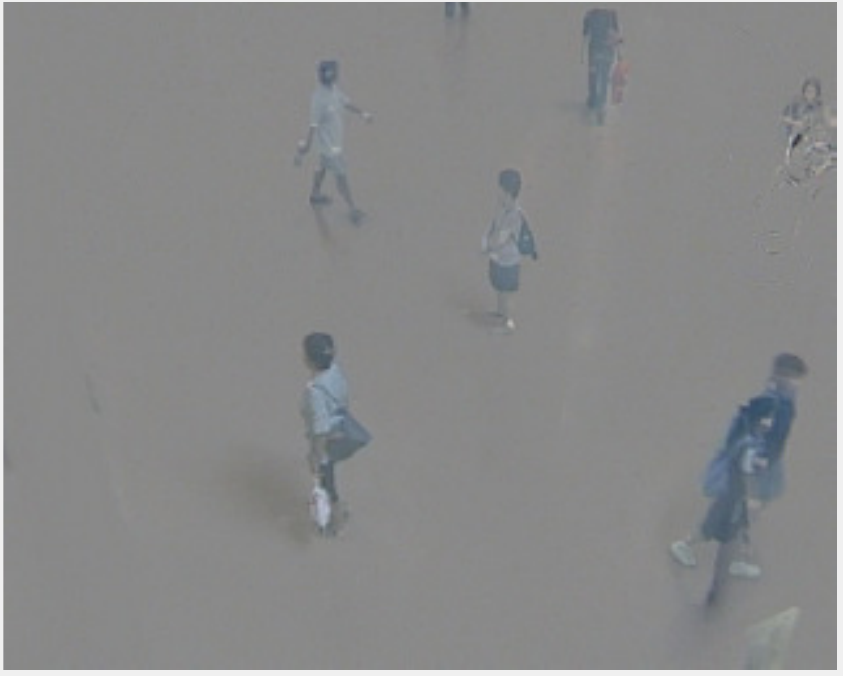}
	\end{subfigure}
	\begin{subfigure}[b]{0.118 \textwidth}
		\centering
		\includegraphics[width=\textwidth]{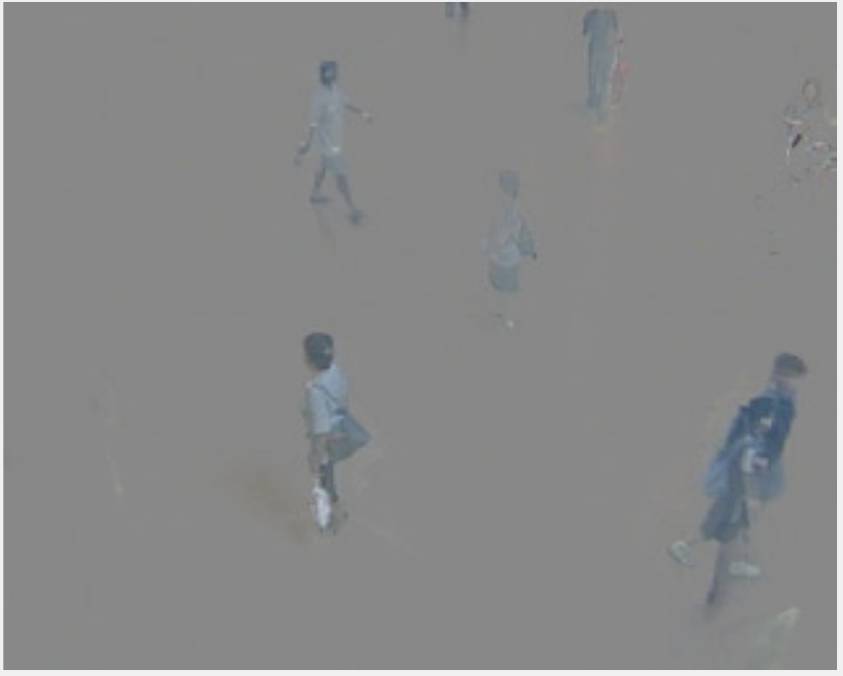}
	\end{subfigure}
	\begin{subfigure}[b]{0.118 \textwidth}
		\centering
		\includegraphics[width=\textwidth]{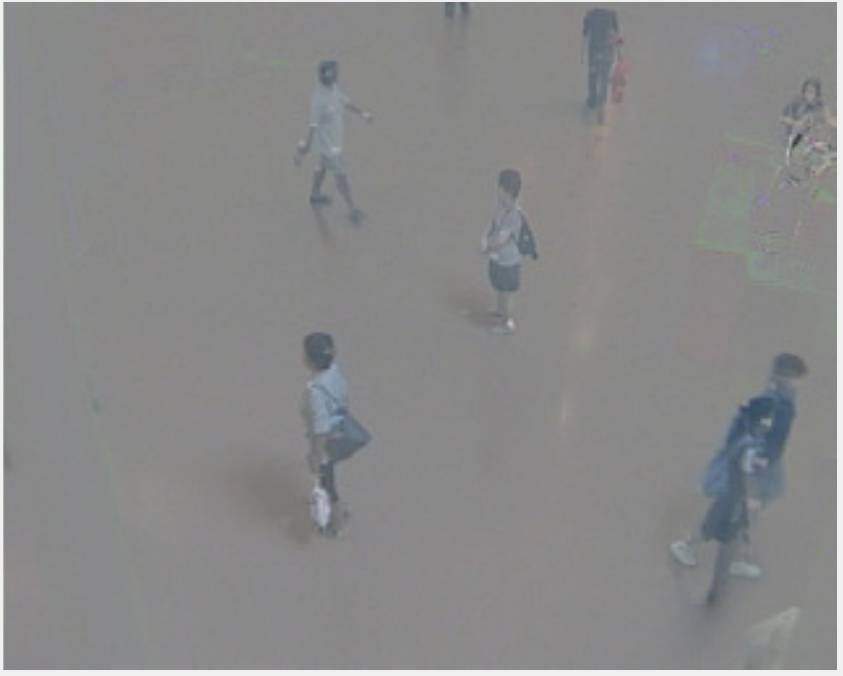}
	\end{subfigure}
	
	\begin{subfigure}[b]{0.118\textwidth}
		\centering
		\includegraphics[width=\textwidth]{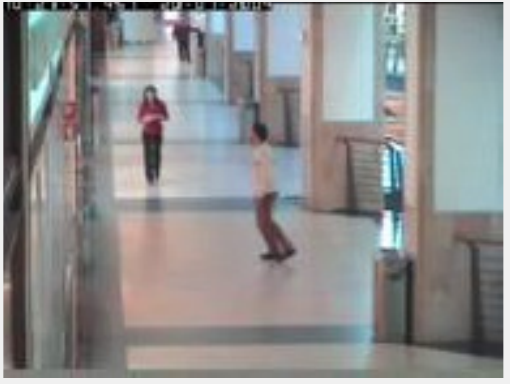}
	\end{subfigure}
	\begin{subfigure}[b]{0.118\textwidth}
		\centering
		\includegraphics[width=\textwidth]{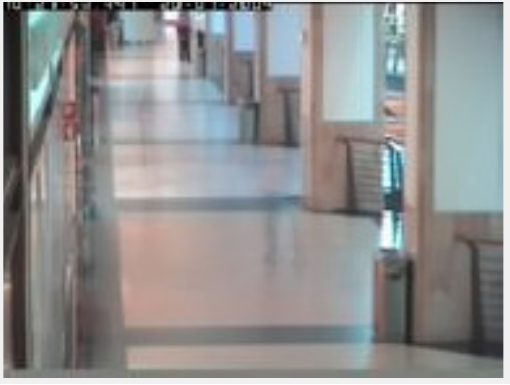}
	\end{subfigure}
	\begin{subfigure}[b]{0.118 \textwidth}
		\centering
		\includegraphics[width=\textwidth]{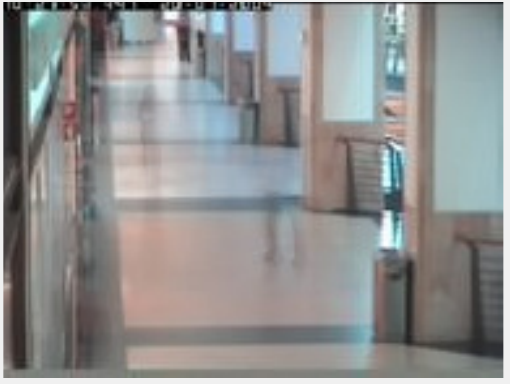}
	\end{subfigure}
	\begin{subfigure}[b]{0.118 \textwidth}
		\centering
		\includegraphics[width=\textwidth]{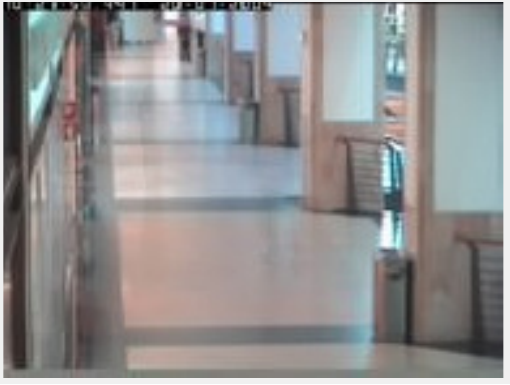}
	\end{subfigure}
	\begin{subfigure}[b]{0.118\textwidth}			
		\caption{Original}	
	\end{subfigure}
	\hfill
	\begin{subfigure}[b]{0.118\textwidth}
		\centering
		\includegraphics[width=\textwidth]{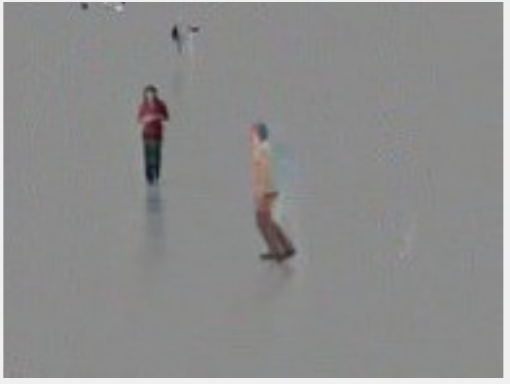}
		\caption{RPCA}
	\end{subfigure}
	\begin{subfigure}[b]{0.118 \textwidth}
		\centering
		\includegraphics[width=\textwidth]{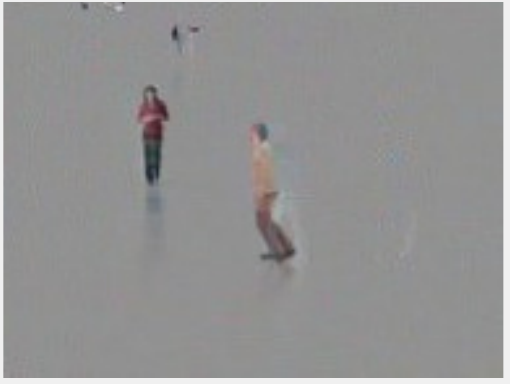}
		\caption{SNN}
	\end{subfigure}
	\begin{subfigure}[b]{0.118 \textwidth}
		\centering
		\includegraphics[width=\textwidth]{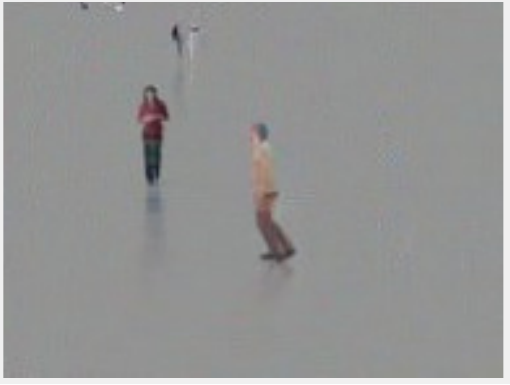}
		\caption{TRPCA}
	\end{subfigure}
		
	\vspace{0.5em}\small 
	\centering
	\begin{minipage}[c]{0.48\textwidth}
		\centering
		\begin{tabular}{c|c|c|c}
			\hline
			& RPCA & SNN & TRPCA \\ \hline
			\textit{Hall}   & \textbf{301.8} & 1553.2 & 323.0 \\ \hline
			\textit{WaterSurface} & 250.1 & 887.3 & \textbf{224.2} \\ \hline
			\canyi{\textit{ShoppingMall}} & \canyi{\textbf{260.9}} & \canyi{744.0} & \canyi{372.4} \\ \hline
			\canyi{\textit{ShopCorridor}} & \canyi{\textbf{321.7}} & \canyi{1438.6} & \canyi{371.3} \\ \hline
		\end{tabular}
		\caption*{\small (e) Running time (seconds) comparison}
	\end{minipage}
	\vspace{-0.1cm}
	\caption{\small{\canyi{Background modeling results of four surveillance video sequences. (a) Original frames; (b)-(d) low rank and sparse components obtained by RPCA, SNN and TRPCA, respectively; (e) running time comparison.}}}
	\label{fig_bg_res}
	\vspace{-0.6cm}
\end{figure}

\section{Conclusions and Future Work} \label{sec_con}

Based on the recently developed tensor-tensor product, which is a natural extension of the matrix-matrix product, we rigorously defined the tensor spectral norm, tensor nuclear norm and tensor average rank, such that their properties and relationships are consistent with the matrix cases. We then studied the Tensor Robust Principal Component (TRPCA) problem which aims to recover a low tubal rank tensor and a sparse tensor from their sum. We proved that under certain suitable assumptions, we can recover both the low-rank and the sparse components exactly by simply solving a convex program whose objective is a weighted combination of the tensor nuclear norm and the $\ell_1$-norm. Benefitting from the ``good" property of tensor nuclear norm, both our model and theoretical guarantee are natural extensions of RPCA. We also developed a more efficient method to compute the tensor singular value thresholding problem which is the key for solving TRPCA.  Numerical experiments verify our theory and the results on images and videos demonstrate the effectiveness of our model. 

\canyi{There have some interesting future works. The work \cite{kernfeld2015tensor} generalizes the t-product  using any invertible linear transform. With a proper choice of the invertible linear transform, it is possible to deduce a new tensor nuclear norm and solve the TRPCA problem. Beyond the convex models, the extensions to nonconvex cases are also important \cite{lu2015generalized}. Finally, it is always interesting in using the developed tensor tools for real applications, e.g., image/video processing, web data analysis, and bioinformatics.}

%

{
	\bibliographystyle{ieee}
	\bibliography{ref}
}
~\
\begin{center}
	\LARGE\textbf{Appendix}
\end{center}

At the following, we give the detailed proofs of Theorem 3.1, Theorem 3.2, and the main result in Theorem 4.1. Section \ref{sec_notations} first gives some notations and properties which will be used in the proofs. Section \ref{profthm3132} gives the proofs of Theorem 3.1 and 3.2 in our paper. Section \ref{sec_dual} provides a way for the construction of the solution to the TRPCA model, and Section \ref{sec_proofofdual} proves that the constructed solution is optimal to the TRPCA problem. Section \ref{sec_prooflemmas} gives the proofs of some lemmas which are used in Section \ref{sec_proofofdual}.

\appendices

\section{Preliminaries}
\label{sec_notations}


Beyond the notations introduced in the paper, we need some other notations used in the proofs.
At the following, we define $\eijk=\ei*\ek*\ej^*$. Then we have $\X_{ijk} = \inproduct{\X}{\eijk}$. We define the projection $$\Pomega(\Z)=\sum_{ijk}\delta_{ijk}z_{ijk}\eijk,$$
where $\delta_{ijk}=1_{(i,j,k)\in\Omegat}$, where $1_{(\cdot)}$ is the indicator function. Also $\Omegat^c$ denotes the 
complement of $\Omegat$ and $\Pomegao$ is the projection onto $\Omegat^c$. Denote $\Tm$ by the set 
\begin{align}
	\Tm = \{ \U*\Y^* + \W*\V^*, \ \Y, \W\in\mathbb{R}^{n\times r\times n_3} \},
\end{align}
and by $\Tm^\bot$  its orthogonal complement. Then the projections onto $\Tm$ and $\Tm^\bot$ are respectively
\begin{align*}
	\PT(\Z) = \U*\U^**\Z + \Z *\V*\V^* -  \U*\U^**\Z *\V*\V^*, 
\end{align*}
\begin{align*}
	\PTo(\Z) =& \Z-\PT(\Z)\\
	=&(\I_{n_1} - \U*\U^*)*\Z*(\I_{n_2}-\V*\V^*),
\end{align*}
where $\I_n$ denotes the $n\times n\times n_3$ identity tensor. Note that $\PT$ is self-adjoint. So we have 
\begin{align*}
	&\norm{\PT(\eijk)}_F^2 \\
	= & \inproduct{\PT(\eijk)}{\eijk} \\
	= & \inproduct{\U*\U^**\eijk + \eijk *\V*\V^*}{\eijk} \\
	& -\inproduct{ \U*\U^**\eijk *\V*\V^*}{\eijk}
\end{align*}
Note that 
\begin{align*}
	&\inproduct{\U*\U^**\eijk}{\eijk}\\
	=&\inproduct{\U*\U^**\ei*\ek*\ej^*}{\ei*\ek*\ej^*}\\
	=&\inproduct{\U^**\ei}{\U^**\ei*(\ek*\ej^**\ej*\ek^*)}\\
	=&\inproduct{\U^**\ei}{\U^**\ei}\\
	=&\norm{\U^**\ei}_F^2,
\end{align*}
where we use the fact that $\ek*\ej^**\ej*\ek^*=\I_1$, which is the $1\times 1\times n_3$ identity tensor. Therefore, it is easy to see that 
\begin{align}
	&\norm{\PT(\eijk)}_F^2 \notag \\
	= & \norm{\U^**\ei}_F^2 + \norm{\V^**\ej}_F^2 - \norm{\U^**\ei*\ek*\ej^**\V}_F^2, \notag\\
	\leq & \norm{\U^**\ei}_F^2 + \norm{\V^**\ej}_F^2 \notag \\
	\leq & \frac{\mu r(n_1+n_2)}{n_1n_2n_3}  \label{proabouPTn1n2}\\
	= & \frac{2\mu r}{nn_3}, \ \text{when } n_1=n_2=n.\label{proabouPT}
\end{align}
where (\ref{proabouPTn1n2}) uses the following tensor incoherence conditions
\begin{align}
	\max_{i=1,\cdots,n_1} \norm{\U^**\mathring{\mathfrak{e}}_i}_F\leq\sqrt{\frac{\mu r}{n_1n_3}}, \label{tic1}\\
	\max_{j=1,\cdots,n_2} \norm{\V^**\mathring{\mathfrak{e}}_j}_F\leq\sqrt{\frac{\mu r}{n_2n_3}},\label{tic2}
\end{align}
and
\begin{equation} \label{tic3}
\norm{\U*\V^*}_\infty\leq \sqrt{\frac{\mu r}{n_1n_2n_3^2}},
\end{equation}	
which are assumed to be satisfied in Theorem 4.1 in our manuscript.


\section{Proofs of Theorem 3.1 and Theorem 3.2}\label{profthm3132}

\subsection{Proof of Theorem 3.1}

\begin{proof}
	To complete the proof, we need the conjugate function concept. The conjugate $\phi^*$ of a function $\phi: C\rightarrow \mathbb{R}$, where $C \subset \mathbb{R}^n$, is defined as
	\begin{equation*}
		\phi^*(\y) = \sup \{\inproduct{\y}{\x} - \phi(\x) | \x \in C \}.
	\end{equation*}
	Note that the conjugate of the conjugate, $\phi^{**}$, is the convex envelope of the function $\phi$. See Theorem 1.3.5 in \cite{hiriart1991convex,fazel2002matrix}.  The proofs has two steps which compute $\phi^*$ and $\phi^{**}$, respectively.
	
	\textbf{Step 1.} \textit{Computing $\phi^*$.} For any $\A\in\Rn$, the conjugate function of the tensor average rank
	\begin{equation*}
		\phi(\A) = {\ranka(\A)}=\frac{1}{n_3}\rank(\bcirc(\A))=\frac{1}{n_3}\rank(\Ambar),
	\end{equation*}
	on the set $S = \{\A\in\Rn |\norm{\A}\leq 1\}$ is
	\begin{align*}
		\phi^*(\B) = & \sup_{\norm{\A}\leq 1} \left(\inproduct{\B}{\A} - {\ranka(\A)}\right) \\
		= & \sup_{\norm{\A}\leq 1} \frac{1}{n_3}(\inproduct{\Bmbar}{\Ambar} - \rank(\Ambar)).
	\end{align*}
	Here $\Ambar, \Bmbar\in\mathbb{C}^{n_1n_3\times n_2n_3}$. Let $q = \min\{n_1n_3, n_2n_3\}$. By von Neumann's trace theorem, 
	\begin{align}\label{vonneutrace}
		\inproduct{\Bmbar}{\Ambar} \leq \sum_{i = 1}^{q} \sigma_i(\Bmbar) \sigma_i(\Ambar),
	\end{align} 
	where $\sigma_i(\Ambar)$ denotes the $i$-th largest singular value of $\Ambar$. Let $\Ambar = \Umbar_1\Smbar_1\Vm^*_1$ and $\Bmbar = \Umbar_2\Smbar_2\Vmbar^*_2$ be the SVD of $\Ambar$ and $\Bmbar$, respectively. Note that the equality (\ref{vonneutrace}) holds when
	\begin{align}\label{conduvab}
		\Umbar_1 = \Umbar_2  \ \text{and} \ \Vmbar_1 = \Vmbar_2.
	\end{align}
	So we can pick $\Umbar_1$ and $\Vmbar_1$ such that (\ref{conduvab}) holds to maximize $\inproduct{\Bmbar}{\Ambar}$. Note that the corresponding $\U$ and $\V$ of $\Umbar_1$ and $\Vmbar_1$  respectively are real tensors and so is $\A$ in this case. Thus, we have
	\begin{align*}
		\phi^*(\B) = & \sup_{\norm{\A}\leq 1} \frac{1}{n_3}\left(\sum_{i = 1}^{q} \sigma_i(\Bmbar) \sigma_i(\Ambar) - \rank(\Ambar)\right).
	\end{align*}
	
	If $\A = 0$, then $\Ambar = 0$, and thus we have $\phi^*(\B) = 0$ for all $\B$. If $\rank(\Ambar) = r$, $1\leq r\leq q$, then $\phi^*(\B) = \frac{1}{n_3}\left(\sum_{i = 1}^{r} \sigma_i(\Bmbar) - r\right)$. Hence $\phi^*(\B)$ can be expressed as
	\begin{align*}
		& n_3 \cdot \phi^*(\B) \\
		= & {\max \left\{0, \sigma_1(\Bmbar) -1, \cdots, \sum_{i = 1}^{r}\sigma_i(\Bmbar) -r, \cdots,  \sum_{i = 1}^{q} \sigma_i(\Bmbar) - q\right\}}.
	\end{align*}
	The largest term in this set is the one that sums all positive $(\sigma_i(\Bmbar) -1)$ terms. Thus, we have
	\begin{align*}
		&\phi^*(\B) \\
		= & \begin{cases}
			0, \qquad\qquad \qquad \quad \quad \quad \quad       \norm{\Bmbar}\leq 1, \\
			\frac{1}{n_3}\left(\sum_{i = 1}^{r}\sigma_i(\Bmbar) -r\right), \quad \sigma_r(\Bmbar) > 1 \text{ and } \sigma_{r+1}(\Bmbar) \leq 1 
		\end{cases} \\
		= &  \frac{1}{n_3} \sum_{i = 1}^{q}(\sigma_i(\Bmbar) -1)_+. 	\end{align*}
	Note that above $\norm{\Bmbar}\leq 1$ is equivalent to $\norm{\B}\leq 1$.
	
	\textbf{Step 2.} \textit{Computing $\phi^{**}$.} Now we compute the conjugate of $\phi^*$, defined as
	\begin{align*}
		\phi^{**}(\C) = & \sup_{\B} ( \inproduct{\C}{\B} - \phi^*(\B) ) \\
		= & \sup_{\B} \left(\frac{1}{n_3} \inproduct{\Cmbar}{\Bmbar}-\phi^*(\B)\right),
	\end{align*}
	for all $\C \in S$. As before, we can choose $\B$ such that 
	\begin{align*}
		\phi^{**}(\C) = & \sup_{\B} \left(\frac{1}{n_3}\sum_{i = 1}^{q} \sigma_i(\Cmbar) \sigma_i(\Bmbar) - \phi^*(\B)\right).
	\end{align*}
	At the following, we consider two cases, $\norm{\C} > 1$ and $\norm{\C} \leq 1$.
	
	If $\norm{\C} > 1$, then $\sigma_1(\Cmbar) = \norm{\Cmbar} = \norm{\C} > 1$.  We can choose $\sigma_1(\Bmbar)$ large enough so that $\phi^{**}(\C)\rightarrow\infty$. To see this, note that  in
	\begin{align*}
		\phi^{**}(\C) = & \sup_{\B}\frac{1}{n_3} \left(\sum_{i = 1}^{q} \sigma_i(\Cmbar) \sigma_i(\Bmbar) - \left( \sum_{i=1}^{r} \sigma_i(\Bmbar)  - r \right)\right),
	\end{align*}
	the coefficient of $\sigma_1(\Bmbar)$ is $\frac{1}{n_3}(\sigma_1(\Cmbar)-1)$ which is positive.
	
	If $\norm{\C} \leq 1$, then $\sigma_1(\Cmbar) = \norm{\Cmbar} = \norm{\C} \leq 1$. If $\norm{\B} = \norm{\Bmbar} \leq 1$, then $\phi^*(\B)=0$ and the supremum is achieved for $\sigma_i(\Bmbar)=1$, $i=1,\cdots,q$, yielding 
	\begin{align*}
		\phi^{**}(\C) = \frac{1}{n_3} \sum_{i=1}^{q} \sigma_i(\Cmbar ) = \frac{1}{n_3} \norm{\Cmbar}_* = \norm{\C}_*.
	\end{align*}
	If $\norm{\C}>1$, we show that the argument of sup is is always smaller than $\norm{\C}_*$. By adding and subtracting the term $\frac{1}{n_3}\sum_i^q\sigma_i(\Cmbar)$ and rearranging the terms, we have
	\begin{align*}
		&\frac{1}{n_3} \left(\sum_{i = 1}^{q} \sigma_i(\Cmbar) \sigma_i(\Bmbar) -  \sum_{i=1}^{r}  \left(\sigma_i(\Bmbar)  - 1 \right)\right) \\
		=& \frac{1}{n_3} \left(\sum_{i = 1}^{q} \sigma_i(\Cmbar) \sigma_i(\Bmbar) -  \sum_{i=1}^{r}  \left(\sigma_i(\Bmbar)  - 1 \right)\right) \\
		& - \frac{1}{n_3}\sum_{i=1}^q\sigma_i(\Cmbar) + \frac{1}{n_3}\sum_{i=1}^q\sigma_i(\Cmbar) \\
		=&\frac{1}{n_3} \sum_{i=1}^r (\sigma_i(\Bmbar)-1)(\sigma_i(\Cmbar)-1)\\
		&+\frac{1}{n_3}\sum_{i=r+1}^q (\sigma_i(\Bmbar)-1)\sigma_i(\Cmbar) +\frac{1}{n_3} \sum_{i=1}^{q} \sigma_i(\Cmbar) \\
		< & \frac{1}{n_3}\sum_{i=1}^{q} \sigma_i(\Cmbar)\\
		= & \norm{\C}_*.
	\end{align*}
	
	In a summary, we have shown that
	\begin{align*}
		\phi^{**}(\C) = \norm{\C}_*, 
	\end{align*}
	over the set $S=\{\C | \norm{\C} \leq 1\}$. Thus, $\norm{\C}_*$ is the convex envelope of the tensor average rank ${\ranka(\C)}$ over $S$.
\end{proof}

\subsection{Proof of Theorem 3.2}

\begin{proof} Let $\G\in \partial \norm{\A}_*$. It is equivalent to the following statements   \cite{watson1992characterization}
	\begin{align}
		\norm{\A}_* &= \inproduct{\G}{\A}, \label{subp1}\\
		\norm{\G} &\leq 1. \label{subp2}
	\end{align}
	So, to   complete the proof, we only need to show that $\G = \U *\V^*+\W$, where $\U^**\W=\0$, $\W*\V=\0$ and $\norm{\W}\leq1$, satisfies (\ref{subp1}) and (\ref{subp2}). 
	First, we have
	\begin{align*}
		\inproduct{\G}{\A} = & \inproduct{\U*\V^*+\W}{\U*\Sbm*\V^*} \\
		= & \inproduct{\I}{\Sbm}+0 \\
		= & \norm{\A}_*.
	\end{align*}
	Also, (\ref{subp2}) is obvious when considering the property of $\W$.  The proof is completed. 
\end{proof}

\section{Dual Certification}
\label{sec_dual}

In this section, we first introduce  conditions for  $(\LL_0,\Sbm_0)$  to be the unique solution to TRPCA in subsection \ref{subdualcertif}.  Then we construct a dual certificate in subsection \ref{subsecducergs} which satisfies the conditions in subsection \ref{subdualcertif}, and thus our main result in Theorem 4.1 in our paper are proved.

\subsection{Dual Certificates}\label{subdualcertif}

\begin{lemma}\label{lem_dual2}
	Assume that $\norm{\Pomega\PT}\leq \frac{1}{2}$ and $\lambda<\frac{1}{\sqrt{n_3}}$. Then $(\LL_0,\Sbm_0)$ is the unique solution to the TRPCA problem if there is a pair $(\W,\F)$ obeying
	\begin{equation*}
		(\U*\V^*+\W)=\lambda(\sgn{\Sbm_0}+\F+\Pomega\D),
	\end{equation*}
	with $\PT\W=\0$, $\norm{\W}\leq\frac{1}{2}$, $\Pomega\F=\0$ and $\norm{\F}_{\infty}\leq\frac{1}{2}$, and $\norm{\Pomega\D}_F\leq\frac{1}{4}$.
\end{lemma}

\begin{proof}
	For any $\HH\neq\0$, $(\LL_0+\HH,\Sbm_0-\HH)$ is also a feasible solution. We show that its objective is larger than that at $(\LL_0,\Sbm_0)$, hence proving that $(\LL_0,\Sbm_0)$ is the unique solution. To do this, let $\U*\V^*+\W_0$ be an arbitrary subgradient of the tensor nuclear norm at $\LL_0$, and $\sgn{\Sbm_0}+\F_0$ be an arbitrary subgradient of the $\ell_1$-norm at $\Sbm_0$. Then we have
	\begin{align*}
		& \norm{\LL_0+\HH}_*+\lambda\norm{\Sbm_0-\HH}_1 \\
		\geq& \norm{\LL_0}_*+\lambda\norm{\Sbm_0}_1+\langle\U*\V^*+\W_0,\HH\rangle \\
		&-\lambda\inproduct{\sgn{\Sbm_0}+\F_0}{\HH}.
	\end{align*}
	Now pick $\W_0$ such that $\inproduct{\W_0}{\HH}=\norm{\PTo\HH}_*$ and $\inproduct{\F_0}{\HH}=-\norm{\Pomegao\HH}$. We have
	\begin{align*}
		& \norm{\LL_0+\HH}_*+\lambda\norm{\Sbm_0-\HH}_1 \\
		\geq& \norm{\LL_0}_*+\lambda\norm{\Sbm_0}_1+\norm{\PTo\HH}_*+\lambda\norm{\Pomegao\HH}_1\\
		&+\inproduct{\U*\V^*-\lambda\sgn{\Sbm_0}}{\HH}.
	\end{align*}
	By assumption
	\begin{align*}
		&\abs{\inproduct{\U*\V^*-\lambda\sgn{\Sbm_0}}{\HH}}\\
		\leq& \abs{\inproduct{\W}{\HH}}+\lambda\abs{\inproduct{\F}{\HH}}+\lambda\abs{\inproduct{\Pomega\D}{\HH}}\\
		\leq&\beta(\norm{\PTo\HH}_*+\lambda\norm{\Pomegao\HH}_1)+\frac{\lambda}{4}\norm{\Pomega\HH}_F,
	\end{align*}
	where $\beta=\max(\norm{\W},\norm{\F}_\infty)<\frac{1}{2}$. Thus	
	\begin{align*}
		& \norm{\LL_0+\HH}_*+\lambda\norm{\Sbm_0-\HH}_1 \\
		\geq& \norm{\LL_0}_*+\lambda\norm{\Sbm_0}_1+\frac{1}{2}(\norm{\PTo\HH}_*+\lambda\norm{\Pomegao\HH}_1) \\
		& -\frac{\lambda}{4}\norm{\Pomega\HH}_F.
	\end{align*}
	On the other hand,
	\begin{align*}
		\norm{\Pomega\HH}_F\leq & \norm{\Pomega\PT\HH}_F+\norm{\Pomega\PTo\HH}_F \\
		\leq & \frac{1}{2}\norm{\HH}_F+\norm{\PTo\HH}_F \\
		\leq & \frac{1}{2}\norm{\Pomega\HH}_F+\frac{1}{2}\norm{\Pomegao\HH}_F+\norm{\PTo\HH}_F.
	\end{align*}
	Thus
	\begin{align*}
		\norm{\Pomega\HH}_F\leq & \norm{\Pomegao\HH}_F+2\norm{\PTo\HH}_F \\
		\leq & \norm{\Pomegao\HH}_1+{2}{\sqrt{n_3}}\norm{\PTo\HH}_*.
	\end{align*}
	In conclusion, 
	\begin{align*}
		& \norm{\LL_0+\HH}_*+\lambda\norm{\Sbm_0-\HH}_1 \\
		\geq& \norm{\LL_0}_*+\lambda\norm{\Sbm_0}_1+\frac{1}{2}\left(1-{\lambda}{\sqrt{n_3}}\right)\norm{\PTo\HH}_*\\
		&+\frac{\lambda}{4}\norm{\Pomegao\HH}_1,
	\end{align*}
	where the last two terms are strictly positive when $\HH\neq \0$. Thus, the proof is completed.
\end{proof}

Lemma \ref{lem_dual2} implies that it is suffices to produce a dual certificate $\W$ obeying
\begin{equation}\label{conddualcertf}
\begin{cases}
\W \in \Tm^\bot, \\
\norm{\W} < \frac{1}{2},\\
\norm{\Pomega(\U*\V^*+\W-\lambda\sgn{\Sbm_0})}_F\leq \frac{\lambda}{4},\\
\norm{\Pomegao(\U*\V^*+\W)}_\infty < \frac{\lambda}{2}.
\end{cases}
\end{equation}

\subsection{Dual Certification via the Golfing Scheme}\label{subsecducergs}
In this subsection, we show how to construct a dual certificate obeying (\ref{conddualcertf}). Before we introduce our construction, our model assumes that $\Omegat\sim \Ber(\rho)$, or equivalently that $\Omegat^c\sim\Ber(1-\rho)$. Now the distribution of $\Omegat^c$ is the same as that of
$\Omegat^c = \Omegat_1 \cup \Omegat_2 \cup \cdots \cup \Omegat_{j_0}$, where each $\Omegat_j$  follows the Bernoulli model with parameter $q$, which satisfies
\begin{equation*}
	\mathbb{P}((i,j,k)\in\Omegat) = \mathbb{P}(\text{Bin}(j_0,q)=0) = (1-q)^{j_0},
\end{equation*}
so that the two models are the same if 
$\rho = (1-q)^{j_0}$.
Note that  because of overlaps between the $\Omegat_j$'s, $q\geq (1-\rho)/j_0$.

Now, we construct a dual certificate
\begin{equation}
\W = \WL+\WS,
\end{equation}
where each component is as follows:
\begin{enumerate}
	\item  Construction of $\WL$ via the Golfing Scheme. Let $j_0= 2\log(nn_3)$ and $\Omegat_j$, $j=1,\cdots,j_0$, be defined as previously described so that $\Omegat^c = \cup_{1\leq j \leq j_0}\Omegat_j $. Then define
	\begin{equation}\label{eq_wl}
	\WL = \PTo\Y_{j_0},
	\end{equation}
	where 
	\begin{equation*}
		\Y_j = \Y_{j-1} + q^{-1}\PP_{\Omegat_j}\PT(\U*\V^*-\Y_{j-1}), \ \Y_0=\0.
	\end{equation*}
	\item Construction of $\WS$ via the Method of Least Squares. Assume that $\norm{\Pomega\PT}<1/2$. 
	Then, $\norm{\Pomega\PT\Pomega}<1/4$, and thus, the operator $\Pomega-\Pomega\PT\Pomega$  mapping $\Omegat$ onto itself is invertible; we denote its inverse by $(\Pomega-\Pomega\PT\Pomega)^{-1}$. We then set
	\begin{equation}\label{eq_ws}
	\WS = \lambda \PTo (\Pomega-\Pomega\PT\Pomega)^{-1} \sgn{\Sbm_0}.
	\end{equation}
	This is equivalent to 
	\begin{equation*}
		\WS = \lambda \PTo \sum_{k\geq 0}   (\Pomega\PT\Pomega)^{k} \sgn{\Sbm_0}.
	\end{equation*}
\end{enumerate}
Since both $\WL$ and $\WS$ belong to $\Tm^{\bot}$ and $\Pomega\WS  = \lambda  \Pomega (\I-\PT) (\Pomega-\Pomega\PT\Pomega)^{-1} \sgn{\Sbm_0}=\lambda\sgn{\Sbm_0}$, we will establish that $\WL+\WS$ is a valid dual certificate if it obeys 
\begin{equation}
\begin{cases}
\norm{\WL+\WS} < \frac{1}{2}, \\
\norm{\Pomega(\U*\V^*+\WL)}_F\leq \frac{\lambda}{4}, \\
\norm{ \Pomegao(\U*\V^*+\WL+\WS}_\infty < \frac{\lambda}{2}.
\end{cases}
\end{equation}
This can be achieved by using the following two key lemmas:
\begin{lemma}\label{eq_keylem1}
	Assume that $\Omegat\sim\Ber(\rho)$ with parameter $\rho\leq \rho_s$ for some $\rho_s>0$. Set $j_0=2\lceil\log (nn_3)\rceil$ (use $\log(\none n_3)$ for the tensors of rectangular frontal slice). Then,  the tensor $\WL$ obeys
	\begin{enumerate}[(a)]
		\item $\norm{\WL} < \frac{1}{4}$,
		\item $\norm{\Pomega(\U*\V^*+\WL)}_F<\frac{\lambda}{4}$,
		\item $\norm{\Pomegao(\U*\V^*+\WL) }_\infty < \frac{\lambda}{4}$.
	\end{enumerate}
\end{lemma}
\begin{lemma}\label{eq_keylem2}
	Assume that $\Sbm_0$ is supported on a set $\Omegat$ sampled as in Lemma \ref{eq_keylem1}, and that the signs of $\Sbm_0$ are independent and identically distributed symmetric (and independent of $\Omegat$). Then, the tensor $\WS$ (\ref{eq_ws}) obeys
	\begin{enumerate}[(a)]
		\item $\norm{\WS} < \frac{1}{4}$,
		\item $\norm{\Pomegao\WS }_\infty < \frac{\lambda}{4}$.
	\end{enumerate}
\end{lemma}

So the left task is to prove Lemma \ref{eq_keylem1} and Lemma \ref{eq_keylem2}, which are given in Section \ref{sec_proofofdual}.

\section{Proofs of Dual Certification}
\label{sec_proofofdual}

This section gives the proofs of Lemma \ref{eq_keylem1} and Lemma \ref{eq_keylem2}. To do this, we first introduce some lemmas with their proofs given in Section \ref{sec_prooflemmas}. 

\begin{lemma}\label{lemspectrm}
	For the Bernoulli sign variable $\M\in\mathbb{R}^{n\times n\times n_3}$ defined as
	\begin{equation}\label{defM0}
	\M_{ijk} = \begin{cases}
	1, & \text{w.p.} \ \rho/2, \\
	0, & \text{w.p.} \ 1-\rho, \\
	-1, & \text{w.p.} \ \rho/2,
	\end{cases}
	\end{equation}
	where $\rho>0$, there exists a function $\varphi(\rho)$  satisfying  $\lim\limits_{\rho\rightarrow0^+}\varphi(\rho)=0$, such that the following statement holds with with large probability,
	\begin{equation*}\label{boundspecm}
		\norm{\M} \leq \varphi(\rho)\sqrt{ nn_3}.
	\end{equation*}
\end{lemma}

\begin{lemma}\label{lem_kem1}
	Suppose $\Omegat\sim\Ber(\rho)$. Then with high probability,
	\begin{equation*}
		\norm{\PT-\rho^{-1}\PT\Pomega\PT}\leq\epsilon,
	\end{equation*}
	provided that $\rho\geq C_0\epsilon^{-2}(\mu r\log (nn_3))/(nn_3)$ for some numerical constant $C_0>0$. For the tensor of  rectangular frontal slice, we need $\rho\geq C_0\epsilon^{-2}(\mu r\log (\none n_3))/(\ntwo n_3)$.
\end{lemma}
\begin{cor}\label{corollar31}
	Assume that $\Omegat\sim\Ber(\rho)$, then $\norm{\Pomega\PT}^2\leq \rho+\epsilon$, provided that $1-\rho\geq C\epsilon^{-2}(\mu r\log(nn_3))/(nn_3)$, where $C$ is as in Lemma \ref{lem_kem1}. For the tensor with frontal slice, the modification is as in Lemma  \ref{lem_kem1}.
\end{cor}
Note that this corollary shows that $\norm{\Pomega\PT}\leq 1/2$, provided $|\Omegat|$ is not too large. 
\begin{lemma}\label{lem_keyinf}
	Suppose that $\Z\in\Tm$ is a fixed tensor, and $\Omegat\sim\Ber(\rho)$. Then, with high probability,
	\begin{equation}
	\norm{\Z-\rho^{-1}\PT\Pomega\Z}_{\infty} \leq \epsilon \norm{\Z}_\infty,
	\end{equation}
	provided that $\rho\geq C_0 \epsilon^{-2}(\mu r \log(nn_3))/(nn_3)$ (for the tensor of  rectangular frontal slice, $\rho\geq C_0\epsilon^{-2}(\mu r\log (\none n_3))/(\ntwo n_3)$) for some numerical constant $C_0>0$.
\end{lemma}

\begin{lemma}\label{lempre3}
	Suppose $\Z$ is fixed, and $\Omegat\sim\Ber(\rho)$. Then, with high probability,
	\begin{equation}
	\norm{(\I-\rho^{-1}\Pomega)\Z} \leq \sqrt{\frac{C_0nn_3\log(nn_3)}{\rho}} \norm{\Z}_\infty,
	\end{equation}
	for some numerical constant $C_0>0$ provided that $\rho\geq C_0\log(nn_3)/(nn_3)$ (or $\rho\geq C_0\log(\none n_3)/(\ntwo n_3)$ for the tensors with rectangular frontal slice).
\end{lemma} 

\subsection{Proof of Lemma \ref{eq_keylem1}}
\begin{proof}
	We first introduce some notations. Set $\Z_j =   \U*\V^* - \PT\Y_j$ obeying
	\begin{equation*}
		\Z_j = (\PT-q^{-1}\PT\PP_{\Omegat_j}\PT)\Z_{j-1}.
	\end{equation*}
	So $\Z_j\in\T$ for all $j\geq 0$. Also, note that when
	\begin{equation}\label{eq_q}
	q\geq C_0\epsilon^{-2}\frac{\mu r\log(nn_3)}{nn_3},
	\end{equation}
	or for the tensors with rectangular frontal slices $q\geq C_0\epsilon^{-2}\frac{\mu r\log(\none n_3)}{\ntwo n_3}$,
	we have
	\begin{equation}\label{eq_Zrel1}
	\norm{\Z_j}_\infty \leq \epsilon \norm{\Z_{j-1}}_\infty\leq \epsilon^j \norm{\U*\V^*}_\infty,
	\end{equation}
	by Lemma \ref{lem_keyinf} and 
	\begin{equation}\label{eq_Zrel2}
	\norm{\Z_j}_F \leq \epsilon \norm{\Z_{j-1}}_F \leq \epsilon^j \norm{\U*\V^*}_F \leq   \epsilon^j \sqrt{r}.
	\end{equation}
	We assume $\epsilon\leq e^{-1}$.
	
	\noindent\textbf{1. Proof of (a).} Note that $\Y_{j_0} = \sum_j q^{-1}\PP_{\Omegat_j}\Z_{j-1}$. We have
	\begin{align*}
		\norm{\WL} = &\norm{\PTo\Y_{j_0}} \leq \sum_j \norm{q^{-1} \PTo\PP_{\Omegat_j} \Z_{j-1} } \\
		\leq & \sum_j \norm{\PTo(q^{-1}\PP_{\Omegat_j}\Z_{j-1}-\Z_{j-1})} \\
		\leq & \sum_j  \norm{q^{-1}\PP_{\Omegat_j}\Z_{j-1}-\Z_{j-1}} \\
		\leq & C_0'\sqrt{\frac{nn_3\log(nn_3)}{q}} \sum_j \norm{\Z_{j-1}}_\infty \\
		\leq & C_0'\sqrt{\frac{nn_3\log(nn_3)}{q}} \sum_j \epsilon^{j-1}\norm{\U*\V^*}_\infty \\
		\leq & C_0'(1-\epsilon)^{-1} \sqrt{\frac{nn_3\log(nn_3)}{q}}\norm{\U*\V^*}_\infty.
	\end{align*}
	The fourth step is from Lemma \ref{lempre3} and the fifth is from (\ref{eq_Zrel1}). Now by using (\ref{eq_q}) and (\ref{tic3}), we have 
	$$\norm{\WL}\leq C'\epsilon,$$
	for some numerical constant $C'$.

	\noindent\textbf{2. Proof of (b).} Since $\Pomega\Y_{j_0} = \0$, $\Pomega(\U*\V^*+\PTo\Y_{j_0}) =  \Pomega(\U*\V^*-\PT\Y_{j_0}) = \Pomega(\Z_{j_0}) $, and it follows from (\ref{eq_Zrel2}) that $$\norm{\Z_{j_0}}_F \leq \epsilon^{j_0} \norm{\U*\V^*}_F \leq \epsilon^{j_0} \sqrt{r}.$$
	Since $\epsilon\leq e^{-1}$ and $j_0 \geq 2\log(nn_3)$, $\epsilon^{j_0}\leq (nn_3)^{-2}$ and this proves the claim.
	
	\noindent\textbf{3. Proof of (c).} We have $\U*\V^*+\WL = \Z_{j_0}+\Y_{j_0}$ and know that $\Y_{j_0}$ is supported on $\Omegat^c$. Therefore, since $\norm{\Z_{j_0}}_F\leq \lambda/8$. We only need to show that $\norm{\Y_{j_0}}_\infty\leq\lambda/8$. Indeed,
	\begin{align*}
		\norm{\Y_{j_0}}_\infty \leq & q^{-1}\sum_j \norm{\PP_{\Omegat_j} \Z_{j-1}}_\infty \\
		\leq & q^{-1}\sum_j \norm{ \Z_{j-1}}_\infty \\
		\leq & q^{-1}\sum_j\epsilon^{j-1} \norm{  \U*\V^*}_\infty.
	\end{align*}
	Since $\norm{\U*\V^*}_\infty \leq \sqrt{\frac{\mu r}{n^2n^2_3}}$, this gives
	\begin{equation*}
		\norm{\Y_{j_0}}_\infty \leq C'\frac{\epsilon^2}{\sqrt{\mu r(\log(nn_3))^2}},
	\end{equation*}
	for some numerical constant $C'$ whenever $q$ obeys (\ref{eq_q}). Since $\lambda = 1/\sqrt{nn_3}$, $\norm{\Y_{j_0}}_\infty \leq \lambda/8$ if
	\begin{align*}
		\epsilon\leq C \left(\frac{\mu r(\log(nn_3))^2}{nn_3}\right)^{1/4}.
	\end{align*}
	The claim is proved by using (\ref{eq_q}), (\ref{tic3}) and sufficiently small $\epsilon$ (provided that $\rho_r$ is sufficiently small. Note that everything is consistent since $C_0\epsilon^{-2}\frac{\mu r\log(nn_3)}{nn_3}<1$.
\end{proof}

\subsection{Proof of Lemma \ref{eq_keylem2}}
\begin{proof}
	We denote $\M=\sgn{\Sbm_0}$ distributed as
	\begin{equation*}\label{defM}
		\M_{ijk} = \begin{cases}
			1, & \text{w.p.} \ \rho/2, \\
			0, & \text{w.p.} \ 1-\rho, \\
			-1, & \text{w.p.} \ \rho/2.
		\end{cases}
	\end{equation*}
	Note that for any $\sigma>0$, $\{\norm{\Pomega\PT}\leq\sigma\}$ holds with high probability provided that $\rho$ is sufficiently small, see Corollary \ref{corollar31}.
	
	\noindent\textbf{1. Proof of (a).} By construction,
	\begin{align*}
		\WS = &\lambda\PTo\M+\lambda\PTo\sum_{k\geq 1} (\Pomega\PT\Pomega)^k\M \\
		:=&  \PTo\WS_0 + \PTo\WS_1.
	\end{align*}
	Note that  $\norm{\PTo\WS_0} \leq \norm{\WS_0} = \lambda\norm{\M}$ and $\norm{\PTo\WS_1} \leq \norm{\WS_1} = \lambda\norm{\R(\M)}$, where $\R = \sum_{k\geq1} (\Pomega\PT\Pomega)^k$. Now, we will respectively show that $\lambda\norm{\M}$ and $\lambda\norm{\R(\M)}$ are small enough when $\rho$ is sufficiently small for $\lambda=1/\sqrt{nn_3}$. Therefor, $\norm{\WS}\leq1/4$. 
	
	\textbf{1) Bound $\lambda\norm{\M}$.} 
	
	By using Lemma \ref{lemspectrm} directly, we have that $\lambda\norm{\M}\leq \varphi(\rho)$ is sufficiently small given $\lambda=1/\sqrt{nn_3}$ and $\rho$ is sufficiently small.

	\textbf{2) Bound $\norm{\R(\M)}$.} 
	
	For simplicity, let $\Z=\R(\M)$. We have
	\begin{equation}\label{boundrm1}
	\norm{\Z}=\norm{\Zmbar}=\sup_{\x\in\mathbb{S}^{nn_3-1}} \norm{\Zmbar\x}_2.
	\end{equation}
	The optimal $\x$ to (\ref{boundrm1}) is an eigenvector of $\Zmbar^*\Zmbar$. Since $\Zmbar$ is a block diagonal matrix, the optimal $\x$ has a  block sparse structure, i.e., 
	$\x \in B=\big\{  \x  \in\mathbb{R}^{nn_3} | \x = [\x_1^\top,\cdots , \x_i^\top \cdots, \x^\top_{n_3} ], \text{with }  \x_i\in\mathbb{R}^n, \text{ and there exists } j \text{ such that } \x_j\neq\0 \text{ and } \x_i=\0, i\neq j  \}$. Note that $\norm{\x}_2 = \norm{\x_j}_2 = 1$. Let $N$ be the $1/2$-net for $\mathbb{S}^{n-1}$ of size at most $5^n$ (see Lemma 5.2 in \cite{vershynin2010introduction}). Then the $1/2$-net, denoted as $N'$, for $B$ has the size at most $n_3\cdot 5^n$.  We have 
	\begin{align*}
		\norm{\R(\M)} =  &\norm{\bdiag( \overline{ \R(\M) }  )} \\
		= & \sup_{\x,\y\in B} \inproduct{\x}{\bdiag( \overline{ \R(\M) }  )\y} \\ 
		= & \sup_{\x,\y\in B}  \inproduct{\x\y^*}{\bdiag(\overline{\R({\M)}})}  \\
		=& \sup_{\x,\y\in B}  \inproduct{\bdiag^*(\x\y^*)}{\overline{\R({\M)}}} ,
	\end{align*}
	where $\bdiag^*$, the joint operator of $\bdiag$, maps the block diagonal matrix $\x\y^*$ to a tensor of size $n\times n\times n_3$. Let $\Z' = \bdiag^*(\x\y^*)$ and $\Z = \mcode{ifft}(\Z',[\ ],3)$. We have
	\begin{align*}
		\norm{\R(\M)} =& \sup_{\x,\y\in B}   \inproduct{\Z'}{\overline{\R({\M)}}} \\
		=&\sup_{\x,\y\in B} n_3   \inproduct{\Z}{{\R({\M)}}}  \\
		=&\sup_{\x,\y\in B} n_3   \inproduct{\R(\Z)}{{{\M}}}  \\
		\leq & \sup_{\x,\y\in N'} 4 n_3   \inproduct{\R(\Z)}{{{\M}}}.
	\end{align*}
	For a fixed pair $(\x,\y)$ of unit-normed vectors, define the random variable
	\begin{equation*}
		X(\x,\y) = \inproduct{4n_3\R(\Z)}{{{\M}}}.
	\end{equation*}
	Conditional on $\Omegat=\text{supp}(\M)$, the signs of $\M$ are independent and identically distributed symmetric and Hoeffding's inequality gives
	\begin{align*}
		\mathbb{P}(|X(\x,\y)|>t| \Omegat) \leq 2\exp\left(\frac{-2t^2}{\norm{4n_3\R(\Z)}_F^2}\right).
	\end{align*}
	Note that $\norm{4n_3\R(\Z))}_F\leq 4n_3\norm{\R}\norm{\Z}_F = 4\sqrt{n_3}\norm{\R}\norm{\Z'}_F =4\sqrt{n_3}\norm{\R}$. Therefore, we have 
	\begin{align*}
		\mathbb{P}\left( \sup_{\x,\y\in N'} |X(\x,\y)| >t | \Omegat \right) \leq 2|N'|^2\exp\left(-\frac{t^2}{8n_3\norm{\R}^2}\right).
	\end{align*}
	Hence,
	\begin{align*}
		\mathbb{P}\left( \norm{\R(\M)} >t | \Omegat \right) \leq 2|N'|^2\exp\left(-\frac{t^2}{8n_3\norm{\R}^2}\right).
	\end{align*}
	On the event $\{\norm{\Pomega\PT} \leq\sigma \}$,
	\begin{align*}
		\norm{\R} \leq \sum_{k\geq1} \sigma^{2k} = \frac{\sigma^2}{1-\sigma^2},
	\end{align*}
	and, therefore, unconditionally,
	\begin{align*}
		& \mathbb{P}\left( \norm{\R(\M)} >t   \right) \\
		\leq & 2|N'|^2\exp\left(-\frac{\gamma^2t^2}{8n_3}\right) + \mathbb{P}\left( \norm{\Pomega\PT} \geq \sigma   \right), \ \gamma =\frac{1-\sigma^2}{2\sigma^2} \\
		=& 2n_3^2 \cdot 5^{2n} \exp\left(-\frac{\gamma^2t^2}{8n_3}\right) + \mathbb{P}\left( \norm{\Pomega\PT} \geq \sigma   \right).
	\end{align*}
	Let $t=c\sqrt{nn_3}$, where $c$ can be a small absolute constant. Then the above inequality implies that $\norm{\R(\M)}\leq t$ with high probability.

	\noindent\textbf{2. Proof of (b)}
	Observe that 
	\begin{align*}
		\Pomegao\WS = -\lambda\Pomegao\PT(\Pomega-\Pomega\PT\Pomega)^{-1}\M.
	\end{align*}
	Now for $(i,j,k)\in\Omegat^c$, $\WS_{ijk} = \inproduct{\WS}{\eijk}$, and we have $\WS_{ijk} =\lambda \inproduct{ \Q(i,j,k)}{\M}$, where $\Q(i,j,k)$ is the tensor $-(\Pomega-\Pomega\PT\Pomega)^{-1}\Pomega\PT(\eijk)$. Conditional on $\Omegat = \text{supp}(\M)$, the signs of $\M$ are independent and identically distributed symmetric, and the Hoeffding's inequality gives
	\begin{equation*}
		\mathbb{P}(|\WS_{ijk}| > t\lambda|\Omegat) \leq 2\exp\left( -\frac{2t^2}{\norm{\Q(i,j,k)}_F^2}\right),
	\end{equation*}
	and
	\begin{align*}
		&\mathbb{P}(\sup_{i,j,k}|\WS_{ijk}| > t\lambda/n_3|\Omegat) \\
		\leq & 2n^2n_3\exp\left( -\frac{2t^2}{\sup_{i,j,k}\norm{\Q(i,j,k)}_F^2}\right).
	\end{align*}
	By using (\ref{proabouPT}), we have
	\begin{align*}
		\norm{\Pomega\PT(\eijk)}_F \leq & \norm{\Pomega\PT}\norm{\PT(\eijk)}_F \\
		\leq & \sigma\sqrt{\frac{2\mu r}{nn_3}},
	\end{align*}
	on the event $\{\norm{\Pomega\PT}\leq\sigma\}$. On the same event, we have $\norm{(\Pomega-\Pomega\PT\Pomega)^{-1}}\leq (1-\sigma^2)^{-1}$ and thus
	$\norm{\Q(i,j,k)}_F^2 \leq \frac{2\sigma^2}{(1-\sigma^2)^2}\frac{\mu r}{nn_3}$.
	Then, unconditionally,
	\begin{align*}
		& \mathbb{P}\left(\sup_{i,j,k}|\WS_{ijk}| >t\lambda\right)  \\
		\leq & 2n^2n_3\exp\left( -\frac{nn_3\gamma^2t^2}{\mu r} \right) + \mathbb{P}(\norm{\Pomega\PT}\geq\sigma),
	\end{align*}
	where $\gamma=\frac{(1-\sigma^2)^2}{2\sigma^2}$. This proves the claim when $\mu r<\rho'_rnn_3\log(nn_3)^{-1}$ and $\rho_r'$ is sufficiently small.
\end{proof}

\section{Proofs of Some  Lemmas}\label{sec_prooflemmas}

\begin{lemma} \cite{tropp2012user}   \label{lembenmatrix}
	Consider a finite sequence $\{\Zm_k\}$ of independent, random $n_1\times n_2$ matrices that satisfy the assumption $\mathbb{E} \Zm_k=\0$ and $\norm{\Zm_k}\leq R$ almost surely. Let $\sigma^2 = \max\{\norm{\sum_k\mathbb{E}[\Zm_k\Zm_k^*]} , \max\{\norm{\sum_k\mathbb{E}[\Zm_k^*\Zm_k]} \}$. Then, for any $t\geq0$, we have
	\begin{align*}
		\mathbb{P}\left[ \normlarge{\sum_k\Zm_k} \geq t \right] \leq & (n_1+n_2) \exp\left( -\frac{t^2}{2\sigma^2+\frac{2}{3}Rt} \right) \\
		\leq &  (n_1+n_2) \exp\left( -\frac{3t^2}{8\sigma^2} \right), \ \text{for } t\leq\frac{\sigma^2}{R}.
	\end{align*}
	Or, for any $c>0$, we have
	\begin{align*}\label{eq_tropbound2}
		\normlarge{\sum_k\Zm_k} \geq 2\sqrt{c\sigma^2\log(n_1+n_2)} + cB\log(n_1+n_2),
	\end{align*}
	with probability at least $1-(n_1+n_2)^{1-c}$.
\end{lemma}

\subsection{Proof of Lemma \ref{lemspectrm}}

\begin{proof}
	The proof has three steps.
	
	\textit{Step 1: Approximation.}
	We first introduce some notations. Let $\f_i^*$ be the $i$-th row of $\F_{n_3}$,  and $\Mh=\begin{bmatrix}
	\Mh_1 \\ \Mh_2 \\ \vdots \\ \Mh_{n} 
	\end{bmatrix} \in\mathbb{R}^{nn_3\times n}$ be a matrix unfolded by $\M$, where $\Mh_i\in\mathbb{R}^{n_3\times n}$ is the $i$-th horizontal slice of $\M$, i.e., $[\Mh_i]_{kj} = \M_{ikj}$. Consider that $\Mbar = \mcode{fft}(\M,[\ ],3)$, we have
	\begin{equation*}\label{pfl351}
		\Mmbar_i = \begin{bmatrix}
			\f_i^*\Mh_1 \\  \f_i^*\Mh_2 \\ \vdots \\  \f_i^*\Mh_{n} 
		\end{bmatrix},
	\end{equation*}\label{pfl352}
	where $\Mmbar_i\in\mathbb{R}^{n\times n}$ is the $i$-th frontal slice of $\M$. Note that 
	\begin{equation}\label{pfl3544440}
	\norm{\M} = \norm{\Mmbar} = \max_{i=1,\cdots,n_3} \ \norm{\Mmbar_i}.
	\end{equation}
	Let $N$ be the $1/2$-net for $\mathbb{S}^{n-1}$ of size at most $5^n$ (see Lemma 5.2 in \cite{vershynin2010introduction}). Then Lemma 5.3 in  \cite{vershynin2010introduction} gives 
	\begin{equation}\label{prf13555}
	\norm{\Mmbar_i} \leq 2 \ \max_{\x\in N} \ \norm{\Mmbar_i\x}_2.
	\end{equation}
	So we consider to bound $\norm{\Mmbar_i\x}_2$. 
	
	\textit{Step 2: Concentration.} 
	We can express $\norm{\Mmbar_i\x}_2^2$ as a sum of independent random variables
	\begin{equation}\label{pfl353}
	\norm{\Mmbar_i\x}_2^2 = \sum_{j=1}^{n} (\f_i^*\Mh_j\x)^2 := \sum_{j=1}^{n} z_j^2,
	\end{equation}
	where $z_j = \langle {\Mh_j},{\f_i\x^*}\rangle$, $j=1,\cdots,n$, are independent sub-gaussian random variables with $\mathbb{E}z_j^2=\rho\norm{\f_i\x^*}_F^2=\rho n_3$. Using (\ref{defM0}), we have
	\begin{equation*} 
		|[\Mh_j]_{kl}| = \begin{cases}
			1, & \text{w.p.} \ \rho, \\
			0, & \text{w.p.} \ 1-\rho.
		\end{cases}
	\end{equation*}
	Thus, the sub-gaussian norm of $[\Mh_j]_{kl}$, denoted as $\norm{\cdot}_{\psi_2}$, is 
	\begin{align*}
		\norm{[\Mh_j]_{kl}}_{\psi_2} = & \sup_{p\geq 1} p^{-\frac{1}{2}}(\mathbb{E}[|[\Mh_j]_{kl}|^p])^{\frac{1}{p}} \\
		= & \sup_{p\geq 1} p^{-\frac{1}{2}}\rho^{\frac{1}{p}}.
	\end{align*}
	Define the function $\phi(x) =x^{-\frac{1}{2}}\rho^{\frac{1}{x}}$ on $[1,+\infty)$. The only stationary point occurs at $x^*=\log \rho^{-2}$. Thus, 
	\begin{align}
		\phi(x)\leq & \max(\phi(1),\phi(x^*)) \notag\\
		= &\max\left(\rho,(\log\rho^{-2})^{-\frac{1}{2}}\rho^{\frac{1}{\log \rho^{-2}}}\right) \notag\\
		:= & \psi(\rho). \label{abcefeqdd}
	\end{align}
	Therefore,  $\norm{[\Mh_j]_{kl}}_{\psi_2}\leq   \psi(\rho)$. Consider that $z_j$ is a sum of independent centered sub-gaussian random variables $[\Mh_j]_{kl}$'s, by using Lemma 5.9 in  \cite{vershynin2010introduction}, we have $\norm{z_j}^2_{\psi_2}\leq c_1(\psi(\rho))^2 n_3$, where $c_1$ is an absolute constant. Therefore, by Remark 5.18 and Lemma 5.14 in \cite{vershynin2010introduction}, $z_j^2-\rho n_3$ are independent centered sub-exponential random variables with $\norm{z_j^2-\rho n_3}_{\psi_1} \leq 2 \norm{z_j}^2_{\psi_1} \leq 4 \norm{z_j}^2_{\psi_2}\leq 4c_1 (\psi(\rho))^2 n_3$.
	
	Now, we use an exponential deviation inequality, Corollary 5.17 in \cite{vershynin2010introduction}, to control the sum  (\ref{pfl353}). We have
	\begin{align*}
		& \mathbb{P}\left(| \norm{\Mmbar_i\x}_2^2 -\rho nn_3 | \geq tn \right) \\
		= & \mathbb{P}\left(\left| \sum_{j=1}^{n} (z_j^2-\rho n_3) \right| \geq tn \right) \\
		\leq & 2\exp\left( -c_2n\min \left( \left(\frac{t}{4c_1 (\psi(\rho))^2 n_3}\right)^2, \frac{t}{4c_1(\psi(\rho))^2 n_3}\right) \right),
	\end{align*}
	where $c_2>0$. Let $t = c_3 (\psi(\rho))^2 n_3$ for some absolute constant $c_3$, we have   
	\begin{align*}
		& \mathbb{P}\left(| \norm{\Mmbar_i\x}_2^2 -\rho nn_3 | \geq c_3(\psi(\rho))^2 nn_3 \right) \\
		\leq & 2\exp\left( -c_2n\min\left( \left(\frac{c_3}{4c_1}\right)^2,\frac{c_3}{4c_1}\right) \right).
	\end{align*}

	\textit{Step 3： Union bound.}
	Taking the union bound over all $\x$ in the net $N$ of cardinality $|N|\leq 5^n$, we obtain
	\begin{align*}
		&\mathbb{P}\left(\left| \max_{\x \in N} \norm{\Mmbar_i\x}_2^2 -\rho nn_3 \right| \geq c_3(\psi(\rho))^2 nn_3 \right) \\
		\leq & 2\cdot 5^n \cdot\exp\left( -  c_2n\min\left(\left(\frac{c_3}{4c_1}\right)^2,\frac{c_3}{4c_1}\right) \right).
	\end{align*}
	Furthermore, taking the union bound over all $i=1,\cdots,n_3$, we have 
	\begin{align*}
		&\mathbb{P}\left(\max_i \ \left| \max_{\x \in N} \norm{\Mmbar_i\x}_2^2 -\rho nn_3 \right| \geq c_3(\psi(\rho))^2 nn_3 \right) \notag\\
		\leq &2\cdot 5^n\cdot n_3 \cdot \exp\left( -  c_2n\min\left(\left(\frac{c_3}{4c_1}\right)^2,\frac{c_3}{4c_1}\right) \right). \label{prfnlemmada000}
	\end{align*}
	This implies that, with high probability (when the constant $c_3$ is  large enough),
	\begin{equation}\label{profboundm2}
	\max_i \  \max_{\x \in N} \norm{\Mmbar_i\x}_2^2 \leq (\rho+c_3(\psi(\rho))^2)  nn_3.
	\end{equation}
	Let $\varphi(\rho) = 2\sqrt{\rho+c_3(\psi(\rho))^2}$ and it satisfies  $\lim\limits_{\rho\rightarrow0^+}\varphi(\rho)=0$ by using (\ref{abcefeqdd}). The proof is completed by further combining  (\ref{pfl3544440}), (\ref{prf13555}) and (\ref{profboundm2}).
\end{proof}

\subsection{Proof of Lemma \ref{lem_kem1}}
\begin{proof}
	For any tensor $\Z$, we can write
	\begin{align*}
		&(\rho^{-1}\PT\Pomega\PT-\PT)\Z \\
		= &\sum_{ijk}\left(\rho^{-1}{\delta_{ijk}}-1\right)\inproduct{\eijk }{\PT\Z}\PT(\eijk) \\
		:=&\sum_{ijk} \HH_{ijk}(\Z)
	\end{align*}
	where $\HH_{ijk}: \mathbb{R}^{\nss}\rightarrow\mathbb{R}^{\nss}$ is a self-adjoint random operator with $\mathbb{E}[\HH_{ijk}]=\0$. Define the matrix operator $\Hmbar_{ijk}: \mathbb{B}\rightarrow \mathbb{B}$, where $\mathbb{B}=\{\Bmbar: \B\in \mathbb{R}^{\nss} \}$ denotes the set consists of block diagonal matrices with the blocks as the frontal slices of $\Bbar$,  as
	\begin{align*}
		\Hmbar_{ijk}(\Zmbar) = & \left(\rho^{-1}{\delta_{ijk}}-1\right)\inproduct{\eijk }{\PT(\Z)}\bdiag(\overline{\PT(\eijk)}).
	\end{align*}
	By the above definitions, we have $\norm{\HH_{ijk}} = \norm{\Hmbar_{ijk}}$ and $\norm{\sum_{ijk}\HH_{ijk}} = \norm{\sum_{ijk}\Hmbar_{ijk}}$. Also $\Hmbar_{ijk}$ is  self-adjoint and $\mathbb{E}[\Hmbar_{ijk}]=0$. 
	To prove the result by the non-commutative Bernstein inequality, we need to bound $\norm{\Hmbar_{ijk}}$ and $\normlarge{\sum_{ijk}\mathbb{E}[\Hmbar^2_{ijk}]}$. First, we have
	\begin{align*}
		\norm{\Hmbar_{ijk}} = & \sup_{\norm{\Zmbar}_F=1} \norm{\Hmbar_{ijk}(\Zmbar)}_F \\
		\leq &  \sup_{\norm{\Zmbar}_F=1} \rho^{-1} \norm{\PT(\eijk)}_F \norm{\bdiag(\overline{\PT(\eijk)})}_F \norm{\Z}_F \\
		= & \sup_{\norm{\Zmbar}_F=1} \rho^{-1} \norm{\PT(\eijk)}_F^2 \norm{\Zmbar}_F \\
		\leq & \frac{2\mu r}{nn_3\rho},
	\end{align*}
	where the last inequality uses (\ref{proabouPT}). On the other hand, by direct computation, we have $\Hmbar_{ijk}^2(\Zmbar) = (\rho^{-1}\delta_{ijk}-1)^2\inproduct{\eijk}{\PT(\Z)}\inproduct{\eijk}{\PT(\eijk)}\bdiag(\overline{\PT(\eijk)})$. Note that $\mathbb{E}[(\rho^{-1}\delta_{ijk}-1)^2]\leq \rho^{-1}$. We have
	\begin{align*}
		& \normlarge{\sum_{ijk}\mathbb{E}[\Hmbar^2_{ijk}(\Zmbar)]}_F \\
		\leq & \rho^{-1}\normlarge{\sum_{ijk} \inproduct{\eijk}{\PT(\Z)}\inproduct{\eijk}{\PT(\eijk)}\bdiag(\overline{\PT(\eijk)})}_F \\
		\leq & \rho^{-1} \sqrt{n_3} \norm{\PT(\eijk)}_F^2\normlarge{\sum_{ijk} \inproduct{\eijk}{\PT(\Z)} }_F \\
		= & \rho^{-1}\sqrt{n_3} \norm{\PT(\eijk)}_F^2\norm{\PT(\Z)}_F\\
		\leq & \rho^{-1}\sqrt{n_3} \norm{\PT(\eijk)}_F^2\norm{\Z}_F\\
		= & \rho^{-1} \norm{\PT(\eijk)}_F^2\norm{\Zmbar}_F\\
		\leq & \frac{2\mu r}{nn_3\rho}\norm{\Zmbar}_F.	
	\end{align*}
	This implies $\normlarge{\sum_{ijk}\mathbb{E}[\Hmbar^2_{ijk}]}\leq \frac{2\mu r}{nn_3\rho}$.  Let $\epsilon\leq 1$. By Lemma \ref{lembenmatrix}, we have
	\begin{align*}
		& \mathbb{P} \left[\norm{\rho^{-1}\PT\Pomega\PT-\PT} > \epsilon \right] \\
		= & \mathbb{P}\left[ \normlarge{  \sum_{ijk} {\HH}_{ijk} } > \epsilon\right] \\
		= & \mathbb{P}\left[ \normlarge{  \sum_{ijk} {\Hmbar}_{ijk} } > \epsilon \right] \\
		\leq & 2nn_3 \exp\left( -\frac{3}{8} \cdot \frac{\epsilon^2}{2\mu r/(nn_3\rho) } \right) \\
		\leq & 2(nn_3)^{1-\frac{3}{16}C_0},
	\end{align*}
	where the last inequality uses  $\rho\geq C_0\epsilon^{-2}\mu r\log(nn_3)/(nn_3)$.
	Thus, $\norm{\rho^{-1}\PT\Pomega\PT-\PT} \leq \epsilon$ holds with high probability for some numerical constant $C_0$.
\end{proof}

\subsection{Proof of Corollary \ref{corollar31}}
\begin{proof}
	From Lemma \ref{lem_kem1}, we have 
	\begin{equation*}
		\norm{\PT-(1-\rho)^{-1}\PT\Pomegao\PT}\leq \epsilon,
	\end{equation*}
	provided that $1-\rho\geq C_0 \epsilon^{-2}(\mu r\log(nn_3))/n$. Note that $\I = \Pomega+\Pomegao$, we have 
	\begin{equation*}
		\norm{\PT-(1-\rho)^{-1}\PT\Pomegao\PT} = (1-\rho)^{-1} (\PT\Pomega\PT - \rho \PT).
	\end{equation*}
	Then, by the triangular inequality
	\begin{equation*}
		\norm{\PT\Pomega\PT} \leq \epsilon (1-\rho) + \rho \norm{\PT} = \rho + \epsilon(1-\rho).
	\end{equation*}
	The proof is completed by using $\norm{\Pomega\PT}^2 = \norm{\PT\Pomega\PT}$.
\end{proof}

\subsection{Proof of Lemma \ref{lem_keyinf}}
\begin{proof} For any tensor $\Z\in\Tm$, we write
	\begin{align*}
		\rho^{-1}\PT\Pomega(\Z) = \sum_{ijk}\rho^{-1}\delta_{ijk} z_{ijk}\PT(\eijk).
	\end{align*}
	The $(a,b,c)$-th entry of $\rho^{-1}\PT\Pomega(\Z)-\Z$ can be written as a sum of independent random variables, i.e.,
	\begin{align*}
		& \inproduct{\rho^{-1}\PT\Pomega(\Z)-\Z}{\eabc} \\
		= & \sum_{ijk} (\rho^{-1}\delta_{ijk}-1)z_{ijk} \inproduct{\PT(\eijk)}{\eabc} \\
		:=& \sum_{ijk} t_{ijk},
	\end{align*}
	where $t_{ijk}$'s are independent and $\mathbb{E}(t_{ijk})=0$. Now we bound $|t_{ijk}|$ and $|\sum_{ijk}\mathbb{E}[t_{ijk}^2]|$. First
	\begin{align*}
		&|t_{ijk}|\\	
		\leq & \rho^{-1} \norm{\Z}_\infty \norm{\PT(\eijk)}_F\norm{\PT(\eabc)}_F \\
		\leq & \frac{2\mu r}{nn_3\rho}\norm{\Z}_\infty.
	\end{align*}
	Second, we have
	\begin{align*}
		& \left|\sum_{ijk}\mathbb{E}[t_{ijk}^2]\right |\\
		\leq & \rho^{-1}\norm{\Z}_\infty^2\sum_{ijk}\inproduct{\PT(\eijk)}{\eabc}^2 \\
		=  & \rho^{-1}\norm{\Z}_\infty^2\sum_{ijk}\inproduct{\eijk}{\PT(\eabc)}^2 \\
		=  & \rho^{-1}\norm{\Z}_\infty^2\norm{\PT(\eabc)}_F^2 \\
		\leq & \frac{2\mu r}{nn_3\rho}\norm{\Z}_\infty^2.
	\end{align*}
	Let $\epsilon\leq 1$. By Lemma \ref{lembenmatrix}, we have
	\begin{align*}
		& \mathbb{P} \left[ |[\rho^{-1}\PT\Pomega(\Z)-\Z]_{abc}| > \epsilon\norm{\Z}_\infty\right] \\
		= & \mathbb{P}\left[ \left| \sum_{ijk} {t}_{ijk}\right|> \epsilon\norm{\Z}_\infty\right] \\
		\leq & 2 \exp\left( -\frac{3}{8} \cdot \frac{\epsilon^2\norm{\Z}^2_\infty}{2\mu r\norm{\Z}^2_\infty/(nn_3\rho) } \right) \\
		\leq & 2(nn_3)^{-\frac{3}{16}C_0},
	\end{align*}
	where the last inequality uses  $\rho\geq C_0\epsilon^{-2}\mu r\log(nn_3)/(nn_3)$.
	Thus, $\norm{\rho^{-1}\PT\Pomega(\Z)-\Z}_\infty\leq \epsilon\norm{\Z}_\infty$ holds with high probability for some numerical constant $C_0$. 	
\end{proof}

\subsection{Proof of Lemma \ref{lempre3}}

\begin{proof}
	Denote the tensor $\HH_{ijk} = \left(1-\rho^{-1}\delta_{ijk}\right)z_{ijk}\eijk$. Then we have
	\begin{equation*}
		(\I-\rho^{-1}\Pomega)\Z = \sum_{ijk}\HH_{ijk}.
	\end{equation*}
	Note that $\delta_{ijk}$'s are independent random scalars. Thus, $\HH_{ijk}$'s are independent random tensors and $\Hmbar_{ijk}$'s are independent random matrices. 
	Observe that $\mathbb{E}[{\Hmbar}_{ijk}] = \0$ and $\norm{{\Hmbar}_{ijk}} \leq {\rho}^{-1} \norm{\Z}_\infty$. We have
	\begin{align*}
		&\normlarge{\sum_{ijk} \mathbb{E} [ {\Hmbar}^*_{ijk} {\Hmbar}_{ijk} ]  } \\
		= & \normlarge{\sum_{ijk} \mathbb{E} [ {\HH}^*_{ijk} *{\HH}_{ijk} ]  } \\
		= & \normlarge{\sum_{ijk}   \mathbb{E}[ (1-\rho^{-1}{\delta_{ijk}})^2 ] z_{ijk}^2 ({\ej}*\ej^*) } \\
		= & \normlarge{\frac{1-\rho}{\rho} \sum_{ijk} z_{ijk}^2 ({\ej}*\ej^*) } \\
		\leq & { \frac{nn_3}{\rho} }\norm{\Z}_\infty^2.
	\end{align*}
	A similar calculation yields $\normlarge{\sum_{ijk} \mathbb{E} [{\Hmbar}_{ijk}^* {\Hmbar}_{ijk} ]  }\leq { \rho^{-1}nn_3 }\norm{\Z}_\infty^2$. Let $t = \sqrt{C_0{nn_3\log(nn_3)}/{\rho}}\norm{\Z}_\infty$. When $\rho\geq C_0\log(nn_3)/(nn_3)$, we apply Lemma \ref{lembenmatrix} and obtain
	\begin{align*}
		& \mathbb{P}\left[ \norm{(\I-\rho^{-1}\Pomega)\Z } > t \right]  \\
		= & \mathbb{P}\left[ \normlarge{  \sum_{ijk} {\HH}_{ijk} } > t \right] \\
		= & \mathbb{P}\left[ \normlarge{  \sum_{ijk} {\Hmbar}_{ijk} } > t \right] \\
		\leq & 2nn_3 \exp\left( -\frac{3}{8} \cdot \frac{C_0nn_3\log(nn_3)\norm{\Z}_\infty^2/\rho}{nn_3\norm{\Z}_\infty^2/\rho } \right) \\
		\leq & 2(nn_3)^{1-\frac{3}{8}C_0}.
	\end{align*}
	Thus, $ \norm{(\I-\rho^{-1}\Pomega)\Z } > t$ holds with high probability for some numerical constant $C_0$.
\end{proof}

\end{document}


\title{Supplementary Material of Tensor Robust Principal Component Analysis with A New Tensor Nuclear Norm}

\author{Canyi~Lu,~Jiashi~Feng,~Yudong~Chen,~Wei Liu,~Member,~IEEE,~Zhouchen Lin,~\IEEEmembership{Fellow,~IEEE},
	
	~and~Shuicheng Yan,~\IEEEmembership{Fellow,~IEEE}
	
}

\maketitle

This document gives the detailed proofs of Theorem 3.1, Theorem 3.2, and the main result in Theorem 4.1. Section \ref{sec_notations} first gives some notations and properties which will be used in the proofs. Section \ref{profthm3132} gives the proofs of Theorem 3.1 and 3.2 in our paper. Section \ref{sec_dual} provides a way for the construction of the solution to the TRPCA model, and Section \ref{sec_proofofdual} proves that the constructed solution is optimal to the TRPCA problem. Section \ref{sec_prooflemmas} gives the proofs of some lemmas which are used in Section \ref{sec_proofofdual}.


\section{Preliminaries}
\label{sec_notations}


Beyond the notations introduced in the paper, we need some other notations used in the proofs.
At the following, we define $\eijk=\ei*\ek*\ej^*$. Then we have $\X_{ijk} = \inproduct{\X}{\eijk}$. We define the projection $$\Pomega(\Z)=\sum_{ijk}\delta_{ijk}z_{ijk}\eijk,$$
where $\delta_{ijk}=1_{(i,j,k)\in\Omegat}$, where $1_{(\cdot)}$ is the indicator function. Also $\Omegat^c$ denotes the 
complement of $\Omegat$ and $\Pomegao$ is the projection onto $\Omegat^c$. Denote $\Tm$ by the set 
\begin{align}
\Tm = \{ \U*\Y^* + \W*\V^*, \ \Y, \W\in\mathbb{R}^{n\times r\times n_3} \},
\end{align}
and by $\Tm^\bot$  its orthogonal complement. Then the projections onto $\Tm$ and $\Tm^\bot$ are respectively
\begin{align*}
\PT(\Z) = \U*\U^**\Z + \Z *\V*\V^* -  \U*\U^**\Z *\V*\V^*, 
\end{align*}
 \begin{align*}
 \PTo(\Z) =& \Z-\PT(\Z)\\
 =&(\I_{n_1} - \U*\U^*)*\Z*(\I_{n_2}-\V*\V^*),
 \end{align*}
 where $\I_n$ denotes the $n\times n\times n_3$ identity tensor. Note that $\PT$ is self-adjoint. So we have 
 \begin{align*}
 &\norm{\PT(\eijk)}_F^2 \\
 = & \inproduct{\PT(\eijk)}{\eijk} \\
 = & \inproduct{\U*\U^**\eijk + \eijk *\V*\V^*}{\eijk} \\
   & -\inproduct{ \U*\U^**\eijk *\V*\V^*}{\eijk}
 \end{align*}
Note that 
  \begin{align*}
  &\inproduct{\U*\U^**\eijk}{\eijk}\\
  =&\inproduct{\U*\U^**\ei*\ek*\ej^*}{\ei*\ek*\ej^*}\\
  =&\inproduct{\U^**\ei}{\U^**\ei*(\ek*\ej^**\ej*\ek^*)}\\
  =&\inproduct{\U^**\ei}{\U^**\ei}\\
  =&\norm{\U^**\ei}_F^2,
  \end{align*}
  where we use the fact that $\ek*\ej^**\ej*\ek^*=\I_1$, which is the $1\times 1\times n_3$ identity tensor. Therefore, it is easy to see that 
   \begin{align}
   &\norm{\PT(\eijk)}_F^2 \notag \\
   = & \norm{\U^**\ei}_F^2 + \norm{\V^**\ej}_F^2 - \norm{\U^**\ei*\ek*\ej^**\V}_F^2, \notag\\
   \leq & \norm{\U^**\ei}_F^2 + \norm{\V^**\ej}_F^2 \notag \\
   \leq & \frac{\mu r(n_1+n_2)}{n_1n_2n_3}  \label{proabouPTn1n2}\\
   = & \frac{2\mu r}{nn_3}, \ \text{when } n_1=n_2=n.\label{proabouPT}
   \end{align}
where (\ref{proabouPTn1n2}) uses the following tensor incoherence conditions
	\begin{align}
	\max_{i=1,\cdots,n_1} \norm{\U^**\mathring{\mathfrak{e}}_i}_F\leq\sqrt{\frac{\mu r}{n_1n_3}}, \label{tic1}\\
	\max_{j=1,\cdots,n_2} \norm{\V^**\mathring{\mathfrak{e}}_j}_F\leq\sqrt{\frac{\mu r}{n_2n_3}},\label{tic2}
	\end{align}
	and
	\begin{equation} \label{tic3}
	\norm{\U*\V^*}_\infty\leq \sqrt{\frac{\mu r}{n_1n_2n_3^2}},
	\end{equation}	
which are assumed to be satisfied in Theorem 4.1 in our manuscript.


\section{Proofs of Theorem 3.1 and Theorem 3.2}\label{profthm3132}
 
\subsection{Proof of Theorem 3.1}

\begin{proof}
	To complete the proof, we need the conjugate function concept. The conjugate $\phi^*$ of a function $\phi: C\rightarrow \mathbb{R}$, where $C \subset \mathbb{R}^n$, is defined as
	\begin{equation*}
	\phi^*(\y) = \sup \{\inproduct{\y}{\x} - \phi(\x) | \x \in C \}.
	\end{equation*}
	Note that the conjugate of the conjugate, $\phi^{**}$, is the convex envelope of the function $\phi$. See Theorem 1.3.5 in \cite{hiriart1991convex,fazel2002matrix}.  The proofs has two steps which compute $\phi^*$ and $\phi^{**}$, respectively.
	
	\textbf{Step 1.} \textit{Computing $\phi^*$.} For any $\A\in\Rn$, the conjugate function of the tensor average rank
	\begin{equation*}
		\phi(\A) = {\ranka(\A)}=\frac{1}{n_3}\rank(\bcirc(\A))=\frac{1}{n_3}\rank(\Ambar),
	\end{equation*}
	on the set $S = \{\A\in\Rn |\norm{\A}\leq 1\}$ is
	\begin{align*}
		\phi^*(\B) = & \sup_{\norm{\A}\leq 1} \left(\inproduct{\B}{\A} - {\ranka(\A)}\right) \\
				   = & \sup_{\norm{\A}\leq 1} \frac{1}{n_3}(\inproduct{\Bmbar}{\Ambar} - \rank(\Ambar)).
	\end{align*}
	Here $\Ambar, \Bmbar\in\mathbb{C}^{n_1n_3\times n_2n_3}$. Let $q = \min\{n_1n_3, n_2n_3\}$. By von Neumann's trace theorem, 
	\begin{align}\label{vonneutrace}
		\inproduct{\Bmbar}{\Ambar} \leq \sum_{i = 1}^{q} \sigma_i(\Bmbar) \sigma_i(\Ambar),
	\end{align} 
	where $\sigma_i(\Ambar)$ denotes the $i$-th largest singular value of $\Ambar$. Let $\Ambar = \Umbar_1\Smbar_1\Vm^*_1$ and $\Bmbar = \Umbar_2\Smbar_2\Vmbar^*_2$ be the SVD of $\Ambar$ and $\Bmbar$, respectively. Note that the equality (\ref{vonneutrace}) holds when
	\begin{align}\label{conduvab}
		\Umbar_1 = \Umbar_2  \ \text{and} \ \Vmbar_1 = \Vmbar_2.
	\end{align}
	So we can pick $\Umbar_1$ and $\Vmbar_1$ such that (\ref{conduvab}) holds to maximize $\inproduct{\Bmbar}{\Ambar}$. Note that the corresponding $\U$ and $\V$ of $\Umbar_1$ and $\Vmbar_1$  respectively are real tensors and so is $\A$ in this case. Thus, we have
	\begin{align*}
	\phi^*(\B) = & \sup_{\norm{\A}\leq 1} \frac{1}{n_3}\left(\sum_{i = 1}^{q} \sigma_i(\Bmbar) \sigma_i(\Ambar) - \rank(\Ambar)\right).
	\end{align*}
	
	If $\A = 0$, then $\Ambar = 0$, and thus we have $\phi^*(\B) = 0$ for all $\B$. If $\rank(\Ambar) = r$, $1\leq r\leq q$, then $\phi^*(\B) = \frac{1}{n_3}\left(\sum_{i = 1}^{r} \sigma_i(\Bmbar) - r\right)$. Hence $\phi^*(\B)$ can be expressed as
	\begin{align*}
		& n_3 \cdot \phi^*(\B) \\
		= & {\max \left\{0, \sigma_1(\Bmbar) -1, \cdots, \sum_{i = 1}^{r}\sigma_i(\Bmbar) -r, \cdots,  \sum_{i = 1}^{q} \sigma_i(\Bmbar) - q\right\}}.
	\end{align*}
 	The largest term in this set is the one that sums all positive $(\sigma_i(\Bmbar) -1)$ terms. Thus, we have
 	\begin{align*}
 		&\phi^*(\B) \\
 		= & \begin{cases}
 		0, \qquad\qquad \qquad \quad \quad \quad \quad       \norm{\Bmbar}\leq 1, \\
 		 \frac{1}{n_3}\left(\sum_{i = 1}^{r}\sigma_i(\Bmbar) -r\right), \quad \sigma_r(\Bmbar) > 1 \text{ and } \sigma_{r+1}(\Bmbar) \leq 1 
 		\end{cases} \\
 		= &  \frac{1}{n_3} \sum_{i = 1}^{q}(\sigma_i(\Bmbar) -1)_+. 	\end{align*}
 	Note that above $\norm{\Bmbar}\leq 1$ is equivalent to $\norm{\B}\leq 1$.
 	
 	\textbf{Step 2.} \textit{Computing $\phi^{**}$.} Now we compute the conjugate of $\phi^*$, defined as
 	\begin{align*}
 		\phi^{**}(\C) = & \sup_{\B} ( \inproduct{\C}{\B} - \phi^*(\B) ) \\
 		= & \sup_{\B} \left(\frac{1}{n_3} \inproduct{\Cmbar}{\Bmbar}-\phi^*(\B)\right),
 	\end{align*}
 	for all $\C \in S$. As before, we can choose $\B$ such that 
 	\begin{align*}
	 	\phi^{**}(\C) = & \sup_{\B} \left(\frac{1}{n_3}\sum_{i = 1}^{q} \sigma_i(\Cmbar) \sigma_i(\Bmbar) - \phi^*(\B)\right).
 	\end{align*}
 	At the following, we consider two cases, $\norm{\C} > 1$ and $\norm{\C} \leq 1$.
 	
 	If $\norm{\C} > 1$, then $\sigma_1(\Cmbar) = \norm{\Cmbar} = \norm{\C} > 1$.  We can choose $\sigma_1(\Bmbar)$ large enough so that $\phi^{**}(\C)\rightarrow\infty$. To see this, note that  in
 	\begin{align*}
 		\phi^{**}(\C) = & \sup_{\B}\frac{1}{n_3} \left(\sum_{i = 1}^{q} \sigma_i(\Cmbar) \sigma_i(\Bmbar) - \left( \sum_{i=1}^{r} \sigma_i(\Bmbar)  - r \right)\right),
 	\end{align*}
 	the coefficient of $\sigma_1(\Bmbar)$ is $\frac{1}{n_3}(\sigma_1(\Cmbar)-1)$ which is positive.
 	
 	If $\norm{\C} \leq 1$, then $\sigma_1(\Cmbar) = \norm{\Cmbar} = \norm{\C} \leq 1$. If $\norm{\B} = \norm{\Bmbar} \leq 1$, then $\phi^*(\B)=0$ and the supremum is achieved for $\sigma_i(\Bmbar)=1$, $i=1,\cdots,q$, yielding 
 	\begin{align*}
 		\phi^{**}(\C) = \frac{1}{n_3} \sum_{i=1}^{q} \sigma_i(\Cmbar ) = \frac{1}{n_3} \norm{\Cmbar}_* = \norm{\C}_*.
 	\end{align*}
 	If $\norm{\C}>1$, we show that the argument of sup is is always smaller than $\norm{\C}_*$. By adding and subtracting the term $\frac{1}{n_3}\sum_i^q\sigma_i(\Cmbar)$ and rearranging the terms, we have
 	\begin{align*}
 		&\frac{1}{n_3} \left(\sum_{i = 1}^{q} \sigma_i(\Cmbar) \sigma_i(\Bmbar) -  \sum_{i=1}^{r}  \left(\sigma_i(\Bmbar)  - 1 \right)\right) \\
 		=& \frac{1}{n_3} \left(\sum_{i = 1}^{q} \sigma_i(\Cmbar) \sigma_i(\Bmbar) -  \sum_{i=1}^{r}  \left(\sigma_i(\Bmbar)  - 1 \right)\right) \\
 		 & - \frac{1}{n_3}\sum_{i=1}^q\sigma_i(\Cmbar) + \frac{1}{n_3}\sum_{i=1}^q\sigma_i(\Cmbar) \\
 		=&\frac{1}{n_3} \sum_{i=1}^r (\sigma_i(\Bmbar)-1)(\sigma_i(\Cmbar)-1)\\
 		&+\frac{1}{n_3}\sum_{i=r+1}^q (\sigma_i(\Bmbar)-1)\sigma_i(\Cmbar) +\frac{1}{n_3} \sum_{i=1}^{q} \sigma_i(\Cmbar) \\
 		< & \frac{1}{n_3}\sum_{i=1}^{q} \sigma_i(\Cmbar)\\
 		= & \norm{\C}_*.
 	\end{align*}
 	
 	In a summary, we have shown that
 	\begin{align*}
	 	\phi^{**}(\C) = \norm{\C}_*, 
 	\end{align*}
 	over the set $S=\{\C | \norm{\C} \leq 1\}$. Thus, $\norm{\C}_*$ is the convex envelope of the tensor average rank ${\ranka(\C)}$ over $S$.
\end{proof}

\subsection{Proof of Theorem 3.2}

\begin{proof} Let $\G\in \partial \norm{\A}_*$. It is equivalent to the following statements   \cite{watson1992characterization}
	\begin{align}
	\norm{\A}_* &= \inproduct{\G}{\A}, \label{subp1}\\
	\norm{\G} &\leq 1. \label{subp2}
	\end{align}
	So, to   complete the proof, we only need to show that $\G = \U *\V^*+\W$, where $\U^**\W=\0$, $\W*\V=\0$ and $\norm{\W}\leq1$, satisfies (\ref{subp1}) and (\ref{subp2}). 
	First, we have
	\begin{align*}
	\inproduct{\G}{\A} = & \inproduct{\U*\V^*+\W}{\U*\SS*\V^*} \\
	= & \inproduct{\I}{\SS}+0 \\
	= & \norm{\A}_*.
	\end{align*}
	Also, (\ref{subp2}) is obvious when considering the property of $\W$.  The proof is completed. 
\end{proof}

\section{Dual Certification}
\label{sec_dual}

In this section, we first introduce  conditions for  $(\LL_0,\SS_0)$  to be the unique solution to TRPCA in subsection \ref{subdualcertif}.  Then we construct a dual certificate in subsection \ref{subsecducergs} which satisfies the conditions in subsection \ref{subdualcertif}, and thus our main result in Theorem 4.1 in our paper are proved.

\subsection{Dual Certificates}\label{subdualcertif}

\begin{lemma}\label{lem_dual2}
	Assume that $\norm{\Pomega\PT}\leq \frac{1}{2}$ and $\lambda<\frac{1}{\sqrt{n_3}}$. Then $(\LL_0,\SS_0)$ is the unique solution to the TRPCA problem if there is a pair $(\W,\F)$ obeying
	\begin{equation*}
	(\U*\V^*+\W)=\lambda(\sgn{\SS_0}+\F+\Pomega\D),
	\end{equation*}
	with $\PT\W=\0$, $\norm{\W}\leq\frac{1}{2}$, $\Pomega\F=\0$ and $\norm{\F}_{\infty}\leq\frac{1}{2}$, and $\norm{\Pomega\D}_F\leq\frac{1}{4}$.
\end{lemma}

\begin{proof}
	For any $\HH\neq\0$, $(\LL_0+\HH,\SS_0-\HH)$ is also a feasible solution. We show that its objective is larger than that at $(\LL_0,\SS_0)$, hence proving that $(\LL_0,\SS_0)$ is the unique solution. To do this, let $\U*\V^*+\W_0$ be an arbitrary subgradient of the tensor nuclear norm at $\LL_0$, and $\sgn{\SS_0}+\F_0$ be an arbitrary subgradient of the $\ell_1$-norm at $\SS_0$. Then we have
	\begin{align*}
	& \norm{\LL_0+\HH}_*+\lambda\norm{\SS_0-\HH}_1 \\
	\geq& \norm{\LL_0}_*+\lambda\norm{\SS_0}_1+\langle\U*\V^*+\W_0,\HH\rangle \\
	&-\lambda\inproduct{\sgn{\SS_0}+\F_0}{\HH}.
	\end{align*}
	Now pick $\W_0$ such that $\inproduct{\W_0}{\HH}=\norm{\PTo\HH}_*$ and $\inproduct{\F_0}{\HH}=-\norm{\Pomegao\HH}$. We have
	\begin{align*}
	& \norm{\LL_0+\HH}_*+\lambda\norm{\SS_0-\HH}_1 \\
	\geq& \norm{\LL_0}_*+\lambda\norm{\SS_0}_1+\norm{\PTo\HH}_*+\lambda\norm{\Pomegao\HH}_1\\
	&+\inproduct{\U*\V^*-\lambda\sgn{\SS_0}}{\HH}.
	\end{align*}
	By assumption
	\begin{align*}
	&\abs{\inproduct{\U*\V^*-\lambda\sgn{\SS_0}}{\HH}}\\
	\leq& \abs{\inproduct{\W}{\HH}}+\lambda\abs{\inproduct{\F}{\HH}}+\lambda\abs{\inproduct{\Pomega\D}{\HH}}\\
	\leq&\beta(\norm{\PTo\HH}_*+\lambda\norm{\Pomegao\HH}_1)+\frac{\lambda}{4}\norm{\Pomega\HH}_F,
	\end{align*}
	where $\beta=\max(\norm{\W},\norm{\F}_\infty)<\frac{1}{2}$. Thus	
	\begin{align*}
	& \norm{\LL_0+\HH}_*+\lambda\norm{\SS_0-\HH}_1 \\
	\geq& \norm{\LL_0}_*+\lambda\norm{\SS_0}_1+\frac{1}{2}(\norm{\PTo\HH}_*+\lambda\norm{\Pomegao\HH}_1) \\
	& -\frac{\lambda}{4}\norm{\Pomega\HH}_F.
	\end{align*}
	On the other hand,
	\begin{align*}
	\norm{\Pomega\HH}_F\leq & \norm{\Pomega\PT\HH}_F+\norm{\Pomega\PTo\HH}_F \\
	\leq & \frac{1}{2}\norm{\HH}_F+\norm{\PTo\HH}_F \\
	\leq & \frac{1}{2}\norm{\Pomega\HH}_F+\frac{1}{2}\norm{\Pomegao\HH}_F+\norm{\PTo\HH}_F.
	\end{align*}
	Thus
	\begin{align*}
	\norm{\Pomega\HH}_F\leq & \norm{\Pomegao\HH}_F+2\norm{\PTo\HH}_F \\
	\leq & \norm{\Pomegao\HH}_1+{2}{\sqrt{n_3}}\norm{\PTo\HH}_*.
	\end{align*}
	In conclusion, 
	\begin{align*}
	& \norm{\LL_0+\HH}_*+\lambda\norm{\SS_0-\HH}_1 \\
	\geq& \norm{\LL_0}_*+\lambda\norm{\SS_0}_1+\frac{1}{2}\left(1-{\lambda}{\sqrt{n_3}}\right)\norm{\PTo\HH}_*\\
	&+\frac{\lambda}{4}\norm{\Pomegao\HH}_1,
	\end{align*}
	where the last two terms are strictly positive when $\HH\neq \0$. Thus, the proof is completed.
\end{proof}

Lemma \ref{lem_dual2} implies that it is suffices to produce a dual certificate $\W$ obeying
\begin{equation}\label{conddualcertf}
\begin{cases}
\W \in \Tm^\bot, \\
\norm{\W} < \frac{1}{2},\\
\norm{\Pomega(\U*\V^*+\W-\lambda\sgn{\SS_0})}_F\leq \frac{\lambda}{4},\\
\norm{\Pomegao(\U*\V^*+\W)}_\infty < \frac{\lambda}{2}.
\end{cases}
\end{equation}

\subsection{Dual Certification via the Golfing Scheme}\label{subsecducergs}
In this subsection, we show how to construct a dual certificate obeying (\ref{conddualcertf}). Before we introduce our construction, our model assumes that $\Omegat\sim \Ber(\rho)$, or equivalently that $\Omegat^c\sim\Ber(1-\rho)$. Now the distribution of $\Omegat^c$ is the same as that of
$\Omegat^c = \Omegat_1 \cup \Omegat_2 \cup \cdots \cup \Omegat_{j_0}$, where each $\Omegat_j$  follows the Bernoulli model with parameter $q$, which satisfies
\begin{equation*}
\mathbb{P}((i,j,k)\in\Omegat) = \mathbb{P}(\text{Bin}(j_0,q)=0) = (1-q)^{j_0},
\end{equation*}
so that the two models are the same if 
$\rho = (1-q)^{j_0}$.
Note that  because of overlaps between the $\Omegat_j$'s, $q\geq (1-\rho)/j_0$.

Now, we construct a dual certificate
\begin{equation}
\W = \WL+\WS,
\end{equation}
where each component is as follows:
\begin{enumerate}
	\item  Construction of $\WL$ via the Golfing Scheme. Let $j_0= 2\log(nn_3)$ and $\Omegat_j$, $j=1,\cdots,j_0$, be defined as previously described so that $\Omegat^c = \cup_{1\leq j \leq j_0}\Omegat_j $. Then define
	\begin{equation}\label{eq_wl}
	\WL = \PTo\Y_{j_0},
	\end{equation}
	where 
	\begin{equation*}
	\Y_j = \Y_{j-1} + q^{-1}\PP_{\Omegat_j}\PT(\U*\V^*-\Y_{j-1}), \ \Y_0=\0.
	\end{equation*}
	\item Construction of $\WS$ via the Method of Least Squares. Assume that $\norm{\Pomega\PT}<1/2$. 
	Then, $\norm{\Pomega\PT\Pomega}<1/4$, and thus, the operator $\Pomega-\Pomega\PT\Pomega$  mapping $\Omegat$ onto itself is invertible; we denote its inverse by $(\Pomega-\Pomega\PT\Pomega)^{-1}$. We then set
	\begin{equation}\label{eq_ws}
	\WS = \lambda \PTo (\Pomega-\Pomega\PT\Pomega)^{-1} \sgn{\SS_0}.
	\end{equation}
	This is equivalent to 
	\begin{equation*}
	\WS = \lambda \PTo \sum_{k\geq 0}   (\Pomega\PT\Pomega)^{k} \sgn{\SS_0}.
	\end{equation*}
\end{enumerate}
Since both $\WL$ and $\WS$ belong to $\Tm^{\bot}$ and $\Pomega\WS  = \lambda  \Pomega (\I-\PT) (\Pomega-\Pomega\PT\Pomega)^{-1} \sgn{\SS_0}=\lambda\sgn{\SS_0}$, we will establish that $\WL+\WS$ is a valid dual certificate if it obeys 
\begin{equation}
\begin{cases}
\norm{\WL+\WS} < \frac{1}{2}, \\
\norm{\Pomega(\U*\V^*+\WL)}_F\leq \frac{\lambda}{4}, \\
\norm{ \Pomegao(\U*\V^*+\WL+\WS}_\infty < \frac{\lambda}{2}.
\end{cases}
\end{equation}
This can be achieved by using the following two key lemmas:
\begin{lemma}\label{eq_keylem1}
	Assume that $\Omegat\sim\Ber(\rho)$ with parameter $\rho\leq \rho_s$ for some $\rho_s>0$. Set $j_0=2\lceil\log (nn_3)\rceil$ (use $\log(\none n_3)$ for the tensors of rectangular frontal slice). Then,  the tensor $\WL$ obeys
	\begin{enumerate}[(a)]
		\item $\norm{\WL} < \frac{1}{4}$,
		\item $\norm{\Pomega(\U*\V^*+\WL)}_F<\frac{\lambda}{4}$,
		\item $\norm{\Pomegao(\U*\V^*+\WL) }_\infty < \frac{\lambda}{4}$.
	\end{enumerate}
\end{lemma}
\begin{lemma}\label{eq_keylem2}
	Assume that $\SS_0$ is supported on a set $\Omegat$ sampled as in Lemma \ref{eq_keylem1}, and that the signs of $\SS_0$ are independent and identically distributed symmetric (and independent of $\Omegat$). Then, the tensor $\WS$ (\ref{eq_ws}) obeys
	\begin{enumerate}[(a)]
		\item $\norm{\WS} < \frac{1}{4}$,
		\item $\norm{\Pomegao\WS }_\infty < \frac{\lambda}{4}$.
	\end{enumerate}
\end{lemma}

So the left task is to prove Lemma \ref{eq_keylem1} and Lemma \ref{eq_keylem2}, which are given in Section \ref{sec_proofofdual}.

\section{Proofs of Dual Certification}
\label{sec_proofofdual}

This section gives the proofs of Lemma \ref{eq_keylem1} and Lemma \ref{eq_keylem2}. To do this, we first introduce some lemmas with their proofs given in Section \ref{sec_prooflemmas}. 

\begin{lemma}\label{lemspectrm}
	For the Bernoulli sign variable $\M\in\mathbb{R}^{n\times n\times n_3}$ defined as
	\begin{equation}\label{defM0}
	\M_{ijk} = \begin{cases}
	1, & \text{w.p.} \ \rho/2, \\
	0, & \text{w.p.} \ 1-\rho, \\
	-1, & \text{w.p.} \ \rho/2,
	\end{cases}
	\end{equation}
	where $\rho>0$, there exists a function $\varphi(\rho)$  satisfying  $\lim\limits_{\rho\rightarrow0^+}\varphi(\rho)=0$, such that the following statement holds with with large probability,
	\begin{equation*}\label{boundspecm}
	\norm{\M} \leq \varphi(\rho)\sqrt{ nn_3}.
	\end{equation*}
\end{lemma}

\begin{lemma}\label{lem_kem1}
	Suppose $\Omegat\sim\Ber(\rho)$. Then with high probability,
	\begin{equation*}
	\norm{\PT-\rho^{-1}\PT\Pomega\PT}\leq\epsilon,
	\end{equation*}
	provided that $\rho\geq C_0\epsilon^{-2}(\mu r\log (nn_3))/(nn_3)$ for some numerical constant $C_0>0$. For the tensor of  rectangular frontal slice, we need $\rho\geq C_0\epsilon^{-2}(\mu r\log (\none n_3))/(\ntwo n_3)$.
\end{lemma}
\begin{cor}\label{corollar31}
	Assume that $\Omegat\sim\Ber(\rho)$, then $\norm{\Pomega\PT}^2\leq \rho+\epsilon$, provided that $1-\rho\geq C\epsilon^{-2}(\mu r\log(nn_3))/(nn_3)$, where $C$ is as in Lemma \ref{lem_kem1}. For the tensor with frontal slice, the modification is as in Lemma  \ref{lem_kem1}.
\end{cor}
Note that this corollary shows that $\norm{\Pomega\PT}\leq 1/2$, provided $|\Omegat|$ is not too large. 
\begin{lemma}\label{lem_keyinf}
	Suppose that $\Z\in\Tm$ is a fixed tensor, and $\Omegat\sim\Ber(\rho)$. Then, with high probability,
	\begin{equation}
	\norm{\Z-\rho^{-1}\PT\Pomega\Z}_{\infty} \leq \epsilon \norm{\Z}_\infty,
	\end{equation}
	provided that $\rho\geq C_0 \epsilon^{-2}(\mu r \log(nn_3))/(nn_3)$ (for the tensor of  rectangular frontal slice, $\rho\geq C_0\epsilon^{-2}(\mu r\log (\none n_3))/(\ntwo n_3)$) for some numerical constant $C_0>0$.
\end{lemma}

\begin{lemma}\label{lempre3}
	Suppose $\Z$ is fixed, and $\Omegat\sim\Ber(\rho)$. Then, with high probability,
	\begin{equation}
	\norm{(\I-\rho^{-1}\Pomega)\Z} \leq \sqrt{\frac{C_0nn_3\log(nn_3)}{\rho}} \norm{\Z}_\infty,
	\end{equation}
	for some numerical constant $C_0>0$ provided that $\rho\geq C_0\log(nn_3)/(nn_3)$ (or $\rho\geq C_0\log(\none n_3)/(\ntwo n_3)$ for the tensors with rectangular frontal slice).
\end{lemma} 

\subsection{Proof of Lemma \ref{eq_keylem1}}
\begin{proof}
	We first introduce some notations. Set $\Z_j =   \U*\V^* - \PT\Y_j$ obeying
	\begin{equation*}
	\Z_j = (\PT-q^{-1}\PT\PP_{\Omegat_j}\PT)\Z_{j-1}.
	\end{equation*}
	So $\Z_j\in\T$ for all $j\geq 0$. Also, note that when
	\begin{equation}\label{eq_q}
	q\geq C_0\epsilon^{-2}\frac{\mu r\log(nn_3)}{nn_3},
	\end{equation}
	or for the tensors with rectangular frontal slices $q\geq C_0\epsilon^{-2}\frac{\mu r\log(\none n_3)}{\ntwo n_3}$,
	we have
	\begin{equation}\label{eq_Zrel1}
	\norm{\Z_j}_\infty \leq \epsilon \norm{\Z_{j-1}}_\infty\leq \epsilon^j \norm{\U*\V^*}_\infty,
	\end{equation}
	by Lemma \ref{lem_keyinf} and 
	\begin{equation}\label{eq_Zrel2}
	\norm{\Z_j}_F \leq \epsilon \norm{\Z_{j-1}}_F \leq \epsilon^j \norm{\U*\V^*}_F \leq   \epsilon^j \sqrt{r}.
	\end{equation}
	We assume $\epsilon\leq e^{-1}$.
	
\noindent\textbf{1. Proof of (a).} Note that $\Y_{j_0} = \sum_j q^{-1}\PP_{\Omegat_j}\Z_{j-1}$. We have
	\begin{align*}
		\norm{\WL} = &\norm{\PTo\Y_{j_0}} \leq \sum_j \norm{q^{-1} \PTo\PP_{\Omegat_j} \Z_{j-1} } \\
		\leq & \sum_j \norm{\PTo(q^{-1}\PP_{\Omegat_j}\Z_{j-1}-\Z_{j-1})} \\
		\leq & \sum_j  \norm{q^{-1}\PP_{\Omegat_j}\Z_{j-1}-\Z_{j-1}} \\
		\leq & C_0'\sqrt{\frac{nn_3\log(nn_3)}{q}} \sum_j \norm{\Z_{j-1}}_\infty \\
		\leq & C_0'\sqrt{\frac{nn_3\log(nn_3)}{q}} \sum_j \epsilon^{j-1}\norm{\U*\V^*}_\infty \\
		\leq & C_0'(1-\epsilon)^{-1} \sqrt{\frac{nn_3\log(nn_3)}{q}}\norm{\U*\V^*}_\infty.
	\end{align*}
	The fourth step is from Lemma \ref{lempre3} and the fifth is from (\ref{eq_Zrel1}). Now by using (\ref{eq_q}) and (\ref{tic3}), we have 
	$$\norm{\WL}\leq C'\epsilon,$$
	 for some numerical constant $C'$.

\noindent\textbf{2. Proof of (b).} Since $\Pomega\Y_{j_0} = \0$, $\Pomega(\U*\V^*+\PTo\Y_{j_0}) =  \Pomega(\U*\V^*-\PT\Y_{j_0}) = \Pomega(\Z_{j_0}) $, and it follows from (\ref{eq_Zrel2}) that $$\norm{\Z_{j_0}}_F \leq \epsilon^{j_0} \norm{\U*\V^*}_F \leq \epsilon^{j_0} \sqrt{r}.$$
Since $\epsilon\leq e^{-1}$ and $j_0 \geq 2\log(nn_3)$, $\epsilon^{j_0}\leq (nn_3)^{-2}$ and this proves the claim.

\noindent\textbf{3. Proof of (c).} We have $\U*\V^*+\WL = \Z_{j_0}+\Y_{j_0}$ and know that $\Y_{j_0}$ is supported on $\Omegat^c$. Therefore, since $\norm{\Z_{j_0}}_F\leq \lambda/8$. We only need to show that $\norm{\Y_{j_0}}_\infty\leq\lambda/8$. Indeed,
\begin{align*}
\norm{\Y_{j_0}}_\infty \leq & q^{-1}\sum_j \norm{\PP_{\Omegat_j} \Z_{j-1}}_\infty \\
 \leq & q^{-1}\sum_j \norm{ \Z_{j-1}}_\infty \\
  \leq & q^{-1}\sum_j\epsilon^{j-1} \norm{  \U*\V^*}_\infty.
\end{align*}
Since $\norm{\U*\V^*}_\infty \leq \sqrt{\frac{\mu r}{n^2n^2_3}}$, this gives
\begin{equation*}
\norm{\Y_{j_0}}_\infty \leq C'\frac{\epsilon^2}{\sqrt{\mu r(\log(nn_3))^2}},
\end{equation*}
for some numerical constant $C'$ whenever $q$ obeys (\ref{eq_q}). Since $\lambda = 1/\sqrt{nn_3}$, $\norm{\Y_{j_0}}_\infty \leq \lambda/8$ if
\begin{align*}
\epsilon\leq C \left(\frac{\mu r(\log(nn_3))^2}{nn_3}\right)^{1/4}.
\end{align*}
The claim is proved by using (\ref{eq_q}), (\ref{tic3}) and sufficiently small $\epsilon$ (provided that $\rho_r$ is sufficiently small. Note that everything is consistent since $C_0\epsilon^{-2}\frac{\mu r\log(nn_3)}{nn_3}<1$.
\end{proof}

\subsection{Proof of Lemma \ref{eq_keylem2}}
\begin{proof}
We denote $\M=\sgn{\SS_0}$ distributed as
\begin{equation*}\label{defM}
\M_{ijk} = \begin{cases}
1, & \text{w.p.} \ \rho/2, \\
0, & \text{w.p.} \ 1-\rho, \\
-1, & \text{w.p.} \ \rho/2.
\end{cases}
\end{equation*}
Note that for any $\sigma>0$, $\{\norm{\Pomega\PT}\leq\sigma\}$ holds with high probability provided that $\rho$ is sufficiently small, see Corollary \ref{corollar31}.

\noindent\textbf{1. Proof of (a).} By construction,
\begin{align*}
\WS = &\lambda\PTo\M+\lambda\PTo\sum_{k\geq 1} (\Pomega\PT\Pomega)^k\M \\
:=&  \PTo\WS_0 + \PTo\WS_1.
\end{align*}
Note that  $\norm{\PTo\WS_0} \leq \norm{\WS_0} = \lambda\norm{\M}$ and $\norm{\PTo\WS_1} \leq \norm{\WS_1} = \lambda\norm{\R(\M)}$, where $\R = \sum_{k\geq1} (\Pomega\PT\Pomega)^k$. Now, we will respectively show that $\lambda\norm{\M}$ and $\lambda\norm{\R(\M)}$ are small enough when $\rho$ is sufficiently small for $\lambda=1/\sqrt{nn_3}$. Therefor, $\norm{\WS}\leq1/4$. 

\textbf{1) Bound $\lambda\norm{\M}$.} 

By using Lemma \ref{lemspectrm} directly, we have that $\lambda\norm{\M}\leq \varphi(\rho)$ is sufficiently small given $\lambda=1/\sqrt{nn_3}$ and $\rho$ is sufficiently small.

\textbf{2) Bound $\norm{\R(\M)}$.} 

For simplicity, let $\Z=\R(\M)$. We have
\begin{equation}\label{boundrm1}
\norm{\Z}=\norm{\Zmbar}=\sup_{\x\in\mathbb{S}^{nn_3-1}} \norm{\Zmbar\x}_2.
\end{equation}
The optimal $\x$ to (\ref{boundrm1}) is an eigenvector of $\Zmbar^*\Zmbar$. Since $\Zmbar$ is a block diagonal matrix, the optimal $\x$ has a  block sparse structure, i.e., 
$\x \in B=\big\{  \x  \in\mathbb{R}^{nn_3} | \x = [\x_1^\top,\cdots , \x_i^\top \cdots, \x^\top_{n_3} ], \text{with }  \x_i\in\mathbb{R}^n, \text{ and there exists } j \text{ such that } \x_j\neq\0 \text{ and } \x_i=\0, i\neq j  \}$. Note that $\norm{\x}_2 = \norm{\x_j}_2 = 1$. Let $N$ be the $1/2$-net for $\mathbb{S}^{n-1}$ of size at most $5^n$ (see Lemma 5.2 in \cite{vershynin2010introduction}). Then the $1/2$-net, denoted as $N'$, for $B$ has the size at most $n_3\cdot 5^n$.  We have 
\begin{align*}
\norm{\R(\M)} =  &\norm{\bdiag( \overline{ \R(\M) }  )} \\
= & \sup_{\x,\y\in B} \inproduct{\x}{\bdiag( \overline{ \R(\M) }  )\y} \\ 
= & \sup_{\x,\y\in B}  \inproduct{\x\y^*}{\bdiag(\overline{\R({\M)}})}  \\
=& \sup_{\x,\y\in B}  \inproduct{\bdiag^*(\x\y^*)}{\overline{\R({\M)}}} ,
\end{align*}
where $\bdiag^*$, the joint operator of $\bdiag$, maps the block diagonal matrix $\x\y^*$ to a tensor of size $n\times n\times n_3$. Let $\Z' = \bdiag^*(\x\y^*)$ and $\Z = \mcode{ifft}(\Z',[\ ],3)$. We have
\begin{align*}
\norm{\R(\M)} =& \sup_{\x,\y\in B}   \inproduct{\Z'}{\overline{\R({\M)}}} \\
=&\sup_{\x,\y\in B} n_3   \inproduct{\Z}{{\R({\M)}}}  \\
=&\sup_{\x,\y\in B} n_3   \inproduct{\R(\Z)}{{{\M}}}  \\
\leq & \sup_{\x,\y\in N'} 4 n_3   \inproduct{\R(\Z)}{{{\M}}}.
\end{align*}
For a fixed pair $(\x,\y)$ of unit-normed vectors, define the random variable
\begin{equation*}
X(\x,\y) = \inproduct{4n_3\R(\Z)}{{{\M}}}.
\end{equation*}
Conditional on $\Omegat=\text{supp}(\M)$, the signs of $\M$ are independent and identically distributed symmetric and Hoeffding's inequality gives
\begin{align*}
\mathbb{P}(|X(\x,\y)|>t| \Omegat) \leq 2\exp\left(\frac{-2t^2}{\norm{4n_3\R(\Z)}_F^2}\right).
\end{align*}
Note that $\norm{4n_3\R(\Z))}_F\leq 4n_3\norm{\R}\norm{\Z}_F = 4\sqrt{n_3}\norm{\R}\norm{\Z'}_F =4\sqrt{n_3}\norm{\R}$. Therefore, we have 
\begin{align*}
\mathbb{P}\left( \sup_{\x,\y\in N'} |X(\x,\y)| >t | \Omegat \right) \leq 2|N'|^2\exp\left(-\frac{t^2}{8n_3\norm{\R}^2}\right).
\end{align*}
Hence,
\begin{align*}
\mathbb{P}\left( \norm{\R(\M)} >t | \Omegat \right) \leq 2|N'|^2\exp\left(-\frac{t^2}{8n_3\norm{\R}^2}\right).
\end{align*}
On the event $\{\norm{\Pomega\PT} \leq\sigma \}$,
\begin{align*}
\norm{\R} \leq \sum_{k\geq1} \sigma^{2k} = \frac{\sigma^2}{1-\sigma^2},
\end{align*}
and, therefore, unconditionally,
\begin{align*}
& \mathbb{P}\left( \norm{\R(\M)} >t   \right) \\
\leq & 2|N'|^2\exp\left(-\frac{\gamma^2t^2}{8n_3}\right) + \mathbb{P}\left( \norm{\Pomega\PT} \geq \sigma   \right), \ \gamma =\frac{1-\sigma^2}{2\sigma^2} \\
=& 2n_3^2 \cdot 5^{2n} \exp\left(-\frac{\gamma^2t^2}{8n_3}\right) + \mathbb{P}\left( \norm{\Pomega\PT} \geq \sigma   \right).
\end{align*}
Let $t=c\sqrt{nn_3}$, where $c$ can be a small absolute constant. Then the above inequality implies that $\norm{\R(\M)}\leq t$ with high probability.

\noindent\textbf{2. Proof of (b)}
Observe that 
\begin{align*}
\Pomegao\WS = -\lambda\Pomegao\PT(\Pomega-\Pomega\PT\Pomega)^{-1}\M.
\end{align*}
Now for $(i,j,k)\in\Omegat^c$, $\WS_{ijk} = \inproduct{\WS}{\eijk}$, and we have $\WS_{ijk} =\lambda \inproduct{ \Q(i,j,k)}{\M}$, where $\Q(i,j,k)$ is the tensor $-(\Pomega-\Pomega\PT\Pomega)^{-1}\Pomega\PT(\eijk)$. Conditional on $\Omegat = \text{supp}(\M)$, the signs of $\M$ are independent and identically distributed symmetric, and the Hoeffding's inequality gives
\begin{equation*}
\mathbb{P}(|\WS_{ijk}| > t\lambda|\Omegat) \leq 2\exp\left( -\frac{2t^2}{\norm{\Q(i,j,k)}_F^2}\right),
\end{equation*}
and
\begin{align*}
&\mathbb{P}(\sup_{i,j,k}|\WS_{ijk}| > t\lambda/n_3|\Omegat) \\
\leq & 2n^2n_3\exp\left( -\frac{2t^2}{\sup_{i,j,k}\norm{\Q(i,j,k)}_F^2}\right).
\end{align*}
By using (\ref{proabouPT}), we have
\begin{align*}
\norm{\Pomega\PT(\eijk)}_F \leq & \norm{\Pomega\PT}\norm{\PT(\eijk)}_F \\
\leq & \sigma\sqrt{\frac{2\mu r}{nn_3}},
\end{align*}
on the event $\{\norm{\Pomega\PT}\leq\sigma\}$. On the same event, we have $\norm{(\Pomega-\Pomega\PT\Pomega)^{-1}}\leq (1-\sigma^2)^{-1}$ and thus
$\norm{\Q(i,j,k)}_F^2 \leq \frac{2\sigma^2}{(1-\sigma^2)^2}\frac{\mu r}{nn_3}$.
Then, unconditionally,
\begin{align*}
& \mathbb{P}\left(\sup_{i,j,k}|\WS_{ijk}| >t\lambda\right)  \\
\leq & 2n^2n_3\exp\left( -\frac{nn_3\gamma^2t^2}{\mu r} \right) + \mathbb{P}(\norm{\Pomega\PT}\geq\sigma),
\end{align*}
where $\gamma=\frac{(1-\sigma^2)^2}{2\sigma^2}$. This proves the claim when $\mu r<\rho'_rnn_3\log(nn_3)^{-1}$ and $\rho_r'$ is sufficiently small.
\end{proof}

\section{Proofs of Some  Lemmas}\label{sec_prooflemmas}

\begin{lemma} \cite{tropp2012user}   \label{lembenmatrix}
	Consider a finite sequence $\{\Zm_k\}$ of independent, random $n_1\times n_2$ matrices that satisfy the assumption $\mathbb{E} \Zm_k=\0$ and $\norm{\Zm_k}\leq R$ almost surely. Let $\sigma^2 = \max\{\norm{\sum_k\mathbb{E}[\Zm_k\Zm_k^*]} , \max\{\norm{\sum_k\mathbb{E}[\Zm_k^*\Zm_k]} \}$. Then, for any $t\geq0$, we have
	\begin{align*}
	\mathbb{P}\left[ \normlarge{\sum_k\Zm_k} \geq t \right] \leq & (n_1+n_2) \exp\left( -\frac{t^2}{2\sigma^2+\frac{2}{3}Rt} \right) \\
	\leq &  (n_1+n_2) \exp\left( -\frac{3t^2}{8\sigma^2} \right), \ \text{for } t\leq\frac{\sigma^2}{R}.
	\end{align*}
	Or, for any $c>0$, we have
	\begin{align*}\label{eq_tropbound2}
	\normlarge{\sum_k\Zm_k} \geq 2\sqrt{c\sigma^2\log(n_1+n_2)} + cB\log(n_1+n_2),
	\end{align*}
	with probability at least $1-(n_1+n_2)^{1-c}$.
\end{lemma}

\subsection{Proof of Lemma \ref{lemspectrm}}

\begin{proof}
The proof has three steps.

\textit{Step 1: Approximation.}
We first introduce some notations. Let $\f_i^*$ be the $i$-th row of $\F_{n_3}$,  and $\Mh=\begin{bmatrix}
\Mh_1 \\ \Mh_2 \\ \vdots \\ \Mh_{n} 
\end{bmatrix} \in\mathbb{R}^{nn_3\times n}$ be a matrix unfolded by $\M$, where $\Mh_i\in\mathbb{R}^{n_3\times n}$ is the $i$-th horizontal slice of $\M$, i.e., $[\Mh_i]_{kj} = \M_{ikj}$. Consider that $\Mbar = \mcode{fft}(\M,[\ ],3)$, we have
\begin{equation*}\label{pfl351}
\Mmbar_i = \begin{bmatrix}
\f_i^*\Mh_1 \\  \f_i^*\Mh_2 \\ \vdots \\  \f_i^*\Mh_{n} 
\end{bmatrix},
\end{equation*}\label{pfl352}
where $\Mmbar_i\in\mathbb{R}^{n\times n}$ is the $i$-th frontal slice of $\M$. Note that 
\begin{equation}\label{pfl3544440}
\norm{\M} = \norm{\Mmbar} = \max_{i=1,\cdots,n_3} \ \norm{\Mmbar_i}.
\end{equation}
Let $N$ be the $1/2$-net for $\mathbb{S}^{n-1}$ of size at most $5^n$ (see Lemma 5.2 in \cite{vershynin2010introduction}). Then Lemma 5.3 in  \cite{vershynin2010introduction} gives 
\begin{equation}\label{prf13555}
\norm{\Mmbar_i} \leq 2 \ \max_{\x\in N} \ \norm{\Mmbar_i\x}_2.
\end{equation}
So we consider to bound $\norm{\Mmbar_i\x}_2$. 

\textit{Step 2: Concentration.} 
We can express $\norm{\Mmbar_i\x}_2^2$ as a sum of independent random variables
\begin{equation}\label{pfl353}
\norm{\Mmbar_i\x}_2^2 = \sum_{j=1}^{n} (\f_i^*\Mh_j\x)^2 := \sum_{j=1}^{n} z_j^2,
\end{equation}
where $z_j = \langle {\Mh_j},{\f_i\x^*}\rangle$, $j=1,\cdots,n$, are independent sub-gaussian random variables with $\mathbb{E}z_j^2=\rho\norm{\f_i\x^*}_F^2=\rho n_3$. Using (\ref{defM0}), we have
\begin{equation*} 
|[\Mh_j]_{kl}| = \begin{cases}
1, & \text{w.p.} \ \rho, \\
0, & \text{w.p.} \ 1-\rho.
\end{cases}
\end{equation*}
Thus, the sub-gaussian norm of $[\Mh_j]_{kl}$, denoted as $\norm{\cdot}_{\psi_2}$, is 
\begin{align*}
\norm{[\Mh_j]_{kl}}_{\psi_2} = & \sup_{p\geq 1} p^{-\frac{1}{2}}(\mathbb{E}[|[\Mh_j]_{kl}|^p])^{\frac{1}{p}} \\
= & \sup_{p\geq 1} p^{-\frac{1}{2}}\rho^{\frac{1}{p}}.
\end{align*}
Define the function $\phi(x) =x^{-\frac{1}{2}}\rho^{\frac{1}{x}}$ on $[1,+\infty)$. The only stationary point occurs at $x^*=\log \rho^{-2}$. Thus, 
\begin{align}
\phi(x)\leq & \max(\phi(1),\phi(x^*)) \notag\\
= &\max\left(\rho,(\log\rho^{-2})^{-\frac{1}{2}}\rho^{\frac{1}{\log \rho^{-2}}}\right) \notag\\
:= & \psi(\rho). \label{abcefeqdd}
\end{align}
Therefore,  $\norm{[\Mh_j]_{kl}}_{\psi_2}\leq   \psi(\rho)$. Consider that $z_j$ is a sum of independent centered sub-gaussian random variables $[\Mh_j]_{kl}$'s, by using Lemma 5.9 in  \cite{vershynin2010introduction}, we have $\norm{z_j}^2_{\psi_2}\leq c_1(\psi(\rho))^2 n_3$, where $c_1$ is an absolute constant. Therefore, by Remark 5.18 and Lemma 5.14 in \cite{vershynin2010introduction}, $z_j^2-\rho n_3$ are independent centered sub-exponential random variables with $\norm{z_j^2-\rho n_3}_{\psi_1} \leq 2 \norm{z_j}^2_{\psi_1} \leq 4 \norm{z_j}^2_{\psi_2}\leq 4c_1 (\psi(\rho))^2 n_3$.

Now, we use an exponential deviation inequality, Corollary 5.17 in \cite{vershynin2010introduction}, to control the sum  (\ref{pfl353}). We have
\begin{align*}
& \mathbb{P}\left(| \norm{\Mmbar_i\x}_2^2 -\rho nn_3 | \geq tn \right) \\
= & \mathbb{P}\left(\left| \sum_{j=1}^{n} (z_j^2-\rho n_3) \right| \geq tn \right) \\
\leq & 2\exp\left( -c_2n\min \left( \left(\frac{t}{4c_1 (\psi(\rho))^2 n_3}\right)^2, \frac{t}{4c_1(\psi(\rho))^2 n_3}\right) \right),
\end{align*}
where $c_2>0$. Let $t = c_3 (\psi(\rho))^2 n_3$ for some absolute constant $c_3$, we have   
\begin{align*}
& \mathbb{P}\left(| \norm{\Mmbar_i\x}_2^2 -\rho nn_3 | \geq c_3(\psi(\rho))^2 nn_3 \right) \\
\leq & 2\exp\left( -c_2n\min\left( \left(\frac{c_3}{4c_1}\right)^2,\frac{c_3}{4c_1}\right) \right).
\end{align*}

\textit{Step 3： Union bound.}
Taking the union bound over all $\x$ in the net $N$ of cardinality $|N|\leq 5^n$, we obtain
\begin{align*}
&\mathbb{P}\left(\left| \max_{\x \in N} \norm{\Mmbar_i\x}_2^2 -\rho nn_3 \right| \geq c_3(\psi(\rho))^2 nn_3 \right) \\
\leq & 2\cdot 5^n \cdot\exp\left( -  c_2n\min\left(\left(\frac{c_3}{4c_1}\right)^2,\frac{c_3}{4c_1}\right) \right).
\end{align*}
Furthermore, taking the union bound over all $i=1,\cdots,n_3$, we have 
\begin{align*}
&\mathbb{P}\left(\max_i \ \left| \max_{\x \in N} \norm{\Mmbar_i\x}_2^2 -\rho nn_3 \right| \geq c_3(\psi(\rho))^2 nn_3 \right) \notag\\
\leq &2\cdot 5^n\cdot n_3 \cdot \exp\left( -  c_2n\min\left(\left(\frac{c_3}{4c_1}\right)^2,\frac{c_3}{4c_1}\right) \right). \label{prfnlemmada000}
\end{align*}
This implies that, with high probability (when the constant $c_3$ is  large enough),
\begin{equation}\label{profboundm2}
\max_i \  \max_{\x \in N} \norm{\Mmbar_i\x}_2^2 \leq (\rho+c_3(\psi(\rho))^2)  nn_3.
\end{equation}
Let $\varphi(\rho) = 2\sqrt{\rho+c_3(\psi(\rho))^2}$ and it satisfies  $\lim\limits_{\rho\rightarrow0^+}\varphi(\rho)=0$ by using (\ref{abcefeqdd}). The proof is completed by further combining  (\ref{pfl3544440}), (\ref{prf13555}) and (\ref{profboundm2}).
\end{proof}

\subsection{Proof of Lemma \ref{lem_kem1}}
\begin{proof}
	For any tensor $\Z$, we can write
	\begin{align*}
	&(\rho^{-1}\PT\Pomega\PT-\PT)\Z \\
	= &\sum_{ijk}\left(\rho^{-1}{\delta_{ijk}}-1\right)\inproduct{\eijk }{\PT\Z}\PT(\eijk) \\
	:=&\sum_{ijk} \HH_{ijk}(\Z)
	\end{align*}
	where $\HH_{ijk}: \mathbb{R}^{\nss}\rightarrow\mathbb{R}^{\nss}$ is a self-adjoint random operator with $\mathbb{E}[\HH_{ijk}]=\0$. Define the matrix operator $\Hmbar_{ijk}: \mathbb{B}\rightarrow \mathbb{B}$, where $\mathbb{B}=\{\Bmbar: \B\in \mathbb{R}^{\nss} \}$ denotes the set consists of block diagonal matrices with the blocks as the frontal slices of $\Bbar$,  as
	\begin{align*}
	\Hmbar_{ijk}(\Zmbar) = & \left(\rho^{-1}{\delta_{ijk}}-1\right)\inproduct{\eijk }{\PT(\Z)}\bdiag(\overline{\PT(\eijk)}).
	\end{align*}
	By the above definitions, we have $\norm{\HH_{ijk}} = \norm{\Hmbar_{ijk}}$ and $\norm{\sum_{ijk}\HH_{ijk}} = \norm{\sum_{ijk}\Hmbar_{ijk}}$. Also $\Hmbar_{ijk}$ is  self-adjoint and $\mathbb{E}[\Hmbar_{ijk}]=0$. 
	To prove the result by the non-commutative Bernstein inequality, we need to bound $\norm{\Hmbar_{ijk}}$ and $\normlarge{\sum_{ijk}\mathbb{E}[\Hmbar^2_{ijk}]}$. First, we have
	\begin{align*}
	\norm{\Hmbar_{ijk}} = & \sup_{\norm{\Zmbar}_F=1} \norm{\Hmbar_{ijk}(\Zmbar)}_F \\
	\leq &  \sup_{\norm{\Zmbar}_F=1} \rho^{-1} \norm{\PT(\eijk)}_F \norm{\bdiag(\overline{\PT(\eijk)})}_F \norm{\Z}_F \\
	= & \sup_{\norm{\Zmbar}_F=1} \rho^{-1} \norm{\PT(\eijk)}_F^2 \norm{\Zmbar}_F \\
	\leq & \frac{2\mu r}{nn_3\rho},
	\end{align*}
	where the last inequality uses (\ref{proabouPT}). On the other hand, by direct computation, we have $\Hmbar_{ijk}^2(\Zmbar) = (\rho^{-1}\delta_{ijk}-1)^2\inproduct{\eijk}{\PT(\Z)}\inproduct{\eijk}{\PT(\eijk)}\bdiag(\overline{\PT(\eijk)})$. Note that $\mathbb{E}[(\rho^{-1}\delta_{ijk}-1)^2]\leq \rho^{-1}$. We have
	\begin{align*}
	& \normlarge{\sum_{ijk}\mathbb{E}[\Hmbar^2_{ijk}(\Zmbar)]}_F \\
	\leq & \rho^{-1}\normlarge{\sum_{ijk} \inproduct{\eijk}{\PT(\Z)}\inproduct{\eijk}{\PT(\eijk)}\bdiag(\overline{\PT(\eijk)})}_F \\
	\leq & \rho^{-1} \sqrt{n_3} \norm{\PT(\eijk)}_F^2\normlarge{\sum_{ijk} \inproduct{\eijk}{\PT(\Z)} }_F \\
	= & \rho^{-1}\sqrt{n_3} \norm{\PT(\eijk)}_F^2\norm{\PT(\Z)}_F\\
	\leq & \rho^{-1}\sqrt{n_3} \norm{\PT(\eijk)}_F^2\norm{\Z}_F\\
	= & \rho^{-1} \norm{\PT(\eijk)}_F^2\norm{\Zmbar}_F\\
	\leq & \frac{2\mu r}{nn_3\rho}\norm{\Zmbar}_F.	
	\end{align*}
	This implies $\normlarge{\sum_{ijk}\mathbb{E}[\Hmbar^2_{ijk}]}\leq \frac{2\mu r}{nn_3\rho}$.  Let $\epsilon\leq 1$. By Lemma \ref{lembenmatrix}, we have
	\begin{align*}
	& \mathbb{P} \left[\norm{\rho^{-1}\PT\Pomega\PT-\PT} > \epsilon \right] \\
	= & \mathbb{P}\left[ \normlarge{  \sum_{ijk} {\HH}_{ijk} } > \epsilon\right] \\
	= & \mathbb{P}\left[ \normlarge{  \sum_{ijk} {\Hmbar}_{ijk} } > \epsilon \right] \\
	\leq & 2nn_3 \exp\left( -\frac{3}{8} \cdot \frac{\epsilon^2}{2\mu r/(nn_3\rho) } \right) \\
	\leq & 2(nn_3)^{1-\frac{3}{16}C_0},
	\end{align*}
	where the last inequality uses  $\rho\geq C_0\epsilon^{-2}\mu r\log(nn_3)/(nn_3)$.
	Thus, $\norm{\rho^{-1}\PT\Pomega\PT-\PT} \leq \epsilon$ holds with high probability for some numerical constant $C_0$.
\end{proof}

\subsection{Proof of Corollary \ref{corollar31}}
\begin{proof}
From Lemma \ref{lem_kem1}, we have 
\begin{equation*}
\norm{\PT-(1-\rho)^{-1}\PT\Pomegao\PT}\leq \epsilon,
\end{equation*}
provided that $1-\rho\geq C_0 \epsilon^{-2}(\mu r\log(nn_3))/n$. Note that $\I = \Pomega+\Pomegao$, we have 
\begin{equation*}
\norm{\PT-(1-\rho)^{-1}\PT\Pomegao\PT} = (1-\rho)^{-1} (\PT\Pomega\PT - \rho \PT).
\end{equation*}
Then, by the triangular inequality
\begin{equation*}
\norm{\PT\Pomega\PT} \leq \epsilon (1-\rho) + \rho \norm{\PT} = \rho + \epsilon(1-\rho).
\end{equation*}
The proof is completed by using $\norm{\Pomega\PT}^2 = \norm{\PT\Pomega\PT}$.
\end{proof}

\subsection{Proof of Lemma \ref{lem_keyinf}}
\begin{proof} For any tensor $\Z\in\Tm$, we write
	\begin{align*}
	\rho^{-1}\PT\Pomega(\Z) = \sum_{ijk}\rho^{-1}\delta_{ijk} z_{ijk}\PT(\eijk).
	\end{align*}
	The $(a,b,c)$-th entry of $\rho^{-1}\PT\Pomega(\Z)-\Z$ can be written as a sum of independent random variables, i.e.,
	\begin{align*}
	& \inproduct{\rho^{-1}\PT\Pomega(\Z)-\Z}{\eabc} \\
	= & \sum_{ijk} (\rho^{-1}\delta_{ijk}-1)z_{ijk} \inproduct{\PT(\eijk)}{\eabc} \\
	:=& \sum_{ijk} t_{ijk},
	\end{align*}
	where $t_{ijk}$'s are independent and $\mathbb{E}(t_{ijk})=0$. Now we bound $|t_{ijk}|$ and $|\sum_{ijk}\mathbb{E}[t_{ijk}^2]|$. First
	\begin{align*}
	&|t_{ijk}|\\	
	\leq & \rho^{-1} \norm{\Z}_\infty \norm{\PT(\eijk)}_F\norm{\PT(\eabc)}_F \\
	\leq & \frac{2\mu r}{nn_3\rho}\norm{\Z}_\infty.
	\end{align*}
	Second, we have
	\begin{align*}
	& \left|\sum_{ijk}\mathbb{E}[t_{ijk}^2]\right |\\
	\leq & \rho^{-1}\norm{\Z}_\infty^2\sum_{ijk}\inproduct{\PT(\eijk)}{\eabc}^2 \\
	=  & \rho^{-1}\norm{\Z}_\infty^2\sum_{ijk}\inproduct{\eijk}{\PT(\eabc)}^2 \\
	=  & \rho^{-1}\norm{\Z}_\infty^2\norm{\PT(\eabc)}_F^2 \\
	\leq & \frac{2\mu r}{nn_3\rho}\norm{\Z}_\infty^2.
	\end{align*}
	 Let $\epsilon\leq 1$. By Lemma \ref{lembenmatrix}, we have
	 \begin{align*}
	 & \mathbb{P} \left[ |[\rho^{-1}\PT\Pomega(\Z)-\Z]_{abc}| > \epsilon\norm{\Z}_\infty\right] \\
	 = & \mathbb{P}\left[ \left| \sum_{ijk} {t}_{ijk}\right|> \epsilon\norm{\Z}_\infty\right] \\
	 \leq & 2 \exp\left( -\frac{3}{8} \cdot \frac{\epsilon^2\norm{\Z}^2_\infty}{2\mu r\norm{\Z}^2_\infty/(nn_3\rho) } \right) \\
	 \leq & 2(nn_3)^{-\frac{3}{16}C_0},
	 \end{align*}
	 where the last inequality uses  $\rho\geq C_0\epsilon^{-2}\mu r\log(nn_3)/(nn_3)$.
	 Thus, $\norm{\rho^{-1}\PT\Pomega(\Z)-\Z}_\infty\leq \epsilon\norm{\Z}_\infty$ holds with high probability for some numerical constant $C_0$. 	
\end{proof}

\subsection{Proof of Lemma \ref{lempre3}}

\begin{proof}
	Denote the tensor $\HH_{ijk} = \left(1-\rho^{-1}\delta_{ijk}\right)z_{ijk}\eijk$. Then we have
	\begin{equation*}
	(\I-\rho^{-1}\Pomega)\Z = \sum_{ijk}\HH_{ijk}.
	\end{equation*}
	Note that $\delta_{ijk}$'s are independent random scalars. Thus, $\HH_{ijk}$'s are independent random tensors and $\Hmbar_{ijk}$'s are independent random matrices. 
	Observe that $\mathbb{E}[{\Hmbar}_{ijk}] = \0$ and $\norm{{\Hmbar}_{ijk}} \leq {\rho}^{-1} \norm{\Z}_\infty$. We have
	\begin{align*}
	&\normlarge{\sum_{ijk} \mathbb{E} [ {\Hmbar}^*_{ijk} {\Hmbar}_{ijk} ]  } \\
	= & \normlarge{\sum_{ijk} \mathbb{E} [ {\HH}^*_{ijk} *{\HH}_{ijk} ]  } \\
	= & \normlarge{\sum_{ijk}   \mathbb{E}[ (1-\rho^{-1}{\delta_{ijk}})^2 ] z_{ijk}^2 ({\ej}*\ej^*) } \\
	= & \normlarge{\frac{1-\rho}{\rho} \sum_{ijk} z_{ijk}^2 ({\ej}*\ej^*) } \\
	\leq & { \frac{nn_3}{\rho} }\norm{\Z}_\infty^2.
	\end{align*}
	A similar calculation yields $\normlarge{\sum_{ijk} \mathbb{E} [{\Hmbar}_{ijk}^* {\Hmbar}_{ijk} ]  }\leq { \rho^{-1}nn_3 }\norm{\Z}_\infty^2$. Let $t = \sqrt{C_0{nn_3\log(nn_3)}/{\rho}}\norm{\Z}_\infty$. When $\rho\geq C_0\log(nn_3)/(nn_3)$, we apply Lemma \ref{lembenmatrix} and obtain
	\begin{align*}
	& \mathbb{P}\left[ \norm{(\I-\rho^{-1}\Pomega)\Z } > t \right]  \\
	= & \mathbb{P}\left[ \normlarge{  \sum_{ijk} {\HH}_{ijk} } > t \right] \\
	= & \mathbb{P}\left[ \normlarge{  \sum_{ijk} {\Hmbar}_{ijk} } > t \right] \\
	\leq & 2nn_3 \exp\left( -\frac{3}{8} \cdot \frac{C_0nn_3\log(nn_3)\norm{\Z}_\infty^2/\rho}{nn_3\norm{\Z}_\infty^2/\rho } \right) \\
	\leq & 2(nn_3)^{1-\frac{3}{8}C_0}.
	\end{align*}
	Thus, $ \norm{(\I-\rho^{-1}\Pomega)\Z } > t$ holds with high probability for some numerical constant $C_0$.
\end{proof}

{\small
\bibliographystyle{ieee}
\bibliography{ref}
}